\documentclass{article} % For LaTeX2e
\usepackage{iclr2026_conference,times}

\usepackage{amsmath,amsfonts,bm}

\def\eqref#1{equation~\ref{#1}}
\def\1{\bm{1}}

\DeclareMathAlphabet{\mathsfit}{\encodingdefault}{\sfdefault}{m}{sl}
\SetMathAlphabet{\mathsfit}{bold}{\encodingdefault}{\sfdefault}{bx}{n}

\declaretheorem[name=Theorem]{thm}
\declaretheorem[name=Lemma]{lem}

\declaretheorem[name=Assumption]{ass}

\usepackage{hyperref}
\usepackage{url}

\newcommand{\n}[1]{{\color{black}#1}}
\newcommand{\ch}[1]{{\color{black}#1}}

\title{MIRA: Memory-Integrated Reinforcement Learning Agent with Limited LLM Guidance}

\author{  
Narjes Nourzad \\
  University of Southern California \\
  Electrical and Computer Engineering  \\
  Los Angeles, CA 90007, USA \\
  \texttt{nourzad@usc.edu} \\
\And
  Carlee Joe-Wong \\
  Carnegie Mellon University \\
  Electrical and Computer Engineering  \\ 
  Pittsburgh, PA 15213, USA \\
  \texttt{cjoewong@andrew.cmu.edu} \\
}

\iclrfinalcopy % Uncomment for camera-ready version, but NOT for submission.

\definecolor{mypurple}{HTML}{BE9FD5}
\definecolor{myblue}{HTML}{ABC5F2}
\definecolor{mygreen}{HTML}{D0DBC4}
\definecolor{mygreen3}{HTML}{dfe8e0} 
\definecolor{mypurple2}{HTML}{d5bddd} 
\definecolor{mygreen2}{HTML}{8EAD91}
\definecolor{myblue2}{HTML}{cfe4ff}
\definecolor{Q}{HTML}{472c78}

\definecolor{mpink}{HTML}{E2B3C3}      
\definecolor{mypink}{HTML}{E2B3C3}
\definecolor{myellow}{HTML}{e2d66f}   % Muted yellow
\definecolor{mgray}{HTML}{7a7a7a}     % Deeper gray
\definecolor{mpurple}{HTML}{BE9FD5}   % Muted purple
\definecolor{mgreen}{HTML}{80A970}    % Muted green
\definecolor{mblue}{HTML}{86addc}     % Muted blue
\definecolor{mgblue}{HTML}{AAADBA}

\hypersetup{
  colorlinks=true,     % use colored text instead of boxes
  linkcolor=black,     % section links
  citecolor=mypurple, % citation links
  urlcolor=black       % urls
}

\begin{document}

\maketitle

\begin{abstract}
Reinforcement learning (RL) agents often face high sample complexity in sparse or delayed reward settings, due to limited prior knowledge.
Conversely, large language models (LLMs) can provide subgoal structures, plausible trajectories, and abstract priors that support early learning.
Yet heavy reliance on LLMs introduces scalability issues and risks dependence on unreliable signals, motivating ongoing efforts to integrate LLM guidance without compromising RL’s autonomy.
We propose MIRA (\underline{M}emory-\underline{I}ntegrated \underline{R}einforcement Learning \underline{A}gent), \ch{which incorporates a structured, evolving \textit{memory graph} to guide early learning.}
This graph stores decision-relevant information, such as trajectory segments and subgoal decompositions, and is co-constructed from the agent’s high-return experiences and LLM outputs, \ch{amortizing LLM queries into a persistent memory instead of relying on continuous real-time supervision.}
From this structure, we derive a \textit{utility} signal that \ch{softly adjusts} advantage estimation to refine policy updates without altering the underlying reward function.
As training progresses, the agent’s policy surpasses the initial LLM-derived priors, and the utility term decays, leaving long-term convergence guarantees intact.
\ch{We show theoretically that this utility-based shaping improves early-stage learning in sparse-reward settings.}
Empirically, MIRA outperforms RL baselines and \ch{reaches returns comparable to methods that rely on frequent LLM supervision,} while requiring substantially fewer online LLM queries\footnote{Project webpage : \url{https://narjesno.github.io/MIRA/}}.

% Reinforcement learning (RL) agents often face high sample complexity in sparse or delayed reward settings due to limited prior knowledge. Conversely, large language models (LLMs) can provide subgoal structures, plausible trajectories, and abstract priors that support early learning. Yet heavy reliance on LLMs introduces scalability issues and risks dependence on unreliable signals, motivating ongoing efforts to integrate LLM guidance without compromising RL autonomy.
% We propose *MIRA* (Memory-Integrated Reinforcement Learning Agent), which incorporates a structured, evolving memory graph to guide early learning. This graph stores decision-relevant information, such as trajectory segments and subgoal decompositions, and is co-constructed from the agent’s high-return experiences and LLM outputs, amortizing LLM queries across a persistent memory rather than relying on continuous, real-time supervision. From this structure, we derive a utility signal that softly adjusts advantage estimation to refine policy updates without altering the underlying reward function. As training progresses, the agent’s policy surpasses the initial LLM-derived priors, and the utility term decays, leaving long-term convergence guarantees intact. We show theoretically that this utility-based shaping improves early-stage learning in sparse-reward settings. Empirically, MIRA outperforms RL baselines and achieves returns comparable to those of methods that rely on frequent LLM supervision, while requiring substantially fewer online LLM queries.

\end{abstract}

\section{Introduction}

Reinforcement learning (RL) models sequential decision-making as interactions with an environment and learns behavior from reward-driven feedback.
RL has achieved strong results in domains including robotic manipulation, dynamic scheduling, and autonomous planning~\citep{nourzadactor, liu2024integrating, luo2024precise}.  
However, these advances often rely on environments with dense, readily accessible rewards. 
In many tasks, rewards are sparse or delayed, appearing only when specific goals are reached or several steps after the action unfolds.
These weak or infrequent rewards obscure which past actions influenced the outcome, making it difficult to credit the eventual reward appropriately~\citep{velu2023hindsight}.
This uncertainty weakens the gradient signal, leaving policy updates underinformed. Thus, agents become highly data-hungry and require large numbers of interactions to learn useful behaviors~\citep{devidze2022exploration}.
These challenges intensify under partial observability, as agents must generalize from limited state information and often struggle early in training~\citep{hausknecht2015deep, kurniawati2022partially}.
In such settings, random exploration rarely uncovers informative trajectories, leading to slow convergence and high variance in returns.

Large language models (LLMs) provide a complementary source of prior knowledge, especially in environments where rewards are sparse, feedback is delayed, and observations are partial.
They have demonstrated capabilities in reasoning over abstract goals, interpreting high-level intent, and leveraging broad prior knowledge~\citep{jimenez2023swe, xu2024theagentcompany}. These properties make them natural candidates for providing structured guidance for RL agents~\citep{schoepp2025evolving}. A growing body of work has explored how pretrained LLMs can support RL to improve sample efficiency. One line of research positions the LLM as an implicit or explicit reward model, either estimating reward signals from environment descriptions or generating code to define reward functions~\citep{ma2025catching, kwon2023reward, fan2022minedojo, rocamonde2023vision, bhambri2024extracting, xie2024text2reward}. Another line leverages LLMs to generate high-level plans, policy sketches, or step-by-step guidance during training~\citep{du2023guiding, hu2023language, dasgupta2023collaborating, wang2023voyager, zhou2023large}. A third direction focuses on task-level guidance such as subgoal decomposition, curriculum design, or goal interpretation from natural language~\citep{wang2024dart, ma2023eureka, shinn2023reflexion}. Additional related approaches appear in Appendix~\ref{rw}.

{\textsc{Research challenges.}}
The existing approaches, although promising, typically require \textit{frequent (often per-step) LLM supervision}, making the agent's performance heavily reliant on LLM inference, which introduces several difficulties. 
First, it can interfere with the RL learning signal~\citep{zhou2023large}, \ch{limiting autonomous decision-making and reducing the agent’s ability to generalize or adapt if the LLM becomes unavailable.}
Second, since LLMs cannot interact directly with the environment or gather real-time feedback, full reliance on their instructions is suboptimal~\citep{qu2024choices, gao2024designing, cao2024survey} and dilutes environment-driven feedback. Indeed, LLMs carry fundamental risks such as hallucinated outputs, prompt sensitivity, and limited grounding in physical environments~\citep{ji2023towards, tonmoy2024comprehensive, bang2025hallulens}, making their outputs potentially unreliable. 
Frequent queries also raise scalability concerns due to computational cost and latency~\citep{zhou2024survey, wan2023efficient}. 
Still, relying solely on RL ignores the structured knowledge encoded in LLMs that could accelerate learning or shape behavior in meaningful ways.
\textit{Thus, the fundamental challenge \ch{lies in incorporating LLM guidance in a way that leverages its complementary benefits while preserving the optimization dynamics that make RL effective.}}

\begin{figure}
    \centering
    \includegraphics[width=1\linewidth]{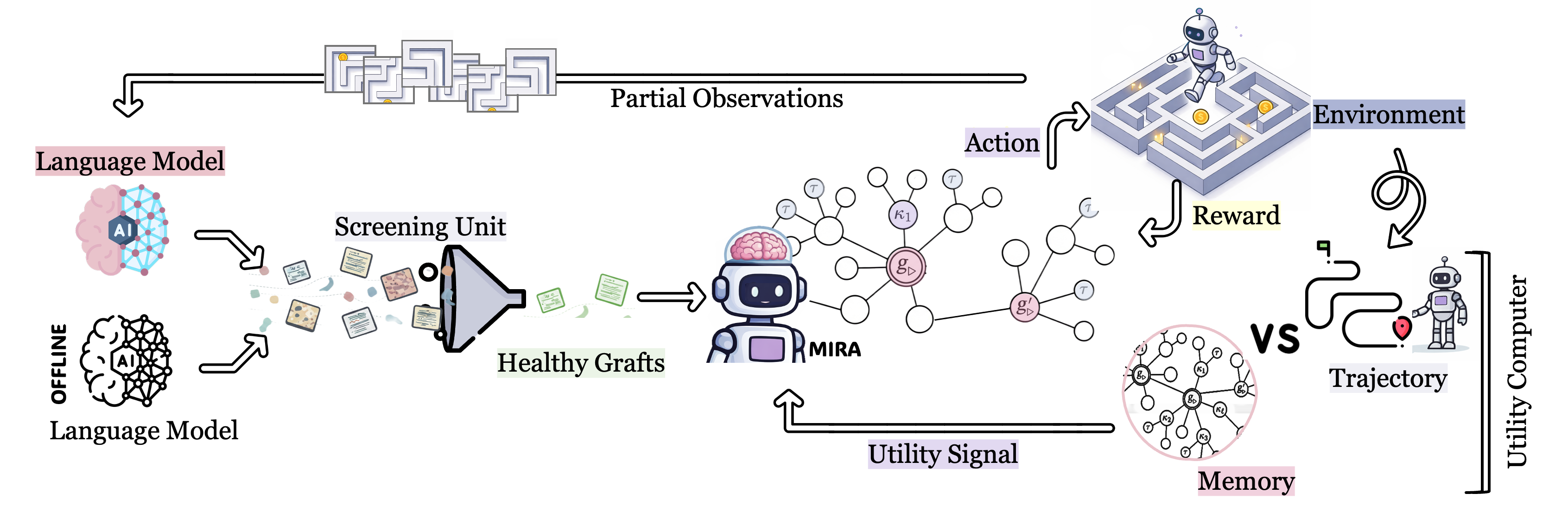}
    \caption{\textcolor{mgray}{Offline} priors and \textcolor{mpink}{online} LLM suggestions are filtered by a \textcolor{mgblue}{screening unit} before being incorporated into the memory graph as \textcolor{mgreen}{\emph{healthy grafts}}. \textcolor{mpurple}{MIRA} agent acts under partial observations, interacting with the environment. A utility module evaluates trajectory rollouts against the evolving \textcolor{mpink}{memory graph}, producing a utility signal that shapes advantage estimation and policy updates.
}
  \label{fig:mira}
\end{figure}
{\textsc{Our contributions.}} In this work, we propose MIRA (\underline{M}emory-\underline{I}ntegrated \underline{R}einforcement Learning \underline{A}gent), a method that integrates LLM-derived guidance into reinforcement learning through a structured \textit{memory graph}. 
The memory graph provides an evolving representation of task-relevant information, co-constructed from the agent’s experience and LLM outputs. 
% \begin{wrapfigure}{r}{0.5\textwidth}
% \vspace{-8pt}
%   \centering
%   \includegraphics[width=\linewidth]{figs/main_aura_sys.png}
%   \caption{Overview of MIRA. \textcolor{mgray}{Offline} priors and \textcolor{mpink}{Online} suggestions from LLMs pass through a \textcolor{mgreen}{Screening Unit} before populating the memory graph with \textit{healthy grafts}. \textcolor{mpurple}{MIRA} agent interacts with the environment, while a \textcolor{mblue}{Utility} module evaluates rollouts against the evolving \textcolor{mypurple2}{memory graph} to shape advantage estimates.}
%   \label{fig:mira}
%   \vspace{-10pt}
% \end{wrapfigure}
% Offline priors, pre-processed over environments or goals, initialize the structure, while infrequent online queries conditioned on batches of partial environment observations refine it during training.
\ch{Offline priors initialize the structure, while infrequent online queries on batches of partial observations can further refine it during training.}
Nodes represent decision-relevant context, such as trajectory segments, and edges encode the hierarchical decomposition linking goals to their subgoals.
% The graph is designed to remain compact, incurring minimal overhead compared to standard replay buffers~\citep{schaul2015prioritized}, yet it provides a persistent structure that can be reused across training. 
% This persistence means that LLM outputs \textit{need not be queried repeatedly}; instead, they are stored and organized as part of a long-term structured knowledge. 
The graph is designed to remain compact, adding minimal overhead relative to standard replay buffers~\citep{schaul2015prioritized}. 
The memory graph allows the agent to organize and reuse information \textit{without repeated LLM queries}, providing a persistent source of structured knowledge.

O\ch{ver time, the agent validates, revises, and extends the structure based on its own experience, improving beyond what LLM guidance alone can provide and filtering out mistaken online suggestions.}
The resulting graph limits dependence on real-time LLM access, alleviating concerns about latency, query cost, and scalability. 
To integrate the LLM-derived information into learning, we derive a \textit{utility signal} from the memory graph and use it to softly shape advantage estimates in each RL iteration. 
% This signal reinforces learning when aligned with reward-driven gradients and gradually corrects miscalibrated updates when the critic is poorly informed.
\ch{This signal guides early rollouts by reinforcing reward-driven gradients when aligned and moderating updates that arise from an inaccurate critic, helping the agent explore more effectively in sparse-reward settings without overriding the environment’s feedback.}
{Theoretically, we show that the utility term accelerates early learning. 
As the policy improves and surpasses the usefulness of LLM-derived guidance, the shaping influence fades, ensuring convergence in the long-horizon limit.
We empirically evaluate the effectiveness, sample efficiency, and overhead of incorporating LLM guidance across multiple benchmark environments. }
\\
Our contributions are summarized as follows:
% \vspace{-5mm}
{\begin{itemize}
    \ch{\item \textsc{\textbf{A memory-integrated framework for LLM guidance:}}}
    We propose \textbf{MIRA}, a reinforcement learning agent that integrates LLM-derived guidance through a memory graph co-constructed \ch{from agent experience and offline or infrequent online LLM outputs. The graph evolves throughout training, reducing reliance on real-time LLM queries.}

   \ch{ \item \textsc{\textbf{Utility-based advantage shaping:}}
    We introduce \textbf{utility-shaped advantage estimation}, which incorporates graph-derived utility into advantage computation} without architectural changes and is compatible with any advantage-based policy-gradient method.

    \ch{\item \textsc{\textbf{Convergence-compatible shaping:}}
    We provide \textbf{theoretical guarantees} showing that by decaying the shaping influence as the policy improves}, we preserve long-horizon convergence properties of Proximal Policy Optimization (PPO)~\citep{schulman2017ppo} \ch{and correct any inaccuracies in LLM outputs}.

   \ch{ \item \textsc{\textbf{Empirical validation across benchmarks:}}}
    We demonstrate \textbf{empirically} that MIRA improves sample efficiency over RL and hierarchical baselines and reaches final performance comparable to methods requiring continuous LLM  supervision~\citep{zhou2023large, bhambri2024extracting}, while using far fewer online queries.
\end{itemize}}

% \begin{itemize}
% \item We propose \textbf{MIRA}, a reinforcement learning agent that integrates LLM-derived guidance through a memory graph co-constructed from agent experience and LLM knowledge. This graph evolves throughout training, combining offline priors with infrequent online queries conditioned on batches of partial observations from the environment.

% \item We develop \textbf{adaptive advantage shaping}, which incorporates utility derived from the memory graph directly into advantage estimates. This mechanism requires no architectural changes and applies to any advantage-based policy-gradient algorithm.

% \item We provide \textbf{theoretical analysis} showing that the shaping mechanism preserves the convergence guarantees of Proximal Policy Optimization (PPO)~\citep{schulman2017ppo} in long horizon limit by augmenting, rather than overriding the optimization dynamics.

% \item We demonstrate \textbf{empirically} that MIRA improves sample efficiency over RL and HRL baselines, and achieves competitive final returns with far fewer LLM queries than methods based on continuous supervision.

% \end{itemize}

The remainder of this paper is organized as follows. Section~\ref{meth} details MIRA’s architecture; Sections~\ref{setup} and~\ref{res} present experimental setup and results across multiple benchmarks; and Section~\ref{conc} concludes with a discussion of our findings and possible directions for future work.

\section{Methodology} \label{meth}
We now describe the design of MIRA, whose development is guided by two desiderata: \textsc{(i)} improve early learning by incorporating task-relevant priors from an LLM,
\textsc{(ii)} minimize reliance on continuous real-time LLM supervision to ensure scalability and maintain autonomous policy learning.
MIRA is built on the standard policy-gradient formulation for reinforcement learning (Appendix~\ref{bck}).

\subsection{Memory Graph Design} \label{graph}
The agent maintains an evolving memory graph that organizes information drawn from both LLM suggestions and agent rollouts. Nodes of the graph represent decision-relevant context, 
\begin{wrapfigure}{r}{0.38\textwidth}
\vspace{-3mm}
   \centering
\includegraphics[width=.8\linewidth]{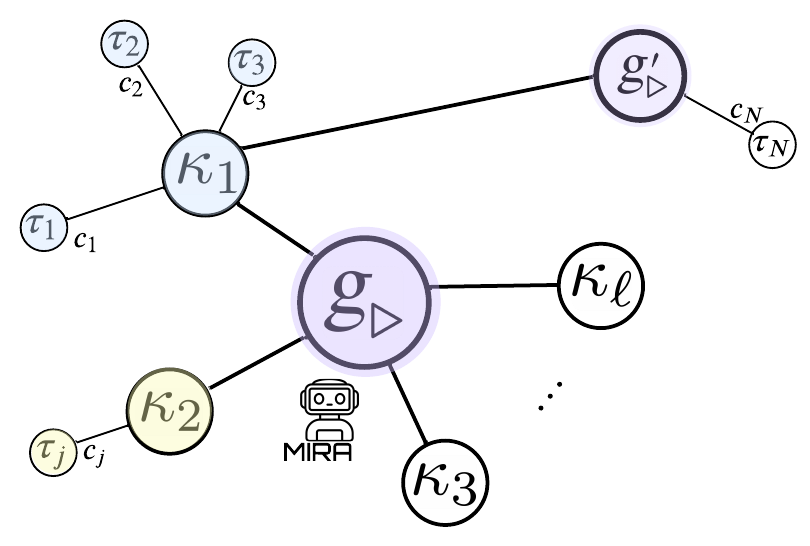}
\caption{\n{MIRA’s evolving memory graph. Trajectory segments $\tau_j$ are grouped under subgoal nodes $\kappa_\ell$. Subgoals can be shared across multiple final goals, enabling reuse of common behaviors.}}
    \label{fig:graph} 
    \vspace{-5mm}
\end{wrapfigure}
and edges encode the hierarchical decomposition of goals into subgoals as provided by the LLM \n{(Figure~\ref{fig:graph})}.
This structure can be expressed as
\begin{equation}
\mathcal{G}= 
\Bigl\{ \big((o, a)_{\tau_j}, \zeta_j, \hat r_j\big)_{c_{j}} \Bigr\}_{j=1}^N 
\;\cup\; \bigl\{ \kappa_\ell \bigr\}_{\ell=1}^L 
\;\cup\; \{ \textsl{g}_\triangleright \}. 
\end{equation}
Each trajectory node $j$ consists of a partial observation $o_{\tau_j}$ and an action $a_{\tau_j}$. 
It is also associated with a goal term $\zeta_j \in \{\textsl{g}_j, \kappa_{\ell}^{\textsl{g}_j}\}$ indicating either a final goal $(\textsl{g}_j)$ or an abstract subgoal $(\kappa_{\ell}^{\textsl{g}_j})$ that the trajectory is intended to complete. 
In addition, the node stores an estimated (sub)goal reward $\hat r_j$ for the action sequence, reflecting progress toward completing its associated subgoal, and a confidence score $c_j$ derived from the LLM-derived statistics.
The second set of nodes $\{\kappa_\ell\}_{\ell=1}^L$ represents subgoals $\kappa_\ell$ provided by the LLM from the environment description. 
The final term $\{\textsl{g}_\triangleright\}$ denotes the agent’s target goal(s).

The graph is initialized with offline LLM priors and evolves as training progresses. 
Nodes are added or updated when the agent \n{produces trajectory segments that either introduce a novel segment for a known (sub)goal or achieve a higher estimated return than the existing entry for the same (sub)goal.}
Existing nodes can also be updated when the agent’s experience strengthens entries that were initially derived from low-confidence LLM outputs.
Online LLM suggestions may also be added as new nodes when available, provided they pass screening; \n{an online query is triggered only when the rollout utility signal remains near zero for several consecutive episodes, indicating that the current graph offers no helpful guidance.  In such cases, the current policy is either exploring regions unsupported by the memory graph or that the graph itself is missing useful segments. } 
Nodes are pruned when they remain unused for a fixed horizon of episodes. \n{Each node maintains an access counter updated whenever it contributes to the utility; nodes whose counters do not change within this window are removed}, reflecting reduced relevance to recent rollouts.
Offline LLM nodes, though generally stable, may also be removed when rendered obsolete. 
\n{Additional details about graph construction and environment-specific examples are given in Appendix~\ref{graph_detail}.}
\n{This process keeps the graph compact and adaptive, while ensuring that any low-quality or misleading prior segments naturally fade as higher-return agent-generated trajectories replace them or lead to their pruning.}

\subsection{Offline and Online Guidance} \label{roles}
MIRA incorporates two complementary forms of LLM guidance, accessed either \textit{offline} prior to training or \textit{online} during training.
Offline outputs are \ch{generated using full access to the task description, providing trajectory segments and subgoal decompositions that initialize the memory graph with structured priors. }
Offline nodes accelerate early exploration and remain a persistent baseline source of guidance that complements the adaptive updates introduced by online LLM queries.

Online suggestions are incorporated during training when the agent fails to obtain useful guidance \ch{(i.e., when the rollout utility remains near zero)} from its memory graph for several consecutive episodes.
The LLM is constrained to the same partial observability as the agent and, when triggered, may return plans that correspond to short trajectories.
Alternatively, it may provide control signals that bias the action preferences over an extended horizon until the current task segment is completed.
All online outputs are passed through the \textit{Screening Unit}, which discards low-confidence suggestions.
Accepted plans are grafted into the memory graph as new trajectory segments, while accepted control signals bias the policy through soft logit injection, i.e., adding a bounded penalty to the logits of discouraged actions. \n{This penalty induces only a soft preference before the softmax and cannot collapse the action distribution. PPO’s clipped objective controls the update size, ensuring that the injected bias functions as lightweight guidance that the critic can override when it strongly disagrees with the value estimates}~\citep{biza2021action}.

\textbf{\textsc{Screening Unit.}}
To ensure reliability, online outputs are passed through a lightweight \textit{Screening Unit} designed to reduce hallucinations and reasoning failures~\citep{ji2023survey, bubeck2023sparks, wang2022self, zhao2021calibrate}. 
Confidence is estimated in two complementary ways.
When token-level likelihoods are available, \n{we compute the geometric mean of the per-token probabilities of the completion as a confidence measure}.
When such likelihoods are unavailable or incomplete (e.g., only top-k likelihoods are provided), we instead estimate confidence by sampling multiple independent completions \n{and measuring agreement across them. A majority-consistency test is then applied to retain only those outputs that appear reliably across the sampled completions.}
Suggestions that fail to meet a fixed threshold under either criterion are discarded. 
While this procedure does not eliminate all high-confidence errors, it serves as an effective filter that reduces the risk of hallucinated or low-quality outputs. 
The screened outputs, referred to as \textit{healthy grafts} in Figure~\ref{fig:mira}, are incorporated into the memory graph as new nodes to further help the policy learning.

% {Together, offline priors and online grafts allow MIRA to combine stable, precomputed knowledge with adaptive updates, reducing dependence on continuous supervision while maintaining the benefits of structured LLM guidance. }

\subsection{Utility Signal Computation} \label{utility}
\ch{Utility is defined at the level of individual state–action pairs and is computed from the same rollouts used for advantage estimation under the current policy.}
Each state-action in the trajectory $\tau = \{(o_t, a_t)\}_{t=1}^T$ is matched against the corresponding state–action pairs  $(o_{t'}, a_{t'})$ in the stored trajectory $\tau_m$. 
The appropriate memory node $m$ is selected based on the environment instance (e.g., the seed-specific layout) in that training iteration. We then compute the utility for each pair $t$ as:
\begin{equation} 
\label{utility_eq}
       U_t \doteq  c_m \cdot \hat{r}_{m} \cdot \rho(\textsl{g}_\triangleright, \zeta_{m}) \cdot  \mathcal{s}\big((o_t,a_t), (o_{t'},a_{t'})_{\tau_m}\big).
\end{equation}
\ch{Steps that do not match any stored trajectory segment receive zero utility.}
The similarity function $ \mathcal{s}(\cdot, \cdot)$ measures how closely the agent’s behavior aligns with the stored trajectory. 
It incorporates both action agreement and spatial consistency, such as overlap in grid positions or directional alignment in tabular settings \n{(Algorithm~\ref{alg:sim})}.  
To capture semantic structure, the raw similarity score is weighted by a goal-alignment factor $\rho(\cdot,\cdot)$. \n{Each subgoal description specifies a target object or region and a high-level action applied to it; simple rule-based parsing yields a paired entity and action-phase token. } The Jaccard similarity between the entity–phase token sets extracted from the agent’s target subgoal and each memory entry defines $\rho$. This weighting increases the influence of memory entries that share underlying entities or action phases with the target subgoal and downweights matches corresponding to unrelated parts of the task \n{(Algorithm~\ref{alg:jaccard})}.
\ch{Thus, a transition contributes to utility only when both its behavioral similarity and its semantic alignment with the relevant subgoal are high.}
Finally, the score is modulated by the confidence $c_m$ and estimated reward $\hat r_m$ attached to the memory node (Algorithm~\ref{alg:util}). 
\n{This formulation also helps limit the influence of incorrect LLM guidance that may have passed screening. When an LLM suggestion is inaccurate, its influence naturally diminishes: such segments typically contribute little utility since their estimated reward is low or their similarity score remains small, reflecting weak positional alignment or inconsistent action patterns between the agent’s rollout and the stored LLM-derived segment. Further details on the utility computation are presented in Appendix~\ref{utility_detail}.}

\subsection{Adaptive Advantage Shaping} \label{shaping}

We incorporate memory-derived utility into the policy update by augmenting the standard advantage term. 
Algorithm~\ref{alg:ppo_ours} outlines the shaped PPO update.  At iteration $k$, trajectories $\mathcal{D}_k = \{(s_t,a_t,r_t)\}$ are collected under the policy $\pi_{\theta_k}$.  The rollout batch is split into minibatches $\mathcal{B}$ for multiple gradient steps. The likelihood ratio $r_t$ compares new and old policies, and the clip parameter $\varepsilon_k$ constrains $r_t$ within $(1 \pm \varepsilon_k)$ as a soft trust region.

The advantage function in policy gradient methods, denoted by $A_t$ at a given time $t$, quantifies how favorable an action $a_t$ is relative to the average action at state $s_t$.  It drives learning by reinforcing actions that have higher-than-expected returns and suppressing those that fall short. 
\begin{wrapfigure}{r}{0.5\textwidth}  
\vspace{-22pt}  
\begin{minipage}{\linewidth}
\begin{algorithm}[H]
\caption{Shaped PPO actor 
(\textcolor{mpurple}{changes})}
\label{alg:ppo_ours}
\begin{algorithmic}
\FOR{$k = 0,1,\dots$}
\STATE Collect  $\mathcal{D}_k = \{(s_t, a_t, r_t)\}$ using $\pi_{\theta_k}$
\STATE Compute ${A}_t$ and \textcolor{mpurple}{$U_t$} from rollouts \label{line:adv_shaping}
\STATE ${\textcolor{mpurple}{\tilde{A}_t}}= \textcolor{mpurple}{\eta_t}{A}_t + \textcolor{mpurple}{\xi_t U_t}$
\FOR{$epoch = 1$ to $K$}
    \FOR{minibatch $\mathcal{B} \subset \mathcal{D}_k$}
        \STATE $r_t (\theta) = {\pi_\theta(a_t|s_t)}/{\pi_{\theta_k}(a_t|s_t)}$
        \STATE \textcolor{mpurple}{$\mathcal L^{\text{shaped}}(\pi_\theta)$} $= \mathbb{E}\left[\min(r_t,1\!\pm\!\varepsilon_k){\textcolor{mpurple}{\tilde{A}_t}} \right]$
        \STATE $\theta \leftarrow \theta + \alpha_\theta \nabla_\theta$ \textcolor{mpurple}{$\mathcal L^{\text{shaped}}(\pi_\theta)$}
    \ENDFOR
\ENDFOR
\ENDFOR
\end{algorithmic}
\end{algorithm}
\vspace{-25pt}
\end{minipage}
\end{wrapfigure}
However, during early training the critic is poorly calibrated due to limited exploration, often producing nearly uniform value estimates across actions~\citep{henderson2018deep}.
As a result, the estimated advantages $A_t$ provide weak learning signals, even when the agent is following behavior that is meaningfully directed toward the task.
This issue is particularly pronounced in sparse-reward settings or tasks with delayed feedback,  where the critic lacks sufficient signal to distinguish between promising and unproductive behaviors.  
In such cases, the estimated advantage tends to be near-zero or highly noisy for most timesteps, especially early in training (Figure~\ref{fig:early}).

To address this, we introduce a shaped advantage as:
\begin{equation}
\tilde{A}_t = \eta_t A_t + \xi_t U_t,  \quad
    0 < \eta_t \leq 1,\; \xi_t \leq \delta \eta_t,\; \delta \in [0,1),\; 
\lim_{t \to \infty} \eta_t = 1,\; \lim_{t \to \infty} \xi_t = 0.
\end{equation}
This formulation preserves the fundamental role of the advantage function, while refining it with utility-based guidance. 
\ch{It forms a cooperative process between critic predictions and the memory-derived utility.}
The critic provides an estimate based on learned reward prediction and bootstrapping, while the utility term injects an inductive bias derived from language-guided priors.
Together, they form a joint estimator in which each component compensates for the other’s limitations without distorting policy optimization.
When the critic signal is weak due to insufficient value discrimination, the resulting gradients are uninformative and impair the agent’s ability to bootstrap from sparse or delayed rewards.
The utility term provides additional directional guidance aligned with task objectives, accelerating learning by compensating for weak or flat gradients \n{(Figure~\ref{fig:early}, Appendix~\ref{ret_detail})}. 

As training progresses and the policy becomes more accurate, the critic’s advantage estimates $A_t$ become more reliable.
\n{Accordingly, the shaping weight $\xi_t$ is annealed to reduce the direct influence of the utility term and $\eta_t$ is ramped toward 1 over training.
Since the LLM-derived signals can be imperfect, annealing $\xi_t$ prevents inaccuracies in those signals from being preserved in the asymptotic regime. This ensures that the final policy is optimized with respect to the true reward function $\mathcal{R}$ and remains consistent with PPO’s stability guarantees. Early in training, however, $\xi_t$ remains large enough for the utility to accelerate exploration using the LLM’s possibly imperfect, but still useful, prior knowledge. 
As $\xi_t$ decays, any suboptimalities in these priors are naturally learned away, yielding the benefits of shaping during the sparse-reward phase, without biasing long-run behavior} or altering the policy or critic structure. See remark~\ref{bv} for more explanation.

Before turning to experiments,  we provide an interpretive perspective on how adaptive advantage shaping affects optimization in sparse-reward regimes and in Appendix~{} establish that the proposed shaping mechanism preserves the policy improvement property of PPO under standard boundedness and scaling assumptions, which we formally enumerate in Appendix~\ref{assu}.  More broadly, the method remains compatible with policy gradient algorithms that relies on advantage estimation, offering a general mechanism for integrating language-derived priors into RL.

\begin{thm}[Non-Vanishing Updates in Sparse-Reward Regimes]
~\label{thm:non_vanishing}
Define the shaped surrogate $\mathcal{L}^{{shaped}}(\theta)
\doteq
\mathbb{E}\!\left[
\nabla_\theta \log \pi_\theta(a_t|s_t)\, \tilde A_t
\right]$, and the PPO surrogate $\mathcal{L}^{{ppo}}(\pi) = 
\mathbb{E}\!\left[
\nabla_\theta \log \pi_\theta(a_t|s_t)\, A_t
\right].$ Consider a training iteration $k$ such that the expected magnitude of the PPO
advantage is small, i.e., $\mathbb{E}[|A_t|] \le \varepsilon_A$ for some
$\varepsilon_A \approx 0$.
Under Assumptions~\ref{ass:bounded}--\ref{ass:tr}, the expected norm of the
shaped PPO policy update satisfies
$$\bigl\| \mathcal{L}_k^{\mathrm{shaped}} \bigr\|
\;\ge\;
\xi_k \, \bigl\| \mathcal{L}_k^{U} \bigr\| - O(\varepsilon_A),$$
where
$\mathcal{L}_k^{U}
\doteq
\mathbb{E}\!\left[
\nabla_\theta \log \pi_\theta(a_t|s_t)\, U_t
\right].$

\end{thm}

\noindent 

\textit{Proof.} Deferred to Appendix~\ref{thm}.

% \n{The result shows that each policy update made under the shaped surrogate yields a surrogate improvement that is at least as large as the improvement PPO would obtain using its standard surrogate at the same update step.
% Aggregated over K iterations, the total surrogate progress under MIRA dominates that of PPO, which explains the accelerated learning observed empirically.}

\section{Experimental Setup} \label{setup}
   \begin{wrapfigure}{r}{0.38\textwidth}
  \vspace{-12pt}
  \centering
  \includegraphics[width=\linewidth]{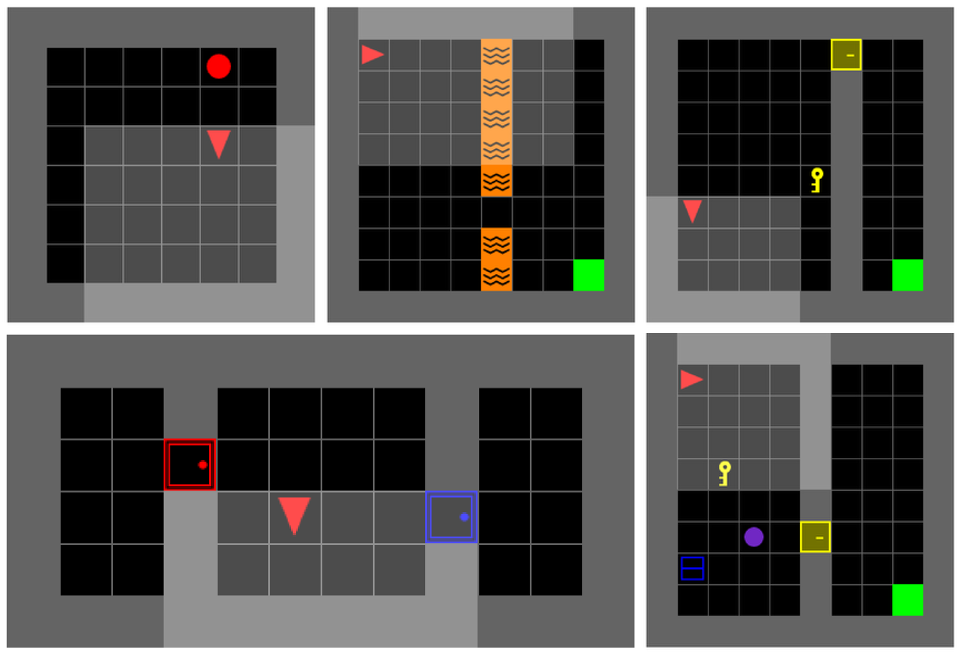}
  \caption{Evaluation environments. Top: \textsc{RedBall} (navigation to target), \textsc{LavaCrossing} (long-horizon navigation with irreversible hazards), \textsc{DoorKey} (sparse reward with key–goal dependency).
  \\
  Bottom: \textsc{RedBlueDoor} (sequence-sensitive toggling), \textsc{Distracted DoorKey} (distractor-rich variant with key-goal dependency).}
  \label{fig:envs}
  \vspace{-37pt}
\end{wrapfigure}
We validate our method through extensive experiments implemented using the RLlib~\citep{liang2018rllib}. Our evaluation focuses on performance gains, sample efficiency, and the computational overhead introduced by LLM integration. 
The objective is to characterize the benefits and trade-offs of incorporating LLM guidance in RL, including how different levels of LLM capabilities influence the policy learning dynamics and final policy quality.

\subsection{Simulation Platform}
 We consider six distinct environments, which are selected to span discrete vs. visual inputs, short- vs. long-horizon dependencies, reversible vs. irreversible dynamics, and with vs. without perceptual distractors, forming a compact yet representative benchmark for sparse-reward RL.

\textsc{{Gymnasium ToyText.}} Gymnasium~\citep{arnold2024gymnasium} provides simple tabular environments for controlled analysis of learning dynamics in low-dimensional settings.  Despite their simplicity, these environments feature sparse rewards and require strategic exploration, making them suitable for isolating the early-stage benefits of memory-guided utility shaping. We include \textsc{FrozenLake} as a minimal benchmark where PPO reliably converges to the optimal policy, enabling us to verify that MIRA preserves convergence while accelerating early learning.

\textsc{MiniGrid and BabyAI.} MiniGrid~\citep{chevalierboisvert2023minigrid} and BabyAI~\citep{chevalierboisvert2019babyai} are suites of lightweight, procedurally generated environments designed to evaluate exploration and planning in partially observable, sparse-reward settings. 
We use these tasks to assess the effectiveness of advantage shaping in long-horizon decision-making environments that require reasoning under uncertainty and robustness to irrelevant stimuli. We include five tasks, selected to cover diverse challenges involving planning, credit assignment, and distraction resilience (Figure~\ref{fig:envs}). 
We use pixel-based observations (RGB images) rendered from the environment as the policy inputs, to introduce perceptual complexity and evaluate agent performance under a more realistic observation setting.

% Appendix~\label{prompt} provides more details on our experiment environments and the LLM prompts used for each environment.

\subsection{Baseline Methods}

% Appendix~\ref{rep} provides detailed specifications of all hyperparameters for each method.

\textsc{PPO (RL baseline). } We train a tabula rasa PPO agent~\citep{schulman2017ppo} that learns purely from environment interaction and rewards. Network architecture, PPO hyperparameters, and rollout settings are held fixed across all methods for fair comparison.

\textsc{Hierarchical RL. } We include hierarchical reinforcement learning (HRL)~\citep{matthews2022hierarchical} as a baseline that uses pre-trained LLM option policies for temporal abstraction.

\textsc{LLM-RS. } We consider the method of \citep{bhambri2024extracting}, which we refer to as LLM-RS. This approach queries the LLM in real time to generate plans for potential-based reward shaping,  with a verifier refining them for valid action sequences. 
% \carlee{Does this use online LLM queries}

\textsc{LLM4Teach. } We include LLM4Teach~\citep{zhou2023large} as a representative teacher-based approach. It employs a pre-trained LLM as a policy teacher and guides the RL agent through policy distillation, and is among the state-of-the-art methods in this category.

\section{Experimental Results} \label{res}

\ch{To assess how effectively MIRA addresses the challenges raised in the introduction, we structure our core experiments around three guiding questions. These questions examine whether utility-shaped advantages compensate for PPO’s weak early gradients in sparse-reward, long-horizon, and partially observable settings, and whether amortizing limited LLM guidance into a persistent memory structure provides sufficient return while supporting scalability and avoiding latency. Formally:

\begin{itemize}
    \item \textbf{Q1}: How does MIRA improve early learning efficiency and convergence relative to PPO, even when PPO’s final performance is competitive?
    \item \textbf{Q2}: How well does MIRA perform in long-horizon environments that demand extended exploration and multi-step reasoning?
    \item \textbf{Q3}: How effectively does MIRA translate a limited number of LLM queries into performance gains compared to query-heavy methods?
\end{itemize}

To isolate the contribution of individual components in MIRA’s design, we further conduct ablations organized around three additional questions. These probe how online queries complement offline-initialized memory, and how sensitive MIRA is to degraded, incorrect, or stylistically varied LLM outputs. Formally:

\begin{itemize}
    \item \textbf{Q4}: How do online LLM queries improve learning, beyond what offline memory provides?
    \item \textbf{Q5}: How does MIRA handle late-stage exposure to degraded LLM guidance once its memory is well-formed?
    \item \textbf{Q6}: How do variations in LLM reasoning affect memory and downstream results?
\end{itemize}

}

Appendix~\ref{ret_detail} provides additional results, including evaluations on unseen seeds to assess generalization, \n{wall-clock analyses of LLM-query overhead (Figure~\ref{fig:wall}), and measurements of memory growth (Figure~\ref{fig:memory}). We also report ablations on the screening threshold (Figure~\ref{fig:threshold}) and prompt wording (Figure~\ref{fig:twoprompt}), each examining its effect on overall learning behavior and addressing aspects of Q5.} In addition, supplementary plots from sweeps over shaping weights (Figure~\ref{fig:early}) are included to analyze their impact on early-stage learning dynamics and reward progression.\n

\subsection{Tabular Benchmark  and Partially Observable Tasks}

\textbf{Q1:} We evaluate \n{two variants of MIRA on \textsc{FrozenLake-8x8}, an offline-only version and an online-only version,} and compare them to the PPO baseline, averaging results over four seeds. 
In the offline variant, three zero-shot GPT-o4-mini queries generate an initial memory graph. The LLM observes the grid layout (matching the agent’s full observability) but does not receive the slipperiness probability, which is hidden from both.
As shown in Figure~\ref{fig:frozen_results}, this initialization results in faster early learning, and the offline variant maintains a higher return than PPO throughout, and the online-only variant in the first 1K iterations. The online-only variant begins with an empty memory graph, without any global information about the map and issues LLM queries during training.

 \begin{wrapfigure}{r}{0.5\textwidth}
  % \vspace{-7pt}
  \centering
  \includegraphics[width=\linewidth]{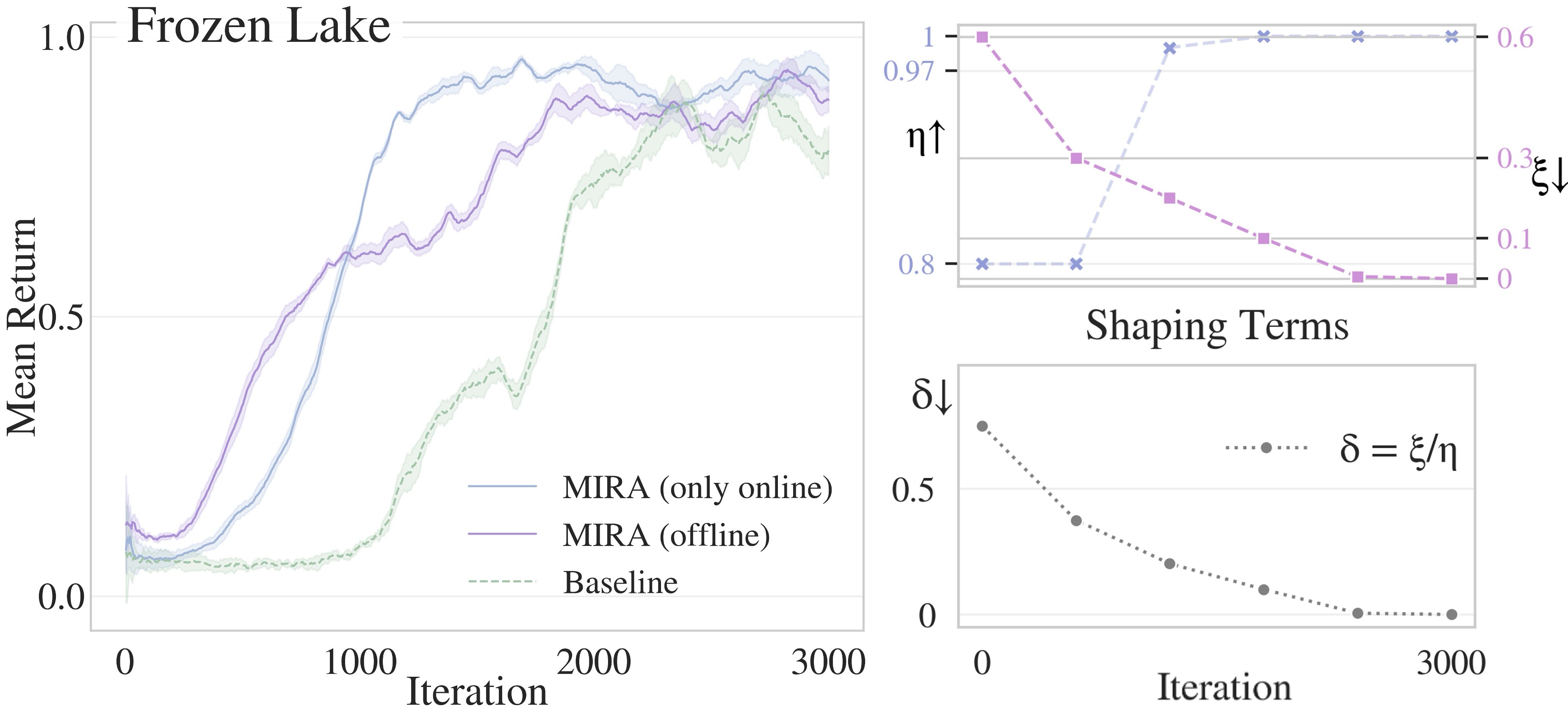}
  \caption{Mean return on \textsc{FrozenLake} (left): \n{Both MIRA variants improve early-stage learning relative to PPO, while PPO eventually attains a comparable asymptotic return.} Evolution of shaping terms $\eta_t$, $\xi_t$, and ratio $\delta_t$ (right): $\delta_t$ decays during training, ensuring convergence as $\delta_t \to 0$.}
  \label{fig:frozen_results}
  \vspace{-10 pt}
\end{wrapfigure}
S\n{ince the layout is deterministic and the action set is small, these online queries can infer short segments from the agent’s rollouts and populate the memory graph as training progresses.}
This provides a faster improvement rate than the offline variant, though it requires more LLM queries.
PPO eventually catches up, and by convergence, the asymptotic returns of all three methods become comparable.
During training, the shaping signal primarily affects the early iterations. As the policy improves, $\eta$ increases, $\xi_t$ decreases, \ch{i.e. the shaping ratio $\delta_t=\xi_t/\eta_t$ decays, reducing reliance on memory.} Under standard stochastic approximation conditions~\citep{kushner2003stochastic}, this decay keeps the critic error within an $O(\delta_t)$ neighborhood of the true value, which vanishes as $\delta_t \to 0$.

\textbf{Q2:} We next evaluate MIRA on five tasks designed to isolate distinct challenges in sparse and partially observable environments. 
Figure~\ref{fig:all_results} shows mean return and success rate across the four tasks, with performance averaged over four different seeds. 
In simpler tasks such as \textsc{RedBall}, PPO shows moderate early gains but plateaus well below optimal performance. Although hierarchical RL eventually catches up, MIRA reaches optimal returns in under half the training iterations. In \textsc{LavaCrossing}, PPO fails to improve beyond near-zero success, indicating ineffective exploration. Hierarchical RL improves steadily but converges more slowly than MIRA. In more complex tasks such as \textsc{DoorKey} and \textsc{RedBlueDoor}, MIRA achieves substantially higher success rates, approximately twice those of HRL, while also converging faster.
\begin{figure*}
\vspace{-9pt}
    \centering
    \includegraphics[width=\textwidth]{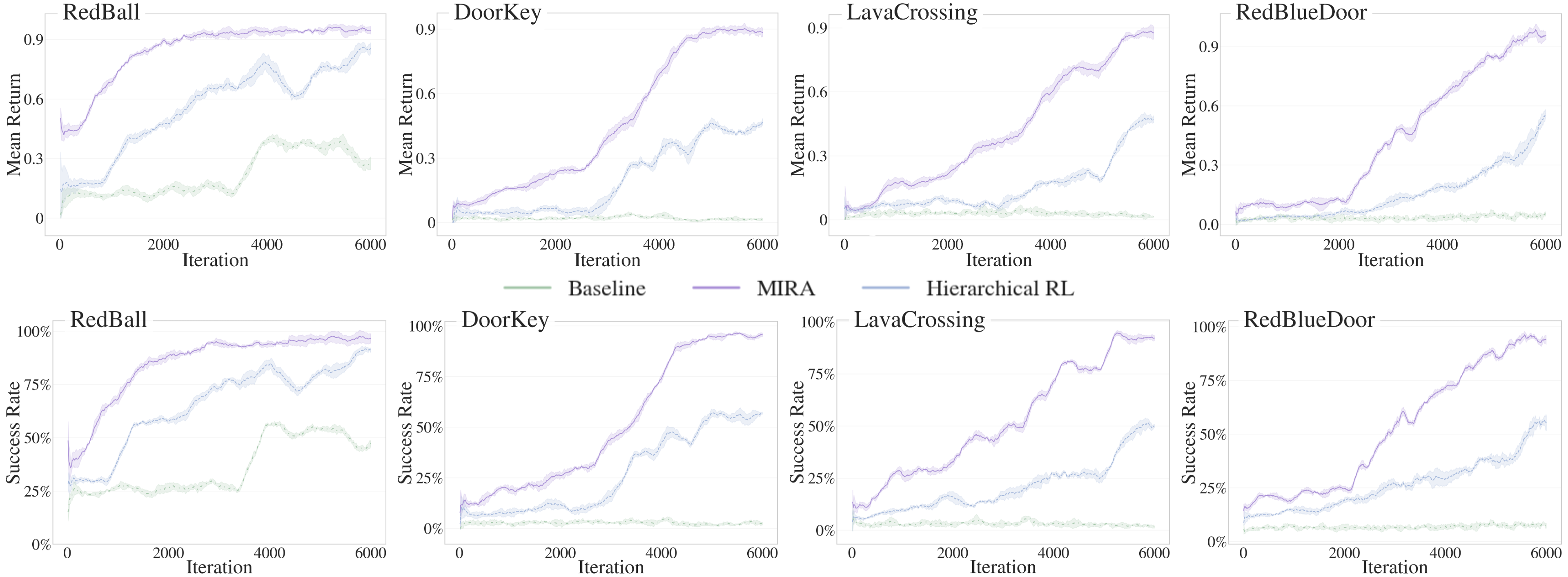}
    \caption{Mean return (top) and success rate (bottom) across four MiniGrid and BabyAI tasks. MIRA consistently outperforms both baselines, achieving faster learning, higher asymptotic return, and greater success rates. These results are obtained with a small LLM budget, using fewer than ten offline prompts to build memory graphs plus infrequent online queries to guide exploration.}
    \label{fig:all_results}
\vspace{-10pt}
\end{figure*}

These gains are achieved with a limited LLM query budget that combines offline and infrequent online access. Offline queries scale with task complexity. In \textsc{RedBall}, four zero-shot prompts to GPT-o4-mini are sufficient to build a useful memory graph, while \textsc{DoorKey} requires about seven queries that mix few-shot and zero-shot prompts. 
Online queries are budgeted separately and also vary with task complexity. In \textsc{RedBall}, about seven online queries suffice to suppress irrelevant actions throughout training. In \textsc{RedBlueDoor}, queries are triggered more frequently early in training to help interpret partial observations and suggest short sequences, such as turning, that align the agent with the door. Once the red door is discovered and toggled, the offline memory becomes sufficient. In this task, rooms behind the doors serve only as distractors; baseline agents, including hierarchical RL, often waste steps exploring them. As shown in Figure~\ref{fig:all_results} (lower right), HRL achieves higher success rates than PPO but yields similar average return in the beginning due to suboptimal trajectory use. By contrast, MIRA avoids such inefficiencies by focusing on goal-relevant behavior earlier in training.

\begin{wrapfigure}{r}{0.4\textwidth}
  % \vspace{-15pt}
  \centering
        \includegraphics[width=\linewidth]{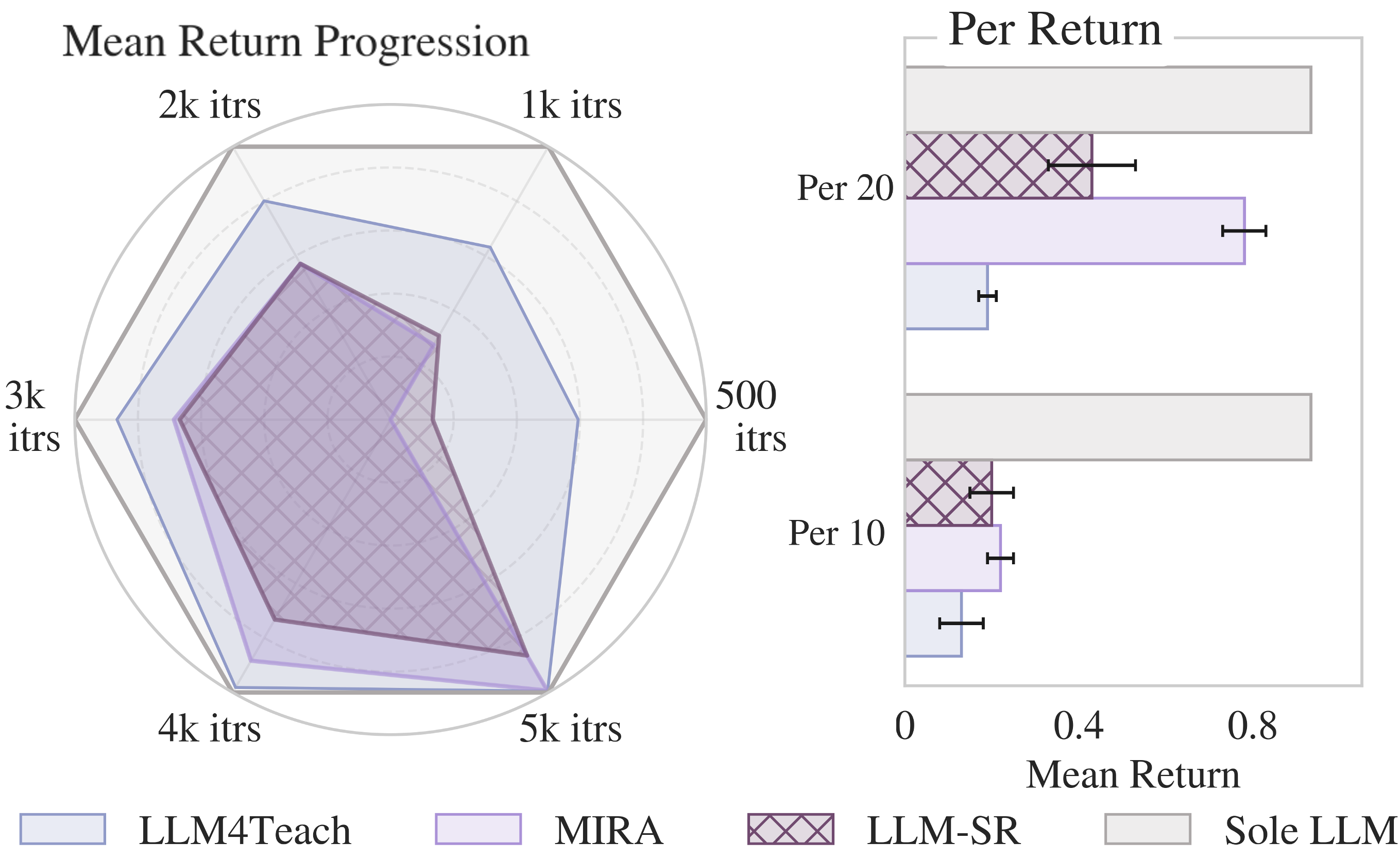}
\caption{Mean return (left): LLM4Teach shows faster early gains, while MIRA steadily improves and matches its final return with fewer queries. LLM-RS benefits early from reward shaping but plateaus lower.
Return per LLM query (right): Under two query budgets, MIRA achieves higher efficiency.}
\vspace{-10pt}
    \label{fig:llm4teach_comparison}
\end{wrapfigure}
\textbf{Q3:} To further evaluate MIRA, we compare it to LLM4Teach and LLM-RS in the custom variant environment \textsc{Distracted DoorKey}. 
We also include a Sole LLM baseline, where GPT-o4-mini executes plans under full observability without learning. 
Figure~\ref{fig:llm4teach_comparison} shows mean return progression at selected training checkpoints.  For Sole LLM, we report average return over 100 seeds to demonstrate that the task is LLM-solvable and that its outputs provide useful structural guidance. The accompanying bar chart reports amortized return per cumulative LLM query under two fixed budgets \n{(10 and 20 queries, both below the usage of any method)},  quantifying how each method translates queries into performance gains.
% \vspace{-5pt}

MIRA achieves higher query efficiency than both LLM4Teach and LLM-RS, \n{using about $30$ queries per run} (seven offline and $20 \pm 3$ online) to obtain higher return per query.
In contrast, LLM4Teach issues dense supervision, querying the LLM \n{once} on every state–action–reward triplet within training batches, \n{which in our experimental setup corresponds to  more than $500$ queries} before the policy stabilizes. 
LLM-RS, which uses LLM-generated reward code, queries once per layout, totaling over $50$ calls in our setup. While lighter than LLM4Teach, this still requires layout-level access throughout training. Despite its heavier budget, LLM4Teach achieves  comparable final performance to MIRA, while LLM-RS fails to match MIRA’s return. Notably, LLM-RS outpaces MIRA early due to reward shaping, but falls behind later. 
LLM4Teach shows an early advantage through front-loaded queries, but at a significantly higher cost. Tables in Appendix~\ref{ret_detail} report results on unseen evaluation seeds to assess generalization.

% \footnotetext{There is no statistically significant difference between MIRA and LLM4Teach at the 95\% confidence level.}

\subsection{Ablation Studies} \label{sub:ablation}
\textbf{Q4 }\textsc{(Online Query Frequency)}\textbf{:} \label{q_freq}
We vary the number of online LLM queries issued during training of \textsc{DoorKey}, to assess how constrained usage affects learning efficiency and final performance. Each agent begins with the same memory graph, initialized from identical offline queries, isolating the contribution of dynamic LLM input from that of static memory initialization. We  compare MIRA under three online budgets: zero, a mid budget of 10 queries, and a high budget of 20. As shown in Figure~\ref{fig:all_abs} (left), more frequent online access accelerates learning, with the large-budget variant achieving optimal return in fewer steps (Table~\ref{tab:freq_res_three}, \n{ Appendix~\ref{ret_detail}}). Even with just 10 online calls, MIRA substantially outperforms the offline-only variant. Nevertheless, MIRA (offline) still yields a notable boost over baseline PPO, indicating that static memory alone can provide meaningful guidance when dynamic access is unavailable.
\n{Given this noticeable improvement over PPO, the offline-only variant is a practical choice when run-time LLM access is restricted or when latency makes online queries impractical, since it relies entirely on the memory graph constructed offline. When limited run-time access is available, the online variant becomes preferable because dynamic queries update the memory during exploration and accelerate learning, at the cost of additional wall-clock time associated with LLM calls (Figure~\ref{fig:wall}, \n{Appendix~\ref{ret_detail}}).}

\begin{figure}
\vspace{-10pt}
    \centering
    \includegraphics[width=\linewidth]{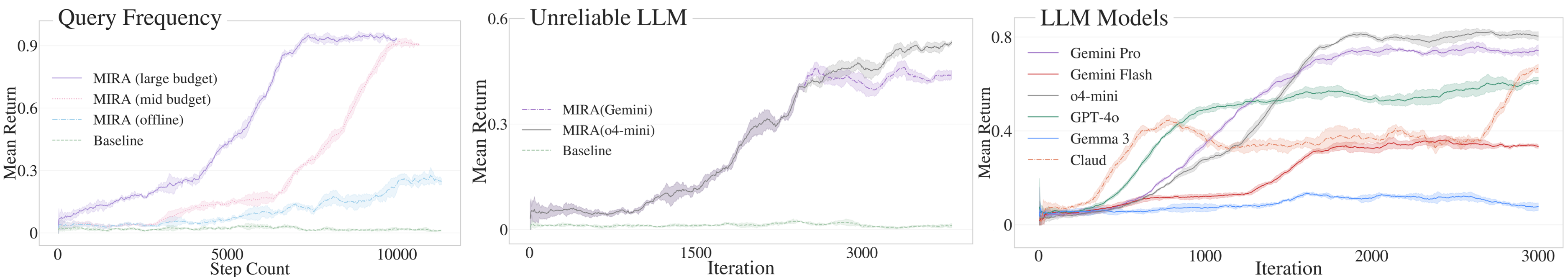}
    \caption{\textbf{Ablation Studies}.
Query frequency (left): Agents share the same offline memory but vary in online budgets. More queries accelerate learning, with high-budget agents achieving optimal returns more quickly.
Unreliable LLM (middle): With identical offline memory, screening is disabled and queries are swapped from GPT-o4-mini to Gemini Pro only in the late phase.  Performance remains stable in the late phase, indicating reduced dependence on online guidance once policy have matured. 
LLM models (right): Agents differ only in the LLM used for memory.  Performance differences reflect divergent reasoning styles: Gemma3 induces inefficient checking, Claude favors cautious exploration, while Gemini Pro and o4-mini enable faster learning.}
\vspace{-8pt}
    \label{fig:all_abs}
\end{figure}

\textbf{Q5 }\textsc{{(Unreliable LLM Outputs)}}\textbf{:} \label{unr}
We evaluate a scenario where the LLM is swapped at a later training stage and the screening unit is disabled only for this final online phase in \textsc{Distracted DoorKey} environment. All agents share the same offline-initialized memory graph and use GPT-o4-mini with screening during earlier online queries. In the final stage, we replace the LLM with Gemini Pro and omit screening. By this point, MIRA has accumulated sufficient experience and memory, allowing it to tolerate low-confidence or incorrect suggestions without collapsing performance. We prompted both LLMs with a scenario where the agent has already explored \textit{thoroughly} and confirmed no key is present (implying it was already collected, since inventory is \textit{unobserved}).
When asked which action to down-weight, GPT-o4-mini gave a consistent suppression, whereas Gemini returned a misaligned alternative. As shown in Figure~\ref{fig:all_abs} (middle), MIRA remains stable under degraded guidance, though convergence slows and final return drops slightly. Details of the LLM responses are shown in Figure~\ref{fig:gptgemm}, \n{Appendix~\ref{prompt}}.

\textbf{Q6 }\textsc{{(Reasoning and Performance)}}\textbf{:}
We assess MIRA's sensitivity to the choice of language model by replacing GPT-o4-mini with alternatives such as GPT-4o~\citep{openai2024gpt4o}, Claude Sonnet 4~\citep{anthropic2024claude4}, Gemma 3 27B~\citep{ananthanarayanan2024gemma}, Gemini 2.5 Flash and Pro~\citep{chen2024gemini}. All models go through the same process to ensure comparability. Unlike the ablation done before, the model swap is applied from the beginning of training.
As shown in Figure~\ref{fig:all_abs} (right), the reasoning style shaping the memory graph strongly impacts downstream RL performance. 
For instance, Gemma3 performs poorly because it recommends checking the door after every pickup, leading to wasteful steps. 
Claude adopts an exploratory policy that yields slow but eventual progress, showing early improvement followed by plateauing, but it eventually recovers as the memory is dynamic.
GeminiPro and GPT-o4-mini both enable fast early learning, but o4-mini’s memory includes \textit{detours} that prove beneficial later, ultimately reaching the highest asymptotic return. 
These differences highlight how the reasoning processes behind LLM outputs directly influence MIRA's long-term policy quality. Figure~\ref{fig:llms} in \n{Appendix~\ref{prompt}} presents the reasoning traces produced by different LLM models used.

\section{Conclusion and Future Work}~\label{conc}
We propose MIRA, a reinforcement learning (RL) framework that integrates large language model (LLM) guidance via a memory graph built from high-return trajectories and LLM-inferred information. By shaping advantage estimates with a utility signal derived from this memory, MIRA accelerates early learning without requiring continuous LLM supervision. Theoretical and empirical results on sparse-reward tasks confirm improved sample efficiency and preserved convergence.  Limitations of the current design are discussed in Appendix~\ref{lim}.
\n{Extending MIRA to continuous-control and vision-based domains is a natural next step. Prior LLM-guided robotics work shows that semantic subgoals can be grounded in continuous spaces through learned affordances or visual embeddings~\citep{brohan2023can, huang2022inner}. In such settings, MIRA’s discrete similarity function can be replaced with comparisons in a latent representation space, e.g., using encoder features or R3M-style embeddings~\citep{nair2022r3m}, while memory scalability can be maintained through clustered or hierarchical organization. These adaptations would allow MIRA’s memory graph and utility-based shaping to apply beyond grid worlds and support longer-horizon tasks with richer state spaces.} Extending MIRA to multi-goal domains such as Crafter~\citep{hafner2023mastering}, where reusable subgoal structure is prominent, is another promising direction.

% Future work includes extending MIRA to continuous action spaces and multi-goal domains like Crafter~\citep{hafner2023mastering}, where long-horizon dependencies and reusable subgoal structure are prominent. We expect that MIRA’s evolving memory and advantage shaping will be especially valuable in such settings, supporting both reuse and abstraction across episodes.

\section*{Acknowledgments}
This work was supported in part by the National Science Foundation (NSF) under grants CNS-2533813 and CNS-2312761 and the Office of Naval Research under grant N000142412073.

\section*{Reproducibility Statement}
 We have taken several steps to ensure reproducibility of our results. 
All theoretical assumptions and complete proofs are included in Appendix~\ref{thm}.
Appendix~\ref{prompt} details the environment specifications and the exact LLM prompts used for both offline and online queries. 
Appendix~\ref{rep} lists the full set of hyperparameters for MIRA across every evaluated environment. 
We also provide pseudocodes for all proposed algorithms in Algorithms~\ref{alg:ppo_ours} through ~\ref{alg:util}, ensuring clarity and transparency despite their straightforward implementation.  The code is available at paper's webpage \url{https://narjesno.github.io/MIRA/}.
Together, these materials supply all information necessary to reproduce our experiments and verify the claims of the paper.

% \pagebreak
\bibliography{iclr2026_conference}

@article{ji2023survey,
  title={A Survey of Hallucination in Natural Language Generation},
  author={Ji, Ziwei and Lee, Nayeon and Frieske, Richard and Yu, Tiezheng and Su, Dan and Xu, Yanlin and Ishii, Etsuko and Bang, Yeon Seon and Madotto, Andrea and Fung, Pascale},
  journal={ACM Computing Surveys (CSUR)},
  volume={55},
  number={12},
  pages={1--34},
  year={2023},
  publisher={ACM}
}

@article{zhao2021calibrate,
  title={Calibrate Before Use: Improving Few-Shot Performance of Language Models},
  author={Zhao, Eric and Wallace, Eric and Feng, Shi and Klein, Dan and Singh, Sameer},
  journal={arXiv preprint arXiv:2102.09690},
  year={2021}
}

@inproceedings{bubeck2023sparks,
  title={Sparks of Artificial General Intelligence: Early experiments with GPT-4},
  author={Bubeck, Sébastien and Chandrasekaran, Varun and Eldan, Ronen and Huster, Bruno and Raghunathan, Ananth and Ray, Jonathan and Zhang, Yang},
  booktitle={arXiv preprint arXiv:2303.12712},
  year={2023}
}

@book{sutton1998reinforcement,
  title={Reinforcement learning: An introduction},
  author={Sutton, Richard S and Barto, Andrew G and others},
  volume={1},
  year={1998},
  publisher={MIT press Cambridge}
}

@inproceedings{liu2024integrating,
  title={Integrating planning and deep reinforcement learning via automatic induction of task substructures},
  author={Liu, Jung-Chun and Chang, Chi-Hsien and Sun, Shao-Hua and Yu, Tian-Li},
  booktitle={The Twelfth International Conference on Learning Representations},
  year={2024}
}

@article{luo2024precise,
  title={Precise and dexterous robotic manipulation via human-in-the-loop reinforcement learning},
  author={Luo, Jianlan and Xu, Charles and Wu, Jeffrey and Levine, Sergey},
  journal={arXiv preprint arXiv:2410.21845},
  year={2024}
}

@inproceedings{nourzadactor,
  title={Actor-Twin Framework for Task Graph Scheduling},
  author={Nourzad, Narjes and Coleman, Jared and Zhao, Zhongyuan and Krishnamachari, Bhaskar and Verma, Gunjan and Segarra, Santiago},
  booktitle={The Seventeenth Workshop on Adaptive and Learning Agents},
year={2024}
}

@article{devidze2022exploration,
  title={Exploration-guided reward shaping for reinforcement learning under sparse rewards},
  author={Devidze, Rati and Kamalaruban, Parameswaran and Singla, Adish},
  journal={Advances in Neural Information Processing Systems},
  volume={35},
  pages={5829--5842},
  year={2022}
}

@article{schoepp2025evolving,
  title={The evolving landscape of llm-and vlm-integrated reinforcement learning},
  author={Schoepp, Sheila and Jafaripour, Masoud and Cao, Yingyue and Yang, Tianpei and Abdollahi, Fatemeh and Golestan, Shadan and Sufiyan, Zahin and Zaiane, Osmar R and Taylor, Matthew E},
  journal={arXiv preprint arXiv:2502.15214},
  year={2025}
}

@inproceedings{brohan2023can,
  title={Do as i can, not as i say: Grounding language in robotic affordances},
  author={Brohan, Anthony and Chebotar, Yevgen and Finn, Chelsea and Hausman, Karol and Herzog, Alexander and Ho, Daniel and Ibarz, Julian and Irpan, Alex and Jang, Eric and Julian, Ryan and others},
  booktitle={Conference on robot learning},
  pages={287--318},
  year={2023},
  organization={PMLR}
}

@article{kwon2023reward,
  title={Reward design with language models},
  author={Kwon, Minae and Xie, Sang Michael and Bullard, Kalesha and Sadigh, Dorsa},
  journal={arXiv preprint arXiv:2303.00001},
  year={2023}
}

@article{ma2023eureka,
  title={Eureka: Human-level reward design via coding large language models},
  author={Ma, Yecheng Jason and Liang, William and Wang, Guanzhi and Huang, De-An and Bastani, Osbert and Jayaraman, Dinesh and Zhu, Yuke and Fan, Linxi and Anandkumar, Anima},
  journal={arXiv preprint arXiv:2310.12931},
  year={2023}
}

@article{fan2022minedojo,
  title={Minedojo: Building open-ended embodied agents with internet-scale knowledge},
  author={Fan, Linxi and Wang, Guanzhi and Jiang, Yunfan and Mandlekar, Ajay and Yang, Yuncong and Zhu, Haoyi and Tang, Andrew and Huang, De-An and Zhu, Yuke and Anandkumar, Anima},
  journal={Advances in Neural Information Processing Systems},
  volume={35},
  pages={18343--18362},
  year={2022}
}

@article{rocamonde2023vision,
  title={Vision-language models are zero-shot reward models for reinforcement learning},
  author={Rocamonde, Juan and Montesinos, Victoriano and Nava, Elvis and Perez, Ethan and Lindner, David},
  journal={arXiv preprint arXiv:2310.12921},
  year={2023}
}

@article{jimenez2023swe,
  title={Swe-bench: Can language models resolve real-world github issues?},
  author={Jimenez, Carlos E and Yang, John and Wettig, Alexander and Yao, Shunyu and Pei, Kexin and Press, Ofir and Narasimhan, Karthik},
  journal={arXiv preprint arXiv:2310.06770},
  year={2023}
}

@article{xu2024theagentcompany,
  title={Theagentcompany: benchmarking llm agents on consequential real world tasks},
  author={Xu, Frank F and Song, Yufan and Li, Boxuan and Tang, Yuxuan and Jain, Kritanjali and Bao, Mengxue and Wang, Zora Z and Zhou, Xuhui and Guo, Zhitong and Cao, Murong and others},
  journal={arXiv preprint arXiv:2412.14161},
  year={2024}
}

@inproceedings{ji2023towards,
  title={Towards mitigating LLM hallucination via self reflection},
  author={Ji, Ziwei and Yu, Tiezheng and Xu, Yan and Lee, Nayeon and Ishii, Etsuko and Fung, Pascale},
  booktitle={Findings of the Association for Computational Linguistics: EMNLP 2023},
  pages={1827--1843},
  year={2023}
}

@inproceedings{xie2024text2reward,
  title={Text2reward: Automated dense reward function generation for reinforcement learning},
  author={Xie, Tianbao and Zhao, Siheng and Wu, Chen Henry and Liu, Yitao and Luo, Qian and Zhong, Victor and Yang, Yanchao and Yu, Tao},
  booktitle={International Conference on Learning Representations (ICLR), 2024 (07/05/2024-11/05/2024, Vienna, Austria)},
  year={2024}
}

@article{wang2024dart,
  title={Dart-llm: Dependency-aware multi-robot task decomposition and execution using large language models},
  author={Wang, Yongdong and Xiao, Runze and Kasahara, Jun Younes Louhi and Yajima, Ryosuke and Nagatani, Keiji and Yamashita, Atsushi and Asama, Hajime},
  journal={arXiv preprint arXiv:2411.09022},
  year={2024}
}

@inproceedings{du2023guiding,
  title={Guiding pretraining in reinforcement learning with large language models},
  author={Du, Yuqing and Watkins, Olivia and Wang, Zihan and Colas, C{\'e}dric and Darrell, Trevor and Abbeel, Pieter and Gupta, Abhishek and Andreas, Jacob},
  booktitle={International Conference on Machine Learning},
  pages={8657--8677},
  year={2023},
  organization={PMLR}
}

@inproceedings{hu2023language,
  title={Language instructed reinforcement learning for human-ai coordination},
  author={Hu, Hengyuan and Sadigh, Dorsa},
  booktitle={International Conference on Machine Learning},
  pages={13584--13598},
  year={2023},
  organization={PMLR}
}

@article{dasgupta2023collaborating,
  title={Collaborating with language models for embodied reasoning},
  author={Dasgupta, Ishita and Kaeser-Chen, Christine and Marino, Kenneth and Ahuja, Arun and Babayan, Sheila and Hill, Felix and Fergus, Rob},
  journal={arXiv preprint arXiv:2302.00763},
  year={2023}
}

@article{qu2024choices,
  title={Choices are more important than efforts: Llm enables efficient multi-agent exploration},
  author={Qu, Yun and Wang, Boyuan and Jiang, Yuhang and Shao, Jianzhun and Mao, Yixiu and Wang, Cheems and Liu, Chang and Ji, Xiangyang},
  journal={arXiv preprint arXiv:2410.02511},
  year={2024}
}

@article{gao2024designing,
  title={On designing effective rl reward at training time for llm reasoning},
  author={Gao, Jiaxuan and Xu, Shusheng and Ye, Wenjie and Liu, Weilin and He, Chuyi and Fu, Wei and Mei, Zhiyu and Wang, Guangju and Wu, Yi},
  journal={arXiv preprint arXiv:2410.15115},
  year={2024}
}

@article{cao2024survey,
  title={Survey on large language model-enhanced reinforcement learning: Concept, taxonomy, and methods},
  author={Cao, Yuji and Zhao, Huan and Cheng, Yuheng and Shu, Ting and Chen, Yue and Liu, Guolong and Liang, Gaoqi and Zhao, Junhua and Yan, Jinyue and Li, Yun},
  journal={IEEE Transactions on Neural Networks and Learning Systems},
  year={2024},
  publisher={IEEE}
}

@article{chen2024gemini,
  title={Gemini 1.5 and 2.5: Unlocking Multimodal Understanding and Reasoning in a Single Model},
  author={Chen, Zhenzhong and Purohit, Siddharth and Wang, William and Anil, Rohan and et al.},
  journal={arXiv preprint arXiv:2507.06261},
  year={2024},
  url={https://arxiv.org/abs/2507.06261}
}

@article{ananthanarayanan2024gemma,
  title={Gemma 2 and 3: Open Models Based on Gemini Research and Technology},
  author={Ananthanarayanan, Rajat and Bhosale, Jaya and Chowdhery, Aakanksha and Driess, Danny and et al.},
  journal={arXiv preprint arXiv:2503.19786},
  year={2024},
  url={https://arxiv.org/abs/2503.19786}
}

@misc{openai2024gpt4o,
  title={GPT-4o},
  author={OpenAI},
  howpublished={\url{https://openai.com/index/gpt-4o}},
  year={2024},
  note={Accessed July 2025}
}

@misc{anthropic2024claude4,
  title={Introducing the Claude 4 Model Family},
  author={Anthropic},
  howpublished={\url{https://www.anthropic.com/news/claude-4}},
  year={2024},
  note={Claude 3.5 Sonnet released June 2024}
}

@article{matthews2022hierarchical,
  title={Hierarchical kickstarting for skill transfer in reinforcement learning},
  author={Matthews, Michael and Samvelyan, Mikayel and Parker-Holder, Jack and Grefenstette, Edward and Rockt{\"a}schel, Tim},
  journal={arXiv preprint arXiv:2207.11584},
  year={2022}
}

@article{zhou2023large,
  title={Large language model as a policy teacher for training reinforcement learning agents},
  author={Zhou, Zihao and Hu, Bin and Zhao, Chenyang and Zhang, Pu and Liu, Bin},
  journal={arXiv preprint arXiv:2311.13373},
  year={2023}
}

@article{chevalierboisvert2023minigrid,
  title={Minigrid \& Miniworld: Modular \& Customizable Reinforcement Learning Environments for Goal-Oriented Tasks},
  author={Chevalier-Boisvert, Maxime and Dai, Bolun and Towers, Mark and de Lazcano, Rodrigo and Willems, Lucas and Lahlou, Salem and Pal, Suman and Castro, Pablo Samuel and Terry, Jordan},
  journal={CoRR},
  volume={abs/2306.13831},
  year={2023},
  url={https://arxiv.org/abs/2306.13831}
}

@inproceedings{chevalierboisvert2019babyai,
  title={{BabyAI}: First Steps Towards Grounded Language Learning With a Human in the Loop},
  author={Chevalier-Boisvert, Maxime and Bahdanau, Dzmitry and Lahlou, Salem and Willems, Lucas and Saharia, Chitwan and Nguyen, Thien Huu and Bengio, Yoshua},
  booktitle={International Conference on Learning Representations (ICLR)},
  year={2019},
  url={https://openreview.net/forum?id=rJeXCo0cYX}
}

@article{arnold2024gymnasium,
  title={A Standard Interface for Reinforcement Learning Environments},
  author={Arnoldo, Arjun KG and others},
  journal={arXiv preprint arXiv:2407.17032},
  year={2024},
  url={https://arxiv.org/abs/2407.17032}
}

@inproceedings{schulman2017ppo,
  title={Proximal Policy Optimization Algorithms},
  author={Schulman, John and Wolski, Filip and Dhariwal, Prafulla and Radford, Alec and Klimov, Oleg},
  booktitle={arXiv preprint arXiv:1707.06347},
  year={2017},
  url={https://arxiv.org/abs/1707.06347}
}

@inproceedings{shinn2023reflexion,
  title={Reflexion: Language agents with verbal reinforcement learning},
  author={Shinn, Noah and others},
  booktitle={NeurIPS},
  year={2023}
}

@article{wang2023voyager,
  title={Voyager: An Open-Ended Embodied Agent with Large Language Models},
  author={Wang, Guiding and Xie, Yuqi and Jiang, Kun and Lu, Xiangyu and Zhang, Weihong and Lu, Haowei and Xiong, Yitao and Li, Qian and Xu, Chuyuan and Huang, Minggang and et al.},
  journal={arXiv preprint arXiv:2305.16291},
  year={2023}
}

@inproceedings{hausknecht2015deep,
  title={Deep Recurrent {Q}-Learning for Partially Observable {MDP}s.},
  author={Hausknecht, Matthew J and Stone, Peter},
  booktitle={AAAI fall symposia},
  volume={45},
  pages={141},
  year={2015}
}

@article{kurniawati2022partially,
  title={Partially observable markov decision processes and robotics},
  author={Kurniawati, Hanna},
  journal={Annual Review of Control, Robotics, and Autonomous Systems},
  volume={5},
  number={1},
  pages={253--277},
  year={2022},
  publisher={Annual Reviews}
}

@inproceedings{liang2018rllib,
  title={RLlib: Abstractions for distributed reinforcement learning},
  author={Liang, Eric and Liaw, Richard and Nishihara, Robert and Moritz, Philipp and Fox, Roy and Goldberg, Ken and Gonzalez, Joseph and Jordan, Michael and Stoica, Ion},
  booktitle={International conference on machine learning},
  pages={3053--3062},
  year={2018},
  organization={PMLR}
}

@misc{silver2015,
    author = {David Silver},
    title = {Lectures on Reinforcement Learning},
    howpublished = {\textsc{url:}~\url{https://www.davidsilver.uk/teaching/}},
    year = {2015}
}

@article{bang2025hallulens,
  title={Hallulens: Llm hallucination benchmark},
  author={Bang, Yejin and Ji, Ziwei and Schelten, Alan and Hartshorn, Anthony and Fowler, Tara and Zhang, Cheng and Cancedda, Nicola and Fung, Pascale},
  journal={arXiv preprint arXiv:2504.17550},
  year={2025}
}

@article{tonmoy2024comprehensive,
  title={A comprehensive survey of hallucination mitigation techniques in large language models},
  author={Tonmoy, SM and Zaman, SM and Jain, Vinija and Rani, Anku and Rawte, Vipula and Chadha, Aman and Das, Amitava},
  journal={arXiv preprint arXiv:2401.01313},
  volume={6},
  year={2024}
}

@article{wang2022self,
  title={Self-consistency improves chain of thought reasoning in language models},
  author={Wang, Xuezhi and Wei, Jason and Schuurmans, Dale and Le, Quoc and Chi, Ed and Narang, Sharan and Chowdhery, Aakanksha and Zhou, Denny},
  journal={arXiv preprint arXiv:2203.11171},
  year={2022}
}

@article{schulman2017proximal,
  title={Proximal policy optimization algorithms},
  author={Schulman, John and Wolski, Filip and Dhariwal, Prafulla and Radford, Alec and Klimov, Oleg},
  journal={arXiv preprint arXiv:1707.06347},
  year={2017}
}

@inproceedings{stolle2002learning,
  title={Learning options in reinforcement learning},
  author={Stolle, Martin and Precup, Doina},
  booktitle={International Symposium on abstraction, reformulation, and approximation},
  pages={212--223},
  year={2002},
  organization={Springer}
}

@article{sutton1999between,
  title={Between {MDP}s and semi-{MDP}s: A framework for temporal abstraction in reinforcement learning},
  author={Sutton, Richard S and Precup, Doina and Singh, Satinder},
  journal={Artificial intelligence},
  volume={112},
  number={1-2},
  pages={181--211},
  year={1999},
  publisher={Elsevier}
}

@inproceedings{macglashan2017interactive,
  title={Interactive learning from policy-dependent human feedback},
  author={MacGlashan, James and Ho, Mark K and Loftin, Robert and Peng, Bei and Wang, Guan and Roberts, David L and Taylor, Matthew E and Littman, Michael L},
  booktitle={International conference on machine learning},
  pages={2285--2294},
  year={2017},
  organization={PMLR}
}

@inproceedings{shiarlis2018taco,
  title={Taco: Learning task decomposition via temporal alignment for control},
  author={Shiarlis, Kyriacos and Wulfmeier, Markus and Salter, Sasha and Whiteson, Shimon and Posner, Ingmar},
  booktitle={International Conference on Machine Learning},
  pages={4654--4663},
  year={2018},
  organization={PMLR}
}

@article{narvekar2020curriculum,
  title={Curriculum learning for reinforcement learning domains: A framework and survey},
  author={Narvekar, Sanmit and Peng, Bei and Leonetti, Matteo and Sinapov, Jivko and Taylor, Matthew E and Stone, Peter},
  journal={Journal of Machine Learning Research},
  volume={21},
  number={181},
  pages={1--50},
  year={2020}
}

@article{williams1992simple,
  title={Simple statistical gradient-following algorithms for connectionist reinforcement learning},
  author={Williams, Ronald J},
  journal={Machine learning},
  volume={8},
  number={3},
  pages={229--256},
  year={1992},
  publisher={Springer}
}

@inproceedings{mnih2016asynchronous,
  title={Asynchronous methods for deep reinforcement learning},
  author={Mnih, Volodymyr and Badia, Adria Puigdomenech and Mirza, Mehdi and Graves, Alex and Lillicrap, Timothy and Harley, Tim and Silver, David and Kavukcuoglu, Koray},
  booktitle={International conference on machine learning},
  pages={1928--1937},
  year={2016},
  organization={PmLR}
}

@inproceedings{henderson2018deep,
  title={Deep reinforcement learning that matters},
  author={Henderson, Peter and Islam, Riashat and Bachman, Philip and Pineau, Joelle and Precup, Doina and Meger, David},
  booktitle={Proceedings of the AAAI conference on artificial intelligence},
  volume={32},
  year={2018}
}

@inproceedings{paischer2022history,
  title={History compression via language models in reinforcement learning},
  author={Paischer, Fabian and Adler, Thomas and Patil, Vihang and Bitto-Nemling, Angela and Holzleitner, Markus and Lehner, Sebastian and Eghbal-Zadeh, Hamid and Hochreiter, Sepp},
  booktitle={International Conference on Machine Learning},
  pages={17156--17185},
  year={2022},
  organization={PMLR}
}

@article{paischer2023semantic,
  title={Semantic {HELM}: A human-readable memory for reinforcement learning},
  author={Paischer, Fabian and Adler, Thomas and Hofmarcher, Markus and Hochreiter, Sepp},
  journal={Advances in Neural Information Processing Systems},
  volume={36},
  pages={9837--9865},
  year={2023}
}

@inproceedings{radford2021learning,
  title={Learning transferable visual models from natural language supervision},
  author={Radford, Alec and Kim, Jong Wook and Hallacy, Chris and Ramesh, Aditya and Goh, Gabriel and Agarwal, Sandhini and Sastry, Girish and Askell, Amanda and Mishkin, Pamela and Clark, Jack and others},
  booktitle={International conference on machine learning},
  pages={8748--8763},
  year={2021},
  organization={PMLR}
}

@article{kim2023efficient,
  title={Efficient policy adaptation with contrastive prompt ensemble for embodied agents},
  author={Kim, Woo Kyung and Kim, SeungHyun and Woo, Honguk and others},
  journal={Advances in Neural Information Processing Systems},
  volume={36},
  pages={55442--55453},
  year={2023}
}

@inproceedings{poudel2024recore,
  title={Recore: Regularized contrastive representation learning of world model},
  author={Poudel, Rudra PK and Pandya, Harit and Liwicki, Stephan and Cipolla, Roberto},
  booktitle={Proceedings of the IEEE/CVF Conference on Computer Vision and Pattern Recognition},
  pages={22904--22913},
  year={2024}
}

@article{basavatia2024starling,
  title={Starling: Self-supervised training of text-based reinforcement learning agent with large language models},
  author={Basavatia, Shreyas and Murugesan, Keerthiram and Ratnakar, Shivam},
  journal={arXiv preprint arXiv:2406.05872},
  year={2024}
}

@inproceedings{sumers2021learning,
  title={Learning rewards from linguistic feedback},
  author={Sumers, Theodore R and Ho, Mark K and Hawkins, Robert D and Narasimhan, Karthik and Griffiths, Thomas L},
  booktitle={Proceedings of the AAAI Conference on Artificial Intelligence},
  volume={35},
  pages={6002--6010},
  year={2021}
}

@article{liang2022code,
  title={Code as policies: Language model programs for embodied control},
  author={Liang, Jacky and Huang, Wenlong and Xia, Fei and Xu, Peng and Hausman, Karol and Ichter, Brian and Florence, Pete and Zeng, Andy},
  journal={arXiv preprint arXiv:2209.07753},
  year={2022}
}

@inproceedings{song2023llm,
  title={{LLM}-planner: Few-shot grounded planning for embodied agents with large language models},
  author={Song, Chan Hee and Wu, Jiaman and Washington, Clayton and Sadler, Brian M and Chao, Wei-Lun and Su, Yu},
  booktitle={Proceedings of the IEEE/CVF international conference on computer vision},
  pages={2998--3009},
  year={2023}
}

@article{chu2023accelerating,
  title={Accelerating reinforcement learning of robotic manipulations via feedback from large language models},
  author={Chu, Kun and Zhao, Xufeng and Weber, Cornelius and Li, Mengdi and Wermter, Stefan},
  journal={arXiv preprint arXiv:2311.02379},
  year={2023}
}

@article{wu2023read,
  title={Read and reap the rewards: Learning to play atari with the help of instruction manuals},
  author={Wu, Yue and Fan, Yewen and Liang, Paul Pu and Azaria, Amos and Li, Yuanzhi and Mitchell, Tom M},
  journal={Advances in Neural Information Processing Systems},
  volume={36},
  pages={1009--1023},
  year={2023}
}

@article{adeniji2023language,
  title={Language reward modulation for pretraining reinforcement learning},
  author={Adeniji, Ademi and Xie, Amber and Sferrazza, Carmelo and Seo, Younggyo and James, Stephen and Abbeel, Pieter},
  journal={arXiv preprint arXiv:2308.12270},
  year={2023}
}

@inproceedings{seo2023masked,
  title={Masked world models for visual control},
  author={Seo, Younggyo and Hafner, Danijar and Liu, Hao and Liu, Fangchen and James, Stephen and Lee, Kimin and Abbeel, Pieter},
  booktitle={Conference on Robot Learning},
  pages={1332--1344},
  year={2023},
  organization={PMLR}
}

@article{wang2024rl,
  title={Rl-vlm-f: Reinforcement learning from vision language foundation model feedback},
  author={Wang, Yufei and Sun, Zhanyi and Zhang, Jesse and Xian, Zhou and Biyik, Erdem and Held, David and Erickson, Zackory},
  journal={arXiv preprint arXiv:2402.03681},
  year={2024}
}

@inproceedings{grauman2022ego4d,
  title={Ego4d: Around the world in 3,000 hours of egocentric video},
  author={Grauman, Kristen and Westbury, Andrew and Byrne, Eugene and Chavis, Zachary and Furnari, Antonino and Girdhar, Rohit and Hamburger, Jackson and Jiang, Hao and Liu, Miao and Liu, Xingyu and others},
  booktitle={Proceedings of the IEEE/CVF conference on computer vision and pattern recognition},
  pages={18995--19012},
  year={2022}
}

@article{yu2023language,
  title={Language to rewards for robotic skill synthesis},
  author={Yu, Wenhao and Gileadi, Nimrod and Fu, Chuyuan and Kirmani, Sean and Lee, Kuang-Huei and Arenas, Montse Gonzalez and Chiang, Hao-Tien Lewis and Erez, Tom and Hasenclever, Leonard and Humplik, Jan and others},
  journal={arXiv preprint arXiv:2306.08647},
  year={2023}
}

@article{madaan2023self,
  title={Self-refine: Iterative refinement with self-feedback},
  author={Madaan, Aman and Tandon, Niket and Gupta, Prakhar and Hallinan, Skyler and Gao, Luyu and Wiegreffe, Sarah and Alon, Uri and Dziri, Nouha and Prabhumoye, Shrimai and Yang, Yiming and others},
  journal={Advances in Neural Information Processing Systems},
  volume={36},
  pages={46534--46594},
  year={2023}
}

@article{song2023self,
  title={Self-refined large language model as automated reward function designer for deep reinforcement learning in robotics},
  author={Song, Jiayang and Zhou, Zhehua and Liu, Jiawei and Fang, Chunrong and Shu, Zhan and Ma, Lei},
  journal={arXiv preprint arXiv:2309.06687},
  year={2023}
}

@article{shi2023unleashing,
  title={Unleashing the power of pre-trained language models for offline reinforcement learning},
  author={Shi, Ruizhe and Liu, Yuyao and Ze, Yanjie and Du, Simon S and Xu, Huazhe},
  journal={arXiv preprint arXiv:2310.20587},
  year={2023}
}

@article{mezghani2023think,
  title={Think before you act: Unified policy for interleaving language reasoning with actions},
  author={Mezghani, Lina and Bojanowski, Piotr and Alahari, Karteek and Sukhbaatar, Sainbayar},
  journal={arXiv preprint arXiv:2304.11063},
  year={2023}
}

@inproceedings{zitkovich2023rt,
  title={Rt-2: Vision-language-action models transfer web knowledge to robotic control},
  author={Zitkovich, Brianna and Yu, Tianhe and Xu, Sichun and Xu, Peng and Xiao, Ted and Xia, Fei and Wu, Jialin and Wohlhart, Paul and Welker, Stefan and Wahid, Ayzaan and others},
  booktitle={Conference on Robot Learning},
  pages={2165--2183},
  year={2023},
  organization={PMLR}
}

@article{wan2025think,
  title={Think Twice, Act Once: A Co-Evolution Framework of LLM and RL for Large-Scale Decision Making},
  author={Wan, Xu and Xu, Wenyue and Yang, Chao and Sun, Mingyang},
  journal={arXiv preprint arXiv:2506.02522},
  year={2025}
}

@article{yao2020keep,
  title={Keep calm and explore: Language models for action generation in text-based games},
  author={Yao, Shunyu and Rao, Rohan and Hausknecht, Matthew and Narasimhan, Karthik},
  journal={arXiv preprint arXiv:2010.02903},
  year={2020}
}

@inproceedings{hausknecht2020interactive,
  title={Interactive fiction games: A colossal adventure},
  author={Hausknecht, Matthew and Ammanabrolu, Prithviraj and C{\^o}t{\'e}, Marc-Alexandre and Yuan, Xingdi},
  booktitle={Proceedings of the AAAI Conference on Artificial Intelligence},
  volume={34},
  pages={7903--7910},
  year={2020}
}

@inproceedings{ma2025catching,
  title={Catching Two Birds with One Stone: Reward Shaping with Dual Random Networks for Balancing Exploration and Exploitation},
  author={Ma, Haozhe and Li, Fangling and Lim, Jing Yu and Luo, Zhengding and Vo, Thanh Vinh and Leong, Tze-Yun},
  booktitle={Forty-second International Conference on Machine Learning},
  year={2025}
}

@article{dalal2024plan,
  title={Plan-seq-learn: Language model guided rl for solving long horizon robotics tasks},
  author={Dalal, Murtaza and Chiruvolu, Tarun and Chaplot, Devendra and Salakhutdinov, Ruslan},
  journal={arXiv preprint arXiv:2405.01534},
  year={2024}
}

@article{micheli2022transformers,
  title={Transformers are sample-efficient world models},
  author={Micheli, Vincent and Alonso, Eloi and Fleuret, Fran{\c{c}}ois},
  journal={arXiv preprint arXiv:2209.00588},
  year={2022}
}

@article{chen2022transdreamer,
  title={Transdreamer: Reinforcement learning with transformer world models},
  author={Chen, Chang and Wu, Yi-Fu and Yoon, Jaesik and Ahn, Sungjin},
  journal={arXiv preprint arXiv:2202.09481},
  year={2022}
}

@article{robine2023transformer,
  title={Transformer-based world models are happy with 100k interactions},
  author={Robine, Jan and H{\"o}ftmann, Marc and Uelwer, Tobias and Harmeling, Stefan},
  journal={arXiv preprint arXiv:2303.07109},
  year={2023}
}

@article{lin2023learning,
  title={Learning to model the world with language},
  author={Lin, Jessy and Du, Yuqing and Watkins, Olivia and Hafner, Danijar and Abbeel, Pieter and Klein, Dan and Dragan, Anca},
  journal={arXiv preprint arXiv:2308.01399},
  year={2023}
}

@inproceedings{lu2023closer,
  title={A closer look at reward decomposition for high-level robotic explanations},
  author={Lu, Wenhao and Zhao, Xufeng and Magg, Sven and Gromniak, Martin and Li, Mengdi and Wermter, Stefan},
  booktitle={2023 IEEE International Conference on Development and Learning (ICDL)},
  pages={429--436},
  year={2023},
  organization={IEEE}
}

@inproceedings{qiu2024memory,
  title={Memory-augmented deep deterministic policy gradient},
  author={Qiu, Qian and Zeng, Fanyu and Yang, Haigen and Xing, Guanyu and Ge, Shuzhi Sam},
  booktitle={International Conference on Social Robotics},
  pages={41--52},
  year={2024},
  organization={Springer}
}

@inproceedings{liu2018effects,
  title={The effects of memory replay in reinforcement learning},
  author={Liu, Ruishan and Zou, James},
  booktitle={2018 56th annual allerton conference on communication, control, and computing (Allerton)},
  pages={478--485},
  year={2018},
  organization={IEEE}
}

@inproceedings{pritzel2017neural,
  title={Neural episodic control},
  author={Pritzel, Alexander and Uria, Benigno and Srinivasan, Sriram and Badia, Adria Puigdomenech and Vinyals, Oriol and Hassabis, Demis and Wierstra, Daan and Blundell, Charles},
  booktitle={International conference on machine learning},
  pages={2827--2836},
  year={2017},
  organization={PMLR}
}

@article{lin1992self,
  title={Self-improving reactive agents based on reinforcement learning, planning and teaching},
  author={Lin, Long-Ji},
  journal={Machine learning},
  volume={8},
  number={3},
  pages={293--321},
  year={1992},
  publisher={Springer}
}

@article{blundell2016model,
  title={Model-free episodic control},
  author={Blundell, Charles and Uria, Benigno and Pritzel, Alexander and Li, Yazhe and Ruderman, Avraham and Leibo, Joel Z and Rae, Jack and Wierstra, Daan and Hassabis, Demis},
  journal={arXiv preprint arXiv:1606.04460},
  year={2016}
}

@inproceedings{beeching2020egomap,
  title={Egomap: Projective mapping and structured egocentric memory for deep RL},
  author={Beeching, Edward and Dibangoye, Jilles and Simonin, Olivier and Wolf, Christian},
  booktitle={Joint European conference on machine learning and knowledge discovery in databases},
  pages={525--540},
  year={2020},
  organization={Springer}
}

@article{rana2023sayplan,
  title={Sayplan: Grounding large language models using 3d scene graphs for scalable robot task planning},
  author={Rana, Krishan and Haviland, Jesse and Garg, Sourav and Abou-Chakra, Jad and Reid, Ian and Suenderhauf, Niko},
  journal={arXiv preprint arXiv:2307.06135},
  year={2023}
}

@article{peng2019advantage,
  title={Advantage-weighted regression: Simple and scalable off-policy reinforcement learning},
  author={Peng, Xue Bin and Kumar, Aviral and Zhang, Grace and Levine, Sergey},
  journal={arXiv preprint arXiv:1910.00177},
  year={2019}
}

@article{lee2021pebble,
  title={Pebble: Feedback-efficient interactive reinforcement learning via relabeling experience and unsupervised pre-training},
  author={Lee, Kimin and Smith, Laura and Abbeel, Pieter},
  journal={arXiv preprint arXiv:2106.05091},
  year={2021}
}

@article{kostrikov2021offline,
  title={Offline reinforcement learning with implicit q-learning},
  author={Kostrikov, Ilya and Nair, Ashvin and Levine, Sergey},
  journal={arXiv preprint arXiv:2110.06169},
  year={2021}
}

@article{schulman2015high,
  title={High-dimensional continuous control using generalized advantage estimation},
  author={Schulman, John and Moritz, Philipp and Levine, Sergey and Jordan, Michael and Abbeel, Pieter},
  journal={arXiv preprint arXiv:1506.02438},
  year={2015}
}

@inproceedings{fakoor2020p3o,
  title={P3o: Policy-on policy-off policy optimization},
  author={Fakoor, Rasool and Chaudhari, Pratik and Smola, Alexander J},
  booktitle={Uncertainty in artificial intelligence},
  pages={1017--1027},
  year={2020},
  organization={PMLR}
}

@article{schulman2015gradient,
  title={Gradient estimation using stochastic computation graphs},
  author={Schulman, John and Heess, Nicolas and Weber, Theophane and Abbeel, Pieter},
  journal={Advances in neural information processing systems},
  volume={28},
  year={2015}
}

@article{kaelbling1998planning,
  title={Planning and acting in partially observable stochastic domains},
  author={Kaelbling, Leslie Pack and Littman, Michael L and Cassandra, Anthony R},
  journal={Artificial intelligence},
  volume={101},
  number={1-2},
  pages={99--134},
  year={1998},
  publisher={Elsevier}
}

@article{velu2023hindsight,
  title={Hindsight-dice: Stable credit assignment for deep reinforcement learning},
  author={Velu, Akash and Vaidyanath, Skanda and Arumugam, Dilip},
  journal={arXiv preprint arXiv:2307.11897},
  year={2023}
}

@article{wan2023efficient,
  title={Efficient large language models: A survey},
  author={Wan, Zhongwei and Wang, Xin and Liu, Che and Alam, Samiul and Zheng, Yu and Liu, Jiachen and Qu, Zhongnan and Yan, Shen and Zhu, Yi and Zhang, Quanlu and others},
  journal={arXiv preprint arXiv:2312.03863},
  year={2023}
}

@article{huang2022inner,
  title={Inner monologue: Embodied reasoning through planning with language models},
  author={Huang, Wenlong and Xia, Fei and Xiao, Ted and Chan, Harris and Liang, Jacky and Florence, Pete and Zeng, Andy and Tompson, Jonathan and Mordatch, Igor and Chebotar, Yevgen and others},
  journal={arXiv preprint arXiv:2207.05608},
  year={2022}
}

@article{nair2022r3m,
  title={R3m: A universal visual representation for robot manipulation},
  author={Nair, Suraj and Rajeswaran, Aravind and Kumar, Vikash and Finn, Chelsea and Gupta, Abhinav},
  journal={arXiv preprint arXiv:2203.12601},
  year={2022}
}

@book{kushner2003stochastic,
  title={Stochastic approximation and recursive algorithms and applications},
  author={Kushner, Harold J and Yin, G George},
  year={2003},
  publisher={Springer}
}

@article{biza2021action,
  title={Action priors for large action spaces in robotics},
  author={Biza, Ondrej and Wang, Dian and Platt, Robert and van de Meent, Jan-Willem and Wong, Lawson LS},
  journal={arXiv preprint arXiv:2101.04178},
  year={2021}
}

@article{zhou2024survey,
  title={A survey on efficient inference for large language models},
  author={Zhou, Zixuan and Ning, Xuefei and Hong, Ke and Fu, Tianyu and Xu, Jiaming and Li, Shiyao and Lou, Yuming and Wang, Luning and Yuan, Zhihang and Li, Xiuhong and others},
  journal={arXiv preprint arXiv:2404.14294},
  year={2024}
}

@article{schaul2015prioritized,
  title={Prioritized experience replay},
  author={Schaul, Tom and Quan, John and Antonoglou, Ioannis and Silver, David},
  journal={arXiv preprint arXiv:1511.05952},
  year={2015}
}

@article{bhambri2024extracting,
  title={Extracting Heuristics from Large Language Models for Reward Shaping in Reinforcement Learning},
  author={Bhambri, Siddhant and Bhattacharjee, Amrita and Kalwar, Durgesh and Guan, Lin and Liu, Huan and Kambhampati, Subbarao},
  journal={arXiv preprint arXiv:2405.15194},
  year={2024}
}

@article{ruxton2006unequal,
  title={The unequal variance t-test is an underused alternative to Student's t-test and the Mann--Whitney U test},
  author={Ruxton, Graeme D},
  journal={Behavioral Ecology},
  volume={17},
  number={4},
  pages={688--690},
  year={2006},
  publisher={Oxford University Press}
}

@article{hafner2023mastering,
  title={Mastering diverse domains through world models},
  author={Hafner, Danijar and Pasukonis, Jurgis and Ba, Jimmy and Lillicrap, Timothy},
  journal={arXiv preprint arXiv:2301.04104},
  year={2023}
}
\bibliographystyle{iclr2026_conference}
% \vspace{10cm}
\appendix 
\pagebreak
\section*{Appendix}

The supplemental material is organized as follows:

\begin{itemize}
    \item \textsc{Section~\ref{bck}} reviews \textit{background} on reinforcement learning definitions and policy gradient algorithms to make the paper self-contained. 
    \item \textsc{Section~\ref{rw}} discusses \textit{related work} relevant to our approach in more depth.
    \item \textsc{Section~\ref{thm}} presents the \textit{theoretical} results and supporting proofs for our method.
    \item  \textsc{Section~\ref{prompt}} describes the \textit{LLM prompting procedures} in MIRA and presents the corresponding \textit{reasoning} traces.
     \item \textsc{Section~\ref{graph_detail}} provides a more detailed explanation of the \textit{construction of memory graph}, expanding on the description in the main paper. 
          \item \n{\textsc{Section~\ref{utility_detail}} provides a more detailed explanation of the \textit{utility calculation}  and includes the complete pseudocode, expanding on the description in the main paper. }
   \item \textsc{Section~\ref{ret_detail}} presents \textit{extended experiments}, including analyses of runtime and detailed numerical results that were not covered in the main text.
   \item \textsc{Section~\ref{lim}} outlines \textit{limitations} of the current design and identifies open challenges. 
   \item  \textsc{Section~\ref{rep}} provides details to support \textit{reproducibility} of our results.
\end{itemize}

% \pagebreak
\section{Background} ~\label{bck}
\subsection{Standard Reinforcement Learning}
Reinforcement learning (RL) is typically modeled as a Markov decision process (MDP), defined by a tuple $(\mathcal{S}, \mathcal{A}, P, r, \gamma)$, where $\mathcal{S}$ is the state space, $\mathcal{A}$ the action space, $P$ the transition function, $r$ the reward function, and $\gamma \in [0,1)$ the discount factor. 
The agent’s behavior is determined by a policy $\pi$, which defines a probability distribution over actions given the current state: $\pi(a|s)$. Learning proceeds through interaction with the environment, producing trajectories, sequences of states, actions, and rewards of the form $ \tau = (s_0, a_0, r_0, s_1, a_1, r_1, \dots)$, and using them to improve the policy. 

The objective is to learn a policy that maximizes the expected return, defined as the discounted sum of rewards along a trajectory:

\begin{equation}
   \mathbb{E}_\pi\left[\sum_{t=0}^{\infty} \gamma^t r(s_t, a_t)\right]. 
\end{equation}

The environment’s reward function implicitly defines the final goal ($\textsl{g}_\triangleright$) by assigning reward to behaviors that accomplish the task~\citep{sutton1998reinforcement, silver2015}.
To estimate this objective, RL algorithms often make use of value functions, which quantify the long-term utility of states or state-action pairs. The state-value function $V(s)$ denotes the expected return when starting from state $s$ and following policy $\pi$:

\begin{equation}
    V(s) = \mathbb{E}_\pi\left[\sum_{t=0}^\infty \gamma^t r(s_t, a_t) \mid s_0 = s \right].
\end{equation}

The action-value function $Q(s, a)$ further conditions on the first action taken and is defined as:

\begin{equation}
    Q(s, a) = \mathbb{E}_\pi\left[\sum_{t=0}^\infty \gamma^t r(s_t, a_t) \mid s_0 = s, a_0 = a \right].
\end{equation}

\subsubsection{Partial Observability and Credit Assignment Challenges}
In many real-world scenarios, the environment is only partially observable. In such cases, the MDP generalizes to a partially observable MDP (POMDP), defined by the tuple $(\mathcal{S}, \mathcal{A}, P, r, \gamma, \mathcal{O}, \Omega)$, where $\mathcal{O}$ is the observation space and $\Omega$ is the observation function. The agent does not directly observe the true state $s_t \in \mathcal{S}$; instead, it receives observations $o_t$ from an observation space $\mathcal{O}$, sampled via $\Omega(o_t|s_t)$, and must rely on its history of observations and actions to make decisions~\citep{kaelbling1998planning}.

These difficulties are further amplified in environments where the agents face sparse and delayed rewards.  Sparse rewards refer to the limited presence of nonzero rewards since this feedback is only provided upon reaching specific goals (i.e., $r(s_t, a_t)$ is typically zero until the agent reaches the final goal state ($\textsl{g}_\triangleright$) defined by the task). On the other hand, delayed rewards refer to settings where the consequences of an action are not reflected in the reward until several steps later.  In both cases, the agent must reason over long horizons to determine which actions contributed to the eventual outcome, a challenge known as the credit assignment problem~\citep{schulman2015gradient}.

Credit assignment is closely tied to the broader challenge of exploration. Inefficient exploration occurs when the agent fails to sufficiently cover the state space, limiting its ability to discover high-return trajectories and improve its policy. 
This problem is exacerbated in high-dimensional environments, where the number of possible state-action sequences grows exponentially and random exploration becomes increasingly unlikely to encounter informative transitions with sparse or delayed rewards. In such cases, the combination of large search spaces and limited reward signals often leads to slow convergence, poor sample efficiency, and high variance in learning outcomes. 

\subsubsection{Subgoals and Abstractions}
In long-horizon tasks, reinforcement learning agents often benefit from structuring behavior around subgoals, intermediate objectives that facilitate progress toward the overall task. The concept of subgoals in reinforcement learning originated in hierarchical reinforcement learning (HRL), where it was formalized through the use of temporally extended actions. 
In particular, the options framework introduced by \citet{sutton1999between} defines options as high-level actions composed of an initiation set, a policy, and a termination condition, often interpreted as achieving a subgoal~\citep{stolle2002learning}. 
These subgoals correspond to intermediate states or conditions that decompose long-horizon tasks into smaller, temporally coherent segments that make the final goal more attainable when reached.
More broadly, subgoals provide structure for reasoning over extended time horizons and facilitate learning in sparse-reward settings.

While early approaches focused on explicit or learned state-based subgoals, recent work has explored abstract subgoals that capture semantic or latent-level progress. These abstractions may not correspond to a specific state but instead reflect high-level intentions and meaningful progress (e.g., opening a door, entering a room, or collecting an object). 
Such abstractions enable reasoning at a higher level of granularity and are especially useful in environments with sparse rewards or delayed feedback. Subgoal discovery and abstraction have also been explored in curriculum learning, imitation learning, and human-in-the-loop frameworks to improve exploration and sample efficiency~\citep{macglashan2017interactive, shiarlis2018taco, narvekar2020curriculum}.

\subsection{Policy Gradient Methods}

Policy gradient methods directly optimize a parameterized policy $\pi_\theta(a|s)$ by ascending the gradient of expected return. The objective is to find parameters $\theta$ that maximize:
\begin{equation}
   J(\theta) = \mathbb{E}_{\tau \sim \pi_\theta} \left[ \sum_{t=0}^\infty \gamma^t r(s_t, a_t) \right], 
\end{equation}

where $\tau$ denotes a trajectory generated by following the current policy. 
 The gradient of this objective can be estimated via the likelihood ratio trick, yielding the REINFORCE estimator~\citep{williams1992simple}:
\begin{equation}
    \nabla_\theta J(\pi_\theta) = \mathbb{E}_\pi \left[ \nabla_\theta \log \pi_\theta(a_t|s_t) R\right],
\end{equation}

where $R_t$ is the return from time $t$ onward. While theoretically sound and unbiased, this estimator suffers from high variance, making it challenging to apply in practice without further refinement.

\subsubsection{Advantage-Based Policy Optimization}

To reduce variance and improve sample efficiency, modern policy gradient algorithms often use advantage functions, which quantify the relative quality of an action compared to the policy’s baseline behavior. The advantage function is defined as:
\begin{equation}
    A(s, a) = Q(s, a) - V(s),
\end{equation}
    
where $Q(s, a)$ is the expected return from taking action $a$ in state $s$, and $V(s)$ is the expected return from $s$ under policy $\pi$.
Using this formulation, the policy gradient becomes:

\begin{equation}
    \nabla_\theta J(\pi_\theta) = \mathbb{E}_\pi \left[ \nabla_\theta \log \pi_\theta(a_t \mid s_t) A \right],
\end{equation}

which improves stability while preserving unbiasedness.

This idea underpins a family of actor-critic algorithms, where the actor updates the policy using the advantage-weighted gradient, and the critic estimates value functions used to compute $A(s,a)$.  Representative algorithms in this class include A2C and A3C~\citep{mnih2016asynchronous}, which leverage parallel actors to accelerate training and stabilize updates, and PPO~\citep{schulman2017proximal}, which constrains policy updates by clipping the policy ratio in the surrogate objective:

\begin{equation}
    \mathcal{L}^{\text{PPO}}(\pi) = \mathbb{E}\left[\min(r_t A_t, \text{clip}(r_t, 1 - \varepsilon, 1 + \varepsilon) A_t)\right],
\end{equation}
where $\varepsilon > 0$ is a small trust region parameter that limits how much the policy is allowed to change at each update.

These methods are widely used in modern deep reinforcement learning due to their scalability and consistent empirical performance across a range of tasks. Since MIRA operates by shaping the advantage function, it is compatible with any policy optimization method that relies on advantage-weighted updates.

\section{Related Works} \label{rw}

\subsection{Language Model Guidance in RL }
A growing line of work explores how large language models (LLMs) can be integrated into reinforcement learning by framing them as auxiliary components within the agent–environment loop.  
A recent taxonomy by~\citet{cao2024survey} outlines the roles of LLMs in RL along four key dimensions: information processors, reward designers, decision-makers, and generators. 

As \textit{information processors}, LLMs extract and organize task-relevant knowledge from natural language, environment descriptions, or prior experience. This includes synthesizing high-level goals, parsing instructions, and transforming language input into actionable constraints or representations~\citep{wang2024dart, shinn2023reflexion}. A common approach is to use frozen pre-trained models to encode task-relevant features without fine-tuning, though they may perform poorly on out-of-distribution data due to limited adaptability~\citep{radford2021learning, paischer2022history, paischer2023semantic}. 
\\Alternatively, fine-tuned models can better align with task-specific distributions, leading to more robust RL performance and improved generalization in unseen environments~\citep{kim2023efficient, poudel2024recore}. In addition, LLMs can convert human instructions or task prompts into formal representations or structured goals, and interpret descriptions of the environment, such as objects, layouts, or dynamics, into usable priors for downstream RL modules. This reduces the burden of language comprehension for RL agents and improves sample efficiency~\citep{basavatia2024starling, sumers2021learning, song2023llm, liang2022code}. 
These models can decouple information processing from control, with the LLM handling language grounding and feature extraction while the policy module focuses on decision-making. Such capabilities can reduce learning complexity and accelerate policy acquisition by shaping the agent’s representation space early in training.

As \textit{reward designers}, LLMs provide auxiliary supervision by scoring agent behavior or generating rewards. 
This can take the form of natural language critiques, programmatic reward code, or goal-conditioned evaluations.  In the implicit reward setting, LLMs serve as proxy reward models by either being directly prompted to evaluate agent behavior~\citep{chu2023accelerating, wu2023read} or by computing alignment between visual observations and language goals using pretrained vision-language models~\citep{wang2024rl, adeniji2023language, seo2023masked, grauman2022ego4d}. These methods enable reward shaping via natural language instructions or preference feedback and have been shown to improve learning efficiency and generalization. In the explicit reward setting, LLMs are used to generate executable code that defines reward functions programmatically. This includes frameworks that iteratively refine reward code using self-reflection and feedback from training outcomes~\citep{yu2023language, madaan2023self, song2023self}. Compared to manually engineered rewards, these LLM-generated functions offer transparency and adaptability, and in some cases match or exceed human performance, especially in complex manipulation tasks.

As \textit{decision-makers}, LLMs output action plans, policy sketches, or even direct action sequences based on current observations. 
These methods embed LLMs tightly into the decision loop, either guiding exploration or dictating behavior in few-shot or zero-shot settings. One approach leverages pre-trained LLMs for direct action generation, often adapting transformer-based models like Decision Transformers to treat offline RL as a sequence modeling problem. These LLM-backed policies show improved generalization, particularly in sparse-reward and long-horizon tasks, by transferring latent structure learned from large-scale language data. Some methods further fine-tune LLMs using task-specific trajectories or append small task-specific modules to facilitate adaptation, achieving notable gains in sample efficiency and task transfer~\citep{zitkovich2023rt, shi2023unleashing, mezghani2023think}. 
\\
Other works integrate LLMs more loosely as action guides, generating action candidates or expert priors to support exploration and training. For example, LLMs can prune the action space by proposing high-probability candidates or decompose complex goals into sequential subtasks, improving exploration in environments with large or unstructured action spaces~\citep{yao2020keep, hausknecht2020interactive, dalal2024plan, wan2025think}. They have also been used to regularize policy updates, align agent behavior with human intent, or inject expert-level motion plans. Across both low-level and strategic roles, LLM-based decision-making enables agents to learn from rich, structured priors and execute more informed behaviors in complex settings.

As \textit{generators}, LLMs contribute to reinforcement learning by either simulating environmental dynamics or providing policy-level explanations to enhance transparency. In the simulation role, LLMs function as world model simulators that generate trajectories or learn latent dynamics representations from multimodal data, thereby improving sample efficiency in model-based RL. Recent work has leveraged Transformer-based architectures to model complex visual or sequential environments, demonstrating gains in generalization and long-horizon reasoning. These models either auto-regressively generate rollouts from pre-trained dynamics or use representation learning to predict future states and rewards, often incorporating language as an additional modality for grounding and abstraction~\citep{micheli2022transformers, chen2022transdreamer, robine2023transformer}. Separately, LLMs have been used as policy interpreters to generate human-readable explanations of agent behavior from state-action histories or decision trees. This facilitates interpretability, improves human trust, and can inform reward design or debugging, though current work has focused mainly on policy-level summaries~\citep{lin2023learning, lu2023closer}. 

While MIRA incorporates elements of information processing and LLMs as generators, its overall orientation remains distinct and more RL-centric from prior LLM-centered approaches. Rather than positioning the LLM as a decision-maker or continuous feedback provider, MIRA relegates it to a supporting role that gradually fades over time. LLM outputs are used intermittently to enrich a structured memory graph that informs, but does not dictate, learning. The primary learning signal remains grounded in environment interaction, with utility shaping softly modulating advantage estimates rather than overriding the reward function. This design prioritizes policy optimization through reinforcement learning rather than imitation or prompting

\subsection{Memory and Buffers in RL}
Augmenting RL agents with structured memory has been proposed as a means of supporting generalization, planning, and long-horizon credit assignment. 
Early works such as Neural Episodic Control (NEC) and other episodic value-based methods enabled agents to recall high-value past experiences for more sample-efficient decision-making via memory buffers~\citep{pritzel2017neural, blundell2016model, lin1992self}. 
Subsequent approaches extended this idea by integrating differentiable memory into policy networks~\citep{qiu2024memory}. 
Other methods introduce structured representations, such as subgoal graphs or navigation maps, to facilitate hierarchical planning, exploration, or navigation in partially observable environments~\citep{beeching2020egomap, rana2023sayplan}. 
Across these directions, the common pattern is to directly query stored structures, either through replay, imitation, or graph traversal, to guide behavior. 

MIRA aligns with this direction by maintaining a structured memory graph populated with high-return trajectory segments but departs from this pattern in several key ways. 
First, its memory graph is co-constructed from high-return agent trajectories and LLM-inferred subgoals, enabling abstraction and structure difficult to obtain early through interaction alone.
Second, rather than querying memory for action selection or value estimation, MIRA distills the stored information into a utility signal that modulates advantage estimates during training. 
This indirect shaping avoids disrupting the optimization loop or overfitting to specific stored transitions. 
Finally, MIRA maintains a compact memory via pruning and infrequent updates, which avoids the inefficiencies of excessive memory or the brittleness of sparse guidance~\citep{liu2018effects}. 
This makes MIRA more scalable and better suited for tasks where long-term structure must complement autonomous learning.

\subsection{Advantage Modifications in RL}

Modifying the advantage function has been studied as a way to stabilize learning and improve sample efficiency in policy optimization. A common approach adjusts the estimation process to better balance bias and variance. Generalized Advantage Estimation (GAE)~\citep{schulman2015high} introduces a tunable parameter that interpolates between high-bias low-variance and low-bias high-variance estimators, and is widely adopted in actor-critic algorithms. Other methods reformulate policy updates in terms of advantages. Advantage-Weighted Regression (AWR)~\citep{peng2019advantage} avoids policy gradients and instead performs weighted regression over actions. P3O~\citep{fakoor2020p3o} combines on-policy and off-policy learning by applying advantage-weighted importance sampling to stabilize updates. 
In the offline RL setting, advantage estimates are often used to filter experience and address distributional shift. Advantage-based data selection~\citep{kostrikov2021offline} discards transitions with low advantage, helping to focus learning on high-quality samples. Additional work incorporates auxiliary signals into the advantage estimate. Preference-based RL~\citep{lee2021pebble} derives implicit advantage signals from human comparisons, while other approaches integrate value correction from ensemble critics or confidence measures to adjust learning.

MIRA builds on these ideas but takes a different path. Instead of replacing the estimator or introducing new objectives, it shapes the advantage using a utility term derived from a structured memory graph. This utility reflects agent experience and LLM-derived subgoals, allowing guidance without overriding reward feedback. The resulting signal is integrated into PPO’s update rule without disrupting its optimization dynamics, enabling structured shaping while maintaining scalability and convergence guarantees.

\section{Theoretical Results} ~\label{thm}
Since the utility term does not alter the policy or critic structure, and enters additively, MIRA preserves standard stability and asymptotic properties of policy gradient methods such as PPO under standard assumptions:
\subsection{Assumptions} \label{assu}
% ---------- Core assumptions ----------
\begin{ass}[Boundedness]\label{ass:bounded}
\leavevmode
\begin{itemize}
\item[a.] For all updates $k$ and all $(s, a)$
\begin{align}
|A_k(s, a)| \leq A_{\max}, \quad  0 \leq {U}_k(s, a) \leq {U}_{\max}
\end{align}
\item[b.] 
Define the scale-adjusted shaping term as: 
% \carlee{what do the $\langle\rangle$ brackets mean?} \jes{It’s the average of $A_k$ (a vector). I normalize the entire U with this batch-level average, not each pair individually. }
\begin{align}
U_k(s,a) = \bar{A}_k \cdot {U}_k(s,a), \quad \text{where } \bar{A}_k = \langle |A_k| \rangle
\end{align}
and set
\begin{align}
U_{\max} = {U}_{\max} \cdot \sup_k \bar{A}_k
\end{align}

\end{itemize}
\end{ass}

\begin{ass}[Scale control]\label{ass:scale}
\leavevmode
\begin{itemize}
  \item[a.] For all $k$, the scaling parameters satisfy:
  \begin{align}
  0 < \eta_k \leq 1, \quad \xi_k \leq \delta_t \eta_k \quad \text{for some } \delta_t \in [0, 1)
  \end{align}
  \item[b.] Asymptotically, the schedule satisfies:
  \begin{align}
  \lim_{k \to \infty} \eta_k = 1, \quad \lim_{k \to \infty} \xi_k = 0
  \end{align}
\end{itemize}
\end{ass}

\begin{ass}[Trust region]\label{ass:tr}

$$ {KL}\!\left(\pi_k,\pi_{k+1}\right)\le\frac{(1-\gamma)\,\varepsilon_k^2}{2}$$
(implied by PPO clip ratio $r_\pi\in[1-\varepsilon_k,1+\varepsilon_k]$).

\end{ass}

% \paragraph{Surrogate notation.}
% \hspace{4mm} 

% \noindent The clipped PPO surrogate is  

% $$\mathcal{L}_k^{}(\pi)=
%   \E_{d_{\pi_k}}\!\bigl[\text{clip}(r_\pi,1\!\pm\!\varepsilon_k)\,A_k\bigr].$$

% \noindent The \textbf{shaped} surrogate replaces \(A_k\) by
% \( \tilde A_k^{\eta,\xi} \doteq \tilde A_k=\eta_k A_k+\xi_k U_k\):

% $$\mathcal{L}_k^{\text{shaped}}(\pi)=
%   \E_{d_{\pi_k}}\!\bigl[\operatorname{clip}(r_\pi,1\!\pm\!\varepsilon_k)\,
%         \tilde A_k\bigr].$$

\subsection{Stability and Safety Results}
\begin{lem}[Bounded Shaped Policy Updates] \label{lem:bounded_update}
Under Assumptions~\ref{ass:bounded} (Boundedness), \ref{ass:scale} (Scale Control), and
\ref{ass:tr} (Trust Region), the magnitude of each shaped PPO policy update is uniformly bounded.
\end{lem}
\begin{proof}
By definition, the shaped advantage satisfies
\[
|\tilde A_t|
\le
\eta_t |A_t| + \xi_t |U_t|.
\]
Using Assumption~\ref{ass:bounded} and Assumption~\ref{ass:scale}(a), we obtain
\[
|\tilde A_t|
\le
\eta_t A_{\max} + \xi_t U_{\max}
\le
\eta_t A_{\max} + \delta_t \eta_t U_{\max}
\le
(1+\delta_t) A_{\max}.
\]

Under the PPO trust-region constraint in Assumption~\ref{ass:tr}, the likelihood ratio is clipped and
the score function $\nabla_\theta \log \pi_\theta(a_t|s_t)$ has bounded second moment.
Consequently, the policy gradient update
\[
\mathcal{L}_k^{ {shaped}}
=
\mathbb{E}\!\left[
\nabla_\theta \log \pi_\theta(a_t|s_t)\,\tilde A_t
\right]
\]
is uniformly bounded in norm.
\end{proof}

\begin{thm}[Non-Divergence under Trust Region]
Under Assumptions~\ref{ass:bounded} (Boundedness),
\ref{ass:scale} (Scale Control), and
\ref{ass:tr} (Trust Region),
the PPO updates computed using the shaped advantage
\(
\tilde A_t = \eta_t A_t + \xi_t U_t
\)
remain within the prescribed trust region.
In particular, the KL divergence between successive policies
is uniformly bounded, and the optimization does not diverge.
\end{thm}

\begin{proof}
At iteration $k$, PPO computes the policy update by maximizing a clipped surrogate
objective of the form
\[
\mathcal{L}_k(\theta)
=
\mathbb{E}\!\left[
\min\!\big(
r_t(\theta)\,\tilde A_t,\;
\operatorname{clip}(r_t(\theta),1-\varepsilon_k,1+\varepsilon_k)\,\tilde A_t
\big)
\right],
\]
where $r_t(\theta)=\pi_\theta(a_t|s_t)/\pi_{\theta_k}(a_t|s_t)$.
The clipping operation enforces
\[
r_t(\theta)\in[1-\varepsilon_k,\,1+\varepsilon_k],
\]
which implies the trust-region bound in Assumption~\ref{ass:tr}
(see \cite{schulman2017proximal}).

This constraint depends only on the policy parameterization and the clipping
threshold $\varepsilon_k$, and is independent of the specific form of the
advantage estimator.
Therefore, optimizing the shaped surrogate induces the same likelihood-ratio
constraint as standard PPO.

By Lemma~\ref{lem:bounded_update} (Bounded Shaped Policy Updates), the resulting
policy update has uniformly bounded magnitude.
Combining bounded update magnitude with the likelihood-ratio constraint yields
\[
\mathrm{KL}(\pi_k \,\|\, \pi_{k+1})
\le \tfrac{(1-\gamma)\varepsilon_k^2}{2}.
\]
Hence, the sequence of shaped PPO updates remains within the prescribed trust
region, and the optimization does not diverge.
\end{proof}

\subsection{Asymptotic Behavior}

\begin{thm}[Asymptotic Equivalence to PPO]
Let $\{\pi_k\}$ be the sequence of policies produced by PPO using the shaped advantage
\begin{equation}
\tilde A_t = \eta_t A_t + \xi_t U_t,
\end{equation}
where $U_t$ is computed on-policy from the same rollouts as $A_t$.
Under Assumptions~\ref{ass:bounded} (Boundedness), \ref{ass:scale} (Scale Control), and
\ref{ass:tr} (Trust Region), the shaped update is asymptotically equivalent to the standard PPO update.
In particular, any stationary point of PPO is also a stationary point of the shaped objective.
\end{thm}
\begin{proof}
The PPO policy gradient update with shaped advantage is given by
\begin{equation}
\mathcal{L}^{ {shaped}}_k
=
\mathbb{E}_{(s_t,a_t)\sim\pi_k}
\big[
\nabla_\theta \log \pi_\theta(a_t|s_t)\,\tilde A_t
\big].
\end{equation}
Substituting the definition of $\tilde A_t$ yields
\begin{equation}
\mathcal{L}^{ {shaped}}_k
=
\eta_k\, \mathcal{L}_k^{\text{PPO}}
+
\xi_k\, \mathcal{L}_k^{U},
\end{equation}
where
\begin{align}
\mathcal L_k^{\text{PPO}} &\doteq
\mathbb{E}\!\left[
\nabla_\theta \log \pi_\theta(a_t|s_t)\, A_t
\right], \\
\mathcal L_k^{U} &\doteq
\mathbb{E}\!\left[
\nabla_\theta \log \pi_\theta(a_t|s_t)\, U_t
\right].
\end{align}

By Assumption~\ref{ass:bounded}, both $A_t$ and $U_t$ are uniformly bounded.
Under the trust-region constraint in Assumption~\ref{ass:tr}, the score function
$\nabla_\theta \log \pi_\theta(a_t|s_t)$ has bounded second moment, ensuring that
both gradient components are finite.

By Assumption~\ref{ass:scale}(b),
\begin{equation}
\lim_{k\to\infty} \eta_k = 1,
\qquad
\lim_{k\to\infty} \xi_k = 0,
\end{equation}
which implies
\begin{equation}
\lim_{k\to\infty} \mathcal{L}_k = \mathcal{L}_k^{\text{PPO}}.
\end{equation}

Therefore, the shaped update converges to the standard PPO update.
Any policy $\pi^\star$ satisfying $\mathcal{L}^{\text{PPO}}(\pi^\star)=0$ is also a stationary
point of the shaped objective asymptotically.
\end{proof}

\subsection{Early-Training Advantage}

\textbf{Theorem}[reinstate] (Non-Vanishing Updates in Sparse-Reward Regimes)\textbf{.}

\begin{proof}
Recall that the shaped advantage is defined as
\[
\tilde A_t = \eta_k A_t + \xi_k U_t,
\]
and the corresponding shaped policy gradient is
\[
\mathcal{L}_k^{\mathrm{shaped}}
\doteq
\mathbb{E}\!\left[
\nabla_\theta \log \pi_\theta(a_t|s_t)\, \tilde A_t
\right].
\]
By linearity of expectation, the shaped gradient decomposes as
\begin{equation}\label{eq:grad-decomp}
\mathcal L_k^{\mathrm{shaped}}
=
\eta_k \mathcal{L}_k^{\mathrm{PPO}} + \xi_k \mathcal{L}_k^{U},
\end{equation}
where
\[
\mathcal L_k^{\mathrm{PPO}}
\doteq
\mathbb{E}\!\left[
\nabla_\theta \log \pi_\theta(a_t|s_t)\, A_t
\right],
\qquad
\mathcal L_k^{U}
\doteq
\mathbb{E}\!\left[
\nabla_\theta \log \pi_\theta(a_t|s_t)\, U_t
\right].
\]

By assumption, the expected magnitude of the PPO advantage is small,
$\mathbb{E}[|A_t|] \le \varepsilon_A$.
Under Assumption~\ref{ass:bounded}, $|A_t| \le A_{\max}$, and under the trust-region
constraint (Assumption~\ref{ass:tr}), the score function
$\nabla_\theta \log \pi_\theta(a_t|s_t)$ has bounded second moment.
Consequently, there exists a constant $C>0$ such that
\[
\bigl\| \mathcal{L}_k^{\mathrm{PPO}} \bigr\|
\le
C\,\mathbb{E}[|A_t|]
\le
C\,\varepsilon_A.
\]

Taking norms on both sides of Eq.~\eqref{eq:grad-decomp} and applying the triangle
inequality yields
\[
\bigl\| \mathcal{L}_k^{\mathrm{shaped}} \bigr\|
\ge
\xi_k \bigl\| \mathcal{L}_k^{U} \bigr\|
-
\eta_k \bigl\| \mathcal{L}_k^{\mathrm{PPO}} \bigr\|.
\]
Substituting the above bound on $\|\mathcal{L}_k^{\mathrm{PPO}}\|$ gives
\[
\bigl\| \mathcal{L}_k^{\mathrm{shaped}} \bigr\|
\ge
\xi_k \bigl\| \mathcal{L}_k^{U} \bigr\|
-
\eta_k C \varepsilon_A.
\]

Since $0<\eta_k\le 1$ by Assumption~\ref{ass:scale}, the second term is
$O(\varepsilon_A)$, and we obtain
\[
\bigl\| \mathcal{L}_k^{\mathrm{shaped}} \bigr\|
\;\ge\;
\xi_k \bigl\| \mathcal{L}_k^{U} \bigr\|
-
O(\varepsilon_A),
\]
which completes the proof.
\end{proof}

\paragraph{Remark (Bias--Variance Trade-off).} ~\label{bv}
The shaping mechanism can also be interpreted through the lens of the
bias--variance trade-off in policy gradient estimation.
Standard PPO relies on advantage estimates (e.g., GAE or Monte Carlo returns)
that depend on long-horizon rollouts and typically exhibit high variance,
particularly in early training.
In contrast, the utility signal $U_t$ depends only on the current state--action
pair and its similarity to stored high-return trajectories, yielding a
lower-variance but biased learning signal.
Initializing $\xi_0>0$ stabilizes early optimization by injecting this bias,
while the decay schedule $\xi_k \to 0$ gradually restores reliance on
reward-based advantage estimates as the policy and critic become more reliable.

\section{LLM Prompting and Reasoning} 

\label{prompt}
\subsection{Gymnasium Toy text}
\begin{wrapfigure}{r}{0.25\textwidth}
\vspace{-8mm}
    \centering
    \includegraphics[width=.9\linewidth]{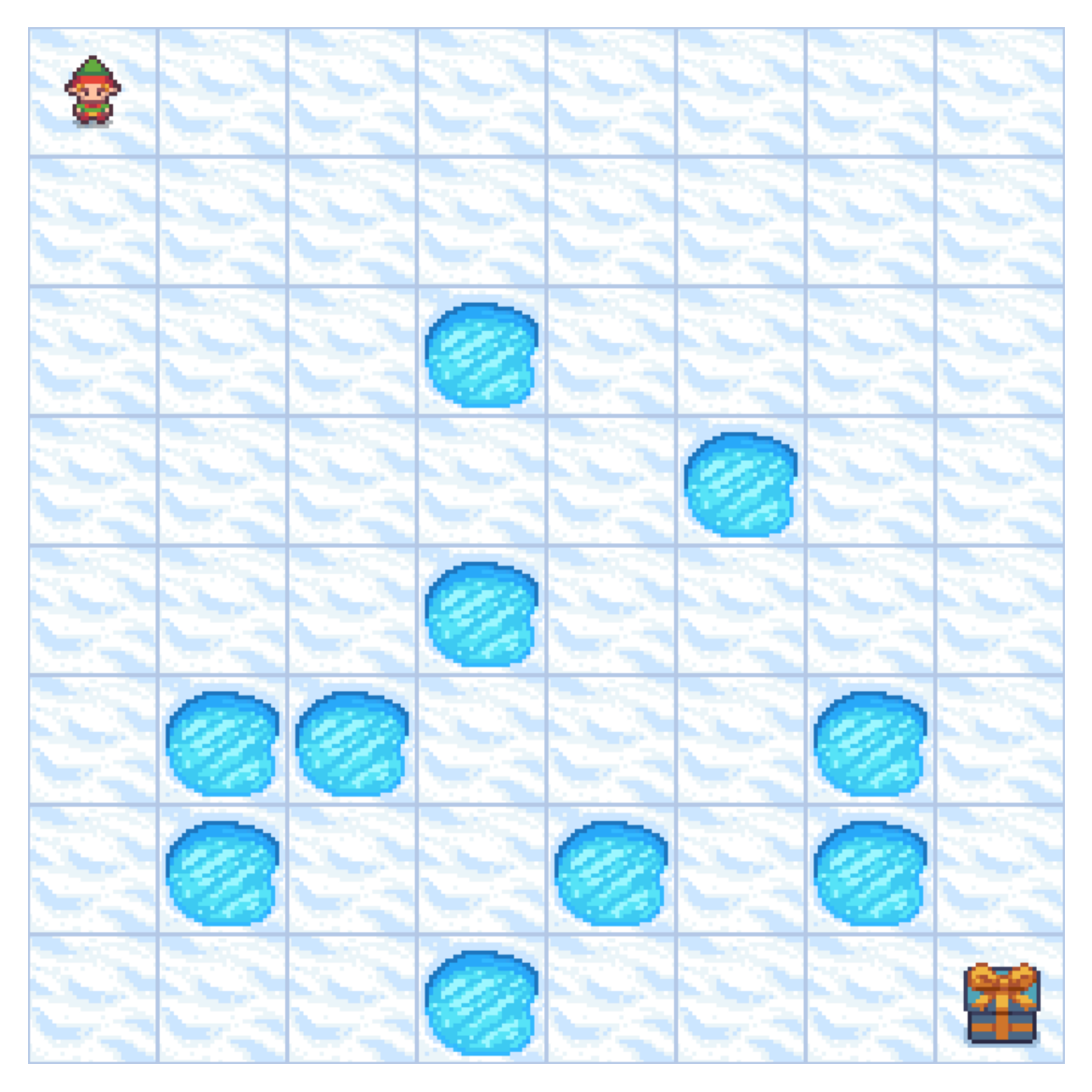}
    \caption{Frozen Lake (Gymnasium)}
    \label{fig:flakeenv}
    \vspace{-12mm}
\end{wrapfigure}
\textsc{FrozenLake} is a tabular RL environment where the agent starts in the top-left and must reach the bottom-right goal while avoiding holes. 
For \textsc{FrozenLake}, we provide the LLM with the complete map of the environment, either as an image (Figure~\ref{fig:flakeenv}) or as a serialized array representation such as \texttt{[`F’, `F’, …, `H’, `F’, …, `G']}. Though the environment is typically stochastic due to slipperiness, the LLM is instructed to assume deterministic transitions.

A\n{lthough much of the prompt is directly from the  official environment description, for clarity and reproducibility, we include the full version.} 
\begin{figure}
    \centering
    \includegraphics[width=\columnwidth]{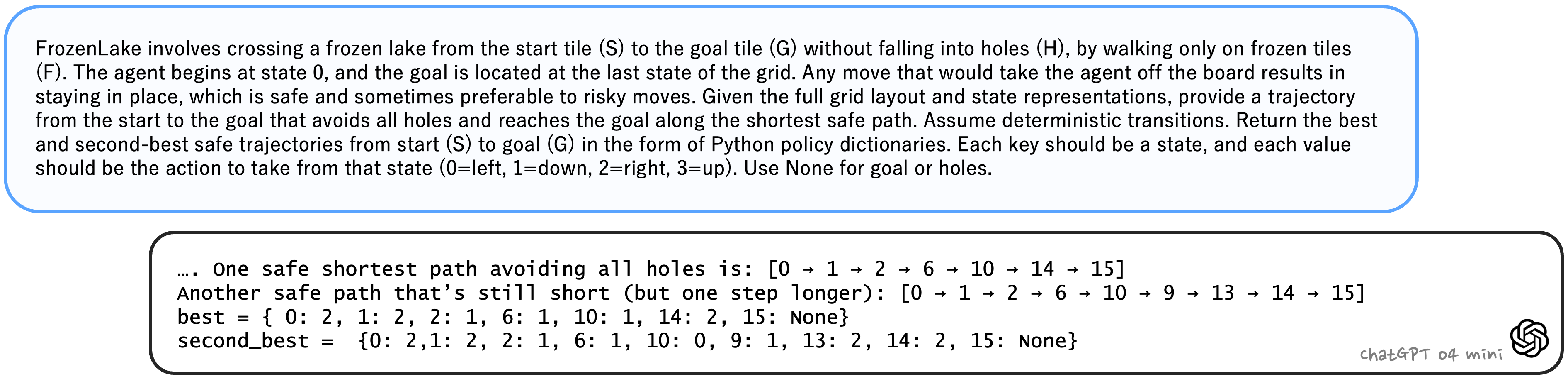}
    \caption{Prompt and response from GPT-4o-mini for the \textsc{FrozenLake} environment. The LLM receives a serialized map or a textual description and is asked to provide the best and second-best safe trajectories from start to goal, avoiding all holes under deterministic dynamics.}
    \label{fig:flake}
\end{figure}
The prompt and the LLM’s response are shown in Figure~\ref{fig:flake}.

% \n{\subsubsection{Prompt Robustness}}
% \n{To evaluate robustness to reasonable variations in prompt wording,  we repeated the  \textsc{FrozenLake} experiment using an alternative prompt with simpler phrasing but identical task information.}
% \n{In this variant, the prompt stated that ``\textit{FrozenLake is a grid where S is the start, G is the goal, F are safe tiles, and H are holes. The agent moves from state 0 using actions {0=left, 1=down, 2=right, 3=up}, and moves that go off the grid keep the agent in place and are safe. Using the grid and assuming deterministic transitions, provide a shortest safe path from S to G that avoids all holes, and return the best and second-best safe paths as Python dictionaries mapping each visited state to its action, using None for the goal or holes.'}'
%  Figure~\ref{fig:twoprompt} compares MIRA under the original and alternative prompts. The learning curves and final returns are closely aligned, showing that MIRA’s performance is stable under natural variations in how the environment description is presented.}

%   \begin{wrapfigure}{r}{0.44\textwidth}
% \vspace{-3mm}
%     \centering
%     \includegraphics[width=\linewidth]{figs/twoprompt_v2.png}
%     \caption{\n{FrozenLake robustness to prompt wording. MIRA achieves similar performance under the original and alternative prompts, showing stability to natural variations in task description.}}
%     \label{fig:twoprompt}
%     \vspace{-2mm}
% \end{wrapfigure}

\subsection{Standard and Custom MiniGrid and BabyAI Environments}
Each environment was chosen for a specific purpose: \textsc{RedBall} involves short-horizon navigation and fast spatial goal acquisition. \textsc{LavaCrossing} introduces irreversible transitions that require long-horizon planning to avoid dead ends.   \textsc{DoorKey} requires the agent to acquire a key, unlock a door,  and reach the goal, forming a delayed dependency chain that challenges temporal credit assignment. 
\textsc{RedBlueDoor} tests the agent’s ability to commit to a correct action sequence, as opening the blue door prematurely ends the episode. At last, \textsc{Distracted DoorKey} introduces BabyAI-style distractors (e.g., irrelevant balls and boxes) alongside the original multi-step dependencies of  \textsc{DoorKey}, allowing us to test whether the LLM can generalize across known task elements and maintain coherent subgoal proposals under added visual distraction.
For standard MiniGrid and BabyAI environments, we used the environment descriptions provided on the MiniGrid website. For our custom environment (\textsc{Distracted DoorKey}), we mimicked the phrasing and structure of the official MiniGrid descriptions (Figure~\ref{fig:distracted}). Unlike in \textsc{Frozen Lake}, obtaining useful trajectories here was not as straightforward. MiniGrid-style environments often required multi-round prompting to obtain meaningful and desired outputs. Moreover, instead of providing an image of the environment, we found it more effective to use a textual description.  This helped reduce confusion and encouraged the LLM to understand that object locations (e.g., the key, door, and agent in \textsc{DoorKey}) can vary across episodes.

\subsection{LLM Reasoning Patterns Across Models} ~\label{reasoning} 
\begin{figure}
% \vspace{-8mm}
    \centering
    \includegraphics[width=\linewidth]{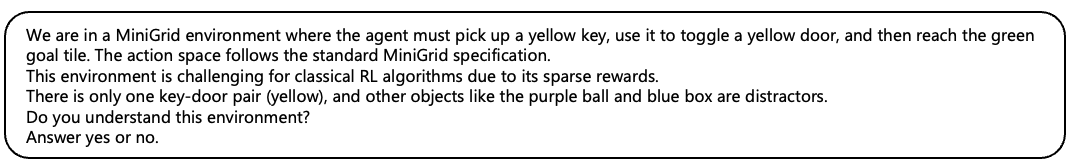}
    \caption{Prompt to Offline LLM for the custom MiniGrid variant \textsc{Distracted DoorKey}. The prompt describes the task setting, object roles, and challenges, and asks the LLM to confirm understanding before suggestions.}
    \label{fig:distracted}
    % \vspace{-6mm}
\end{figure}
We observed that different LLMs produced very different memory graphs. To better understand how different models reason about these environments, we recorded
 not only their output trajectories but also their internal reasoning processes.  For model that include system-level thinking (e.g., GPT-o4-mini), this was extracted directly from the response. For models that do not expose intermediate reasoning  (e.g., Claude 3), we followed up with an auxiliary prompt such as: ``Give me your reasoning as to why you chose this sequence of actions.''

\begin{figure}
    \centering
    \includegraphics[width=0.95\textwidth]{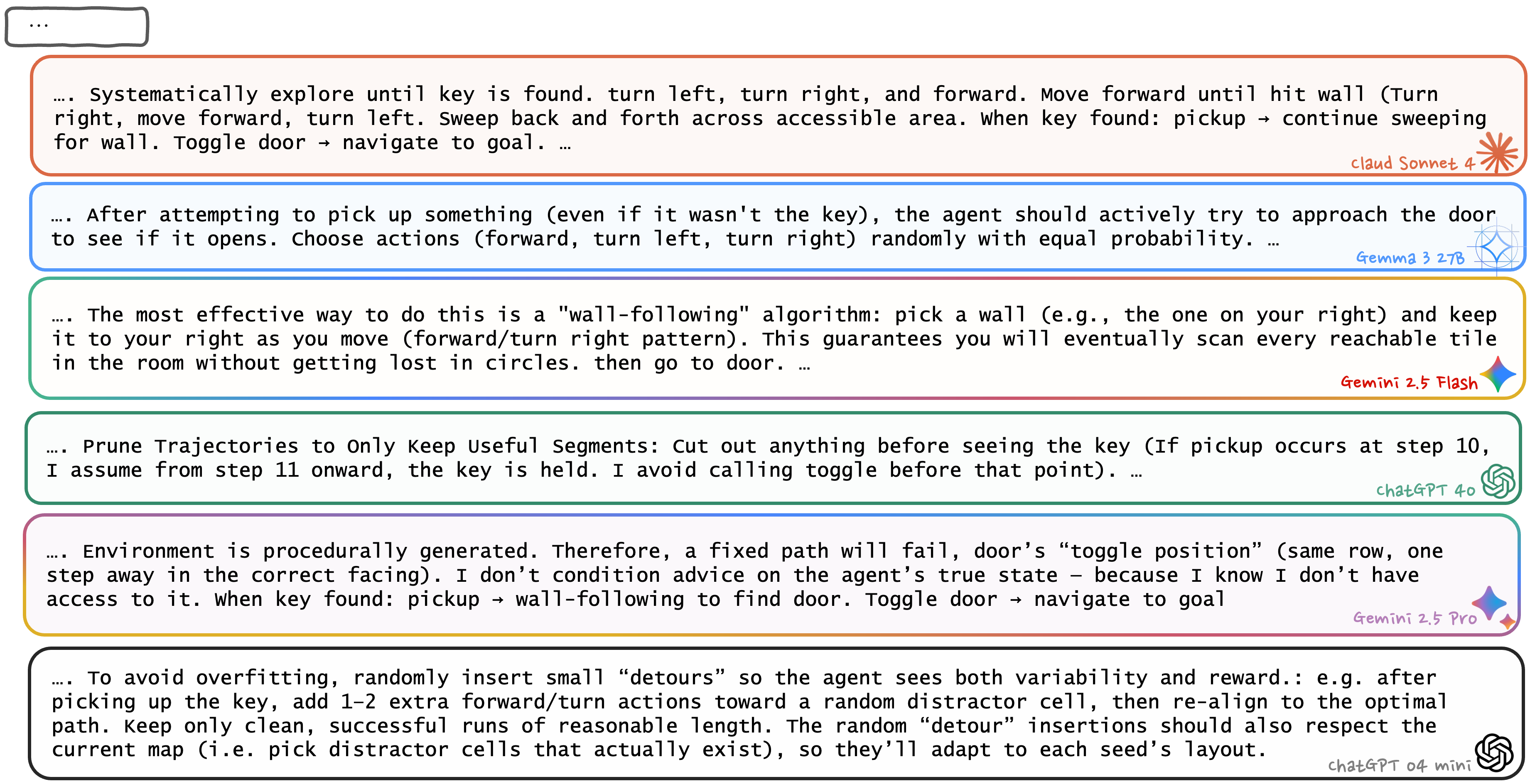}
    \caption{Reasoning traces produced by different LLMs in response to our custom environment prompt as part of ``Offline LLM" prompting. After confirming they understood the environment, each model was asked: ``If you were to give an RL agent useful trajectories to help solve this, what would you do?” For models that do not output internal reasoning (e.g., Claude), we issued a follow-up prompt requesting their thought process. We omit repeated environment restatements and show only the key parts where the model explains how it decided on the action sequence.}
    \label{fig:llms}
\end{figure}
These responses were not used in the MIRA framework, but we found them surprisingly revealing. Despite receiving identical prompts, the models relied on starkly different reasoning strategies. This divergence gave us unexpected insight into how various LLMs process spatial structure, interpret decision sequences, and reason about reinforcement learning dynamics and learning objectives. Differences that, in turn, shape the quality of their output trajectories.
In Figure~\ref{fig:llms}, we present reasoning snippets from the LLMs' outputs. We omit the initial sections where models repeat the prompt or restate the environment description, and instead highlight the specific reasoning steps that led each model to select a particular trajectory. The influence of these differing reasoning strategies on RL performance is reflected in the return curves shown in Figure~\ref{fig:all_abs}.

\subsection{Case Study: \textbf{Distracted DoorKey}} \label{r_unr}

In the ablation study presented in Subsections~\ref{sub:ablation}, GPT-o4-mini and Gemini return different outputs when presented with the same situation. Here, we provide the exact prompt and reasoning traces. As shown in Figure~\ref{fig:gptgemm}, both responses appear plausible at a surface level, but only one is consistent with the task dynamics: given that sufficient exploration has already occurred, the key is likely collected, making suppression of the corresponding action the correct response.  In this case, the divergence leads to a drop in performance under the misaligned output.

\begin{figure}[H]
    \centering
    \includegraphics[width=\linewidth]{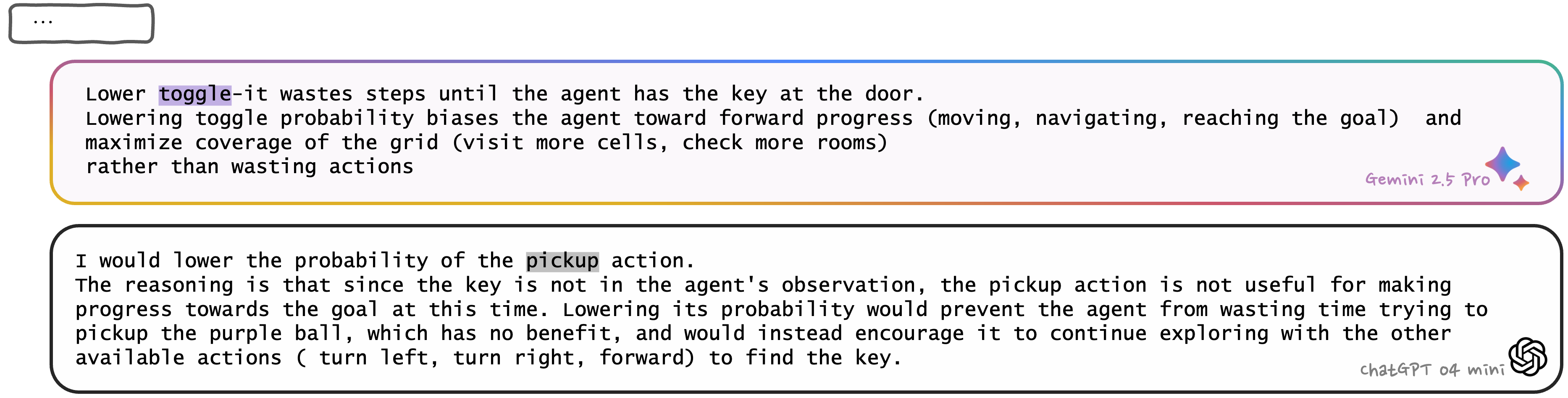}
    \caption{Reasoning traces produced by Gemini and ChatGPT under ``Online LLM'' prompting. The prompt emphasizes that sufficient exploration has already been performed and, from the partial observation, no key is visible. A (flawed but plausible) line of reasoning is that the agent must still be in the phase of searching for the key, so reducing the probability of toggle appears reasonable to prioritize movement actions for exploration.}
    \label{fig:gptgemm}
\end{figure}

\section{Memory Graph Construction Details} \label{graph_detail}
In this section, we further explain the procedure for initializing, updating, and pruning MIRA’s memory graph. As discussed in Section~\ref{meth}, the initial memory graph is constructed from offline LLM-generated suggestions. Once built for a specific environment, this graph can be reused across training episodes or even across agents within the same task.
Since MIRA is designed to generalize across diverse settings, figure~\ref{fig:graph} illustrates how the framework accommodates environments with a single terminal objective as well as tasks with multiple independent objectives

 Given that each task differs slightly, we largely focus our detailed explanation on \textsc{DoorKey} from the MiniGrid suite for the rest of the subsections, as it contains multiple subgoals and is sufficiently complex to show the dynamics of the graph clearly.

\subsection{Initialization}
As shown in figure~\ref{fig:llms}, GPT-o4-mini tends to generate trajectory segments that begin after the key is picked up, with the subgoal ``toggle the door''.  In contrast, models like Claude tend to produce longer, full trajectories from the beginning.
% \begin{wrapfigure}{r}{0.42\textwidth}
% \vspace{-2mm}
%    \centering
% \includegraphics[width=.95\linewidth]{figs/graph.png}
% \caption{Visualization of MIRA’s memory graph. Trajectory segments $\tau_j$ are grouped under subgoal nodes $\kappa_\ell$, which represent abstract intermediate objectives. Subgoals can be shared across multiple final goals (e.g., $\kappa_1$ connects to both $\textsl{g}_\triangleright$ and $\textsl{g}_\triangleright'$), enabling reuse of common behaviors. The graph evolves during training through agent discovery and LLM-guided grafts. }
%     \label{fig:graph} 
%     \vspace{-12mm}
% \end{wrapfigure}
 Interestingly, segmented trajectories are often more useful in this environment.
Since the environment is partially observable and reinforcement learning relies heavily on exploration, allowing the agent to figure out how to reach the key on its own helps it understand the overall layout of the environment better. 
Once the key is acquired, there is a higher chance that the door has already entered the agent’s observation window, making memory-guided navigation toward the door more effective. 

In addition to segments, the LLM also infers subgoals $(\kappa_\ell)$. While the obvious ones are ``Pick up key," ``Open door,'' and ``Reach goal,'' o4-mini returns more detailed versions like:
\begin{center}
\textit{$\kappa_1:$ Go to key → $\kappa_2:$ Pick up key → $\kappa_3:$ Go to door → $\kappa_4:$ Toggle door → $\textsl{g}_\triangleright$: Go to goal}.
\end{center}
This fine-grained subgoal sequence reflects the environment’s control logic: the ``open door" action is valid only if the agent is positioned one step away, properly aligned, and facing the door.

Moreover, for each memory segment, an estimated subgoal reward $\hat r_m$ is stored in the node that reflects geometric progress toward completing its associated subgoal. In discrete environments (e.g., MiniGrid), progress is computed using the normalized shortest-path distance between the states in the segment and the subgoal’s target location.

\subsection{Agent-Induced Updates}

During training, new nodes are added to the memory graph or existing ones are updated whenever the agent produces trajectory segments that improve upon what the graph already stores, either by providing a new segment for a (sub)goal or by achieving a higher estimated return than the current entry for that (sub)goal. For example, if the agent independently discovers a shorter path to the key and then follows a memory-guided trajectory to reach the door or goal, the resulting sequence is added as a new node in the graph. Likewise, if the agent successfully executes a trajectory that was initially stored with low confidence from the offline LLM, we treat this behavior as implicit validation and increase the confidence of the corresponding memory node.

The memory graph remains lightweight throughout training. Each node stores a trajectory segment and metadata, and the total graph size stays compact. Compared to experience replay buffers in standard off-policy RL methods, which retain large volumes of data, the memory graph introduces negligible computational and memory overhead.
To maintain compactness, unused nodes are periodically pruned based on access frequency.
Each memory node tracks an access counter, which is reset every time the node is used. Nodes that are not accessed for 100 episodes are pruned, except for those corresponding to final goal trajectories ($\textsl{g}_\triangleright$), which are retained since the agent might not have reached them early on, but they are essential for guiding successful completion later in training.

\n{Algorithm~\ref{alg:maintain-graph} summarizes the add, update, and prune operations that govern how the memory graph is maintained during training.}

\subsection{Online Grafting and Triggers}
Since the agent has a limited number of steps per episode, it may fail to reach any subgoal (e.g. ``Open Door'') with a matching trajectory in the memory graph early on, preventing utility shaping from activating.  To address this, MIRA includes a fallback mechanism: if the computed utility $U$ is entirely zero for $N$ consecutive episodes, the agent triggers an online LLM query. These online queries return short plans (e.g., ``turn left, move forward, toggle”) based on the agent’s partial observations to help the agent reorient. Once screened for quality, the new suggestion is grafted into MIRA.
Another way online LLM queries contribute is by influencing the agent’s policy preferences directly through soft logit injection. Importantly, the online LLM is constrained by the same partial observability as the agent. It does not receive access to the full environment state and therefore cannot, for example, determine the presence of a key elsewhere in the grid. Furthermore, since inventory status is not part of the agent’s observation space, the LLM is unaware of whether the agent has picked up the key. Instead, the LLM receives a batch of recent partial observations and must infer from them whether any meaningful guidance can be offered.

\n{\section{Utility Computation} \label{utility_detail}
In this section, we provide a detailed explanation of the utility computation introduced in the main text, clarifying how each component contributes to the shaping term. The utility measures how closely the agent’s trajectory aligns with high-return segments stored in the memory graph. When a reference trajectory is matched, utility values are assigned based on reverse-aligned similarity with reference trajectories; unmatched steps receive zero utility. Below, we describe the role of each factor in the computation (Equation~\ref{utility_eq}) and provide pseudocode for the full procedure.}
\n{
\subsection{Similarity Score}

The similarity function assigns a score based on the information extracted from the agent’s and the reference transition’s observations. Depending on the environment, these observations may include position, orientation, or action. For example, in \textsc{FrozenLake}, observations are discrete and include only position, so direction is omitted in the similarity check. High similarity indicates that the agent is reproducing a locally meaningful portion of a successful stored trajectory, whereas low or zero similarity reflects small or no meaningful match. }

\n{
\begin{algorithm}[H]
\caption{Evolving Memory Graph During Training}
\label{alg:maintain-graph}
\begin{algorithmic}

\REQUIRE Memory graph $\mathcal{G}$, 
new trajectory and metadata $(\tau', \zeta', \hat r', c')$,  
prune window $W$ \gap

\STATE $\mathcal{M}_\zeta \leftarrow \{ m \in \mathcal{G} : \zeta_m = \zeta' \}$

\IF{$\text{source} = \textsc{OnlineLLM}$}
    \IF{$\neg \textsc{screening}$}
        \STATE return $\mathcal{G}$ \COMMENT Discard online LLM suggestion
    \ENDIF
\ENDIF

\IF{$\mathcal{M}_\zeta = \emptyset$  \COMMENT No existing segment for this subgoal } 
    \STATE Create node $m' \leftarrow (\tau', \zeta' , \hat r)_{c'}$ 
    \STATE Initialize $\text{access}_{m'} \gets 0$
    \STATE Insert $m'$ into $\mathcal{G}$
\ELSE
    \STATE $m \leftarrow \arg\max_{m \in \mathcal{M}_\zeta} \hat r_{m}$
    \IF{$\hat r' > \hat r_m$}
        \STATE $(\tau_m, \hat r_m) \gets (\tau', \hat r')$
        \STATE $c_m \gets c'$
    \ENDIF
    \IF{$\text{source} = \textsc{Agent}$}
        \STATE $c_m \gets \min(1, c_m + \Delta_c)$ \COMMENT Agent validation increases confidence
    \ENDIF
\ENDIF

\FOR{each node $m \in \mathcal{G}$}
    \IF{$\text{access}_m$ = 0 for $W$ episodes}
        \STATE Remove $m$ from $\mathcal{G}$ \COMMENT Prune nodes unused within inactivity window
    \ENDIF
\ENDFOR

\STATE return $\mathcal{G}$

\end{algorithmic}
\end{algorithm}}

\begin{algorithm}
\caption{Similarity Score $ \mathcal{s}$}
\label{alg:sim}
\begin{algorithmic}
\REQUIRE Agent $x_a$ and Reference $x_m$ annotated transition (metadata)  \gap
\IF{(pos., dir.) $\in (o_a, o_m) $ match \& $a_a = a_m$}
    \STATE \textbf{return} ${\mathcal{s}} = $ high\_sim \COMMENT (1) \gap
\ELSIF{pos. $\in (o_a, o_m) $ match \& $a_a = a_m$}  
    \STATE \textbf{return} ${\mathcal{s}} = $ mod\_sim \COMMENT not align direction (0.7)\gap
\ELSIF{$(d_a \in o_a) \pm 1 \bmod 4 = d_m \in o_m $} 
    \STATE \textbf{return} ${\mathcal{s}} = $  low\_sim \COMMENT action aligned direction \gap(0.4)
\ELSE
    \STATE \textbf{return} ${\mathcal{s}} = $ no\_sim \COMMENT (0)
\ENDIF 
\end{algorithmic}
\end{algorithm}

\n{
\subsection{Goal Alignment}
The goal-alignment factor $\rho$ scales utility based on how closely the subgoal associated with a matched memory segment relates to the subgoal that the current transition corresponds to. Although the RL agent never observes subgoal labels, the \textit{environment state} uniquely identifies which subgoal phase the transition belongs to, for example, in \textsc{DoorKey} environment, whether the agent is still approaching the key or has already picked it up and is proceeding toward the door.
Each memory node carries a subgoal label generated by the LLM. These subgoal descriptions consistently identify (i) the object or region involved and (ii) the high-level action applied to it. We extract these two components through simple rule-based parsing over the LLM-generated subgoals, yielding an entity token (e.g., key, door, ball) and an action-phase tag (e.g., navigate). The alignment score $\rho$ is then calculated as the Jaccard similarity between the token pair of the memory node’s subgoal and the token pair associated with the transition under evaluation. 
As a result, locally similar transitions only contribute to utility when they align semantically with the relevant subgoal, preventing behaviorally similar but semantically unrelated memory segments from influencing the utility signal.

\begin{algorithm}[H]
\caption{Goal Alignment $\rho$}
\label{alg:jaccard}
\begin{algorithmic}

\REQUIRE Agent $\zeta_a$ and Reference $\zeta_m$ subgoal  \gap
\STATE $t_a \leftarrow  \textsc{tokens}(\zeta_a)$ \COMMENT Agent entity–phase token \gap
\STATE $t_m \leftarrow  \textsc{tokens}(\zeta_m)$ \COMMENT Memory entity–phase token 

\STATE $I_{\cap} \leftarrow t_{m} \cap t_a$
\STATE $I_{\cup} \leftarrow t_m \cup t_a$
\RETURN $\rho = {|I_{\cap}|}/{|I_{\cup}|}$
\end{algorithmic}
\end{algorithm}

With the similarity score $\mathcal{s}$ and alignment factor $\rho$ established, we combine them with the memory-stored quantities  $\hat r_m$  (estimated subgoal reward) and $c_m$ (LLM confidence) to construct the utility used in advantage shaping.}

\begin{algorithm}
\caption{Compute Utility Score}
\label{alg:util}
\begin{algorithmic}

\REQUIRE Agent $\tau_{\text{agent}}$ and Reference $\tau_{\text{m}}$ trajectory  \gap
\STATE $x \doteq (o, a, r, \text{meta})$ \COMMENT Denote a transition with metadata (e.g. subgoals)
\STATE Initialize $U \gets [0, \dots, 0]$ \gap
\STATE Align the tail of $\tau_{\text{agent}}$ to length of $\tau_{\text{m}}$\gap
\FOR{each $(x_a, x_m) \in (\tau_{\text{agent}}^{\text{tail}}, \tau_{\text{m}})$}\gap
    \STATE ${\mathcal{s}} \leftarrow \mathcal{s}((o_a, a_a) , (o_m, a_m))$ \COMMENT Compute similarity \gap
    \STATE $\rho \leftarrow \rho(\zeta_a, \zeta_{m})$ \COMMENT Compute goal aligment factor \gap
    \STATE $ u  \leftarrow c_m \cdot \hat{r}_m \cdot \rho \cdot \mathcal{s}$
    \STATE Assign $u$ to corresponding index in $U$
\ENDFOR
\STATE \textbf{return} ${U}$
\end{algorithmic}
\end{algorithm}

\section{Extended Experimental Studies} \label{ret_detail}

\begin{wrapfigure}{r}{0.44\textwidth}
\vspace{-18mm}
    \centering
    \includegraphics[width=\linewidth]{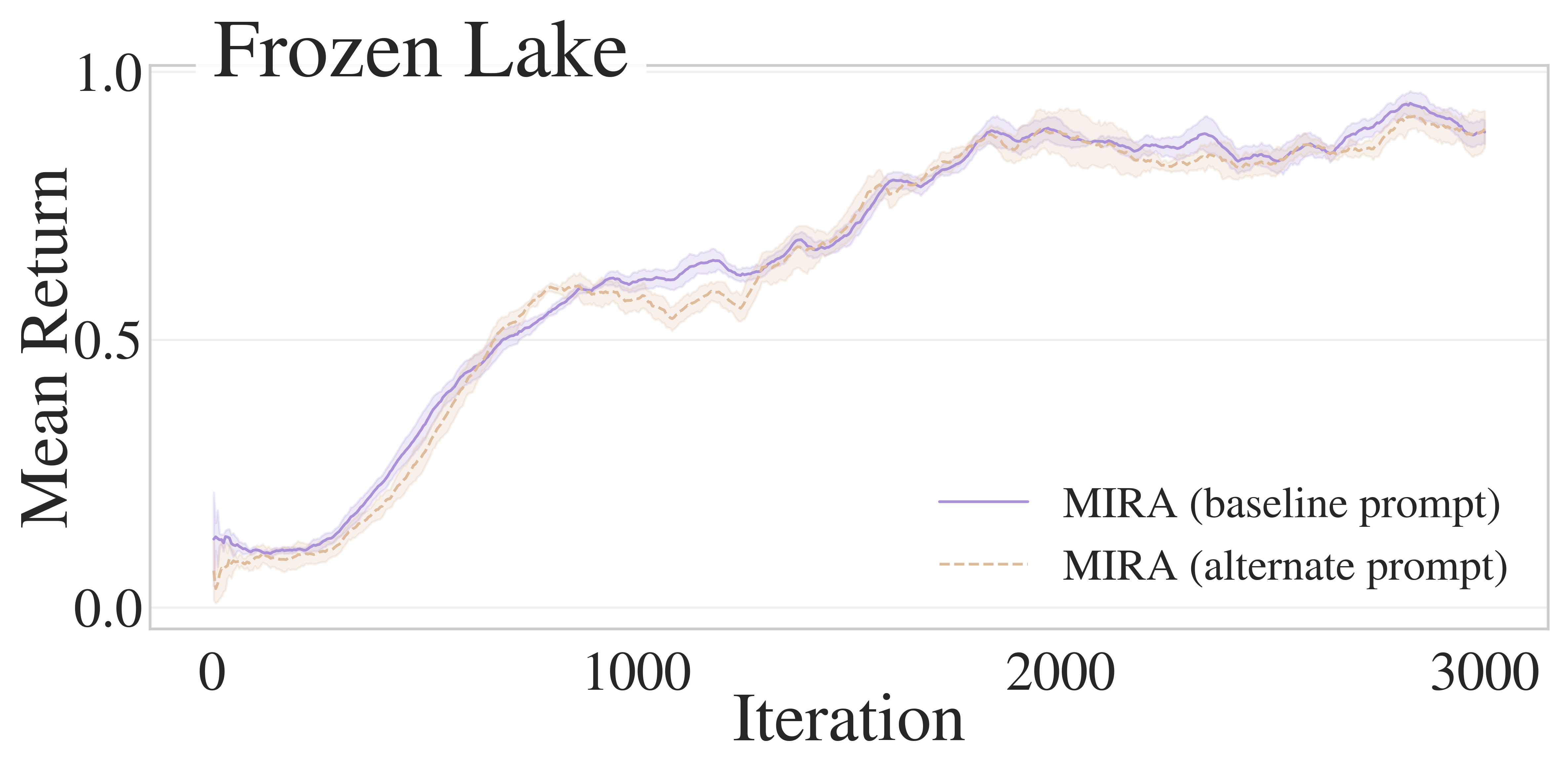}
    \caption{\n{FrozenLake robustness to prompt wording. MIRA achieves similar performance under the original and alternative prompts, showing stability to natural variations in task description.}}
    \label{fig:twoprompt}
    \vspace{-5mm}
\end{wrapfigure}

\n{\subsection{Sensitivity and Robustness Studies}}

\n{\subsubsection{Prompt Robustness}}
T\n{o evaluate robustness to reasonable variations in prompt wording,  we repeated the  \textsc{FrozenLake} experiment using an alternative prompt with simpler phrasing but identical task information.}
\n{In this variant, the prompt stated that ``\textit{FrozenLake is a grid where S is the start, G is the goal, F are safe tiles, and H are holes. The agent moves from state 0 using actions {0=left, 1=down, 2=right, 3=up}, and moves that go off the grid keep the agent in place and are safe. Using the grid and assuming deterministic transitions, provide a shortest safe path from S to G that avoids all holes, and return the best and second-best safe paths as Python dictionaries mapping each visited state to its action, using None for the goal or holes.'}'
 Figure~\ref{fig:twoprompt} compares MIRA under the original and alternative prompts. The learning curves and final returns are closely aligned, showing that MIRA’s performance is stable under natural variations in how the environment description is presented.}
 % \pagebreak

\n{\subsubsection{Threshold Sensitivity}}

  \begin{wrapfigure}{r}{0.44\textwidth}
\vspace{-5mm}
    \centering
    \includegraphics[width=\linewidth]{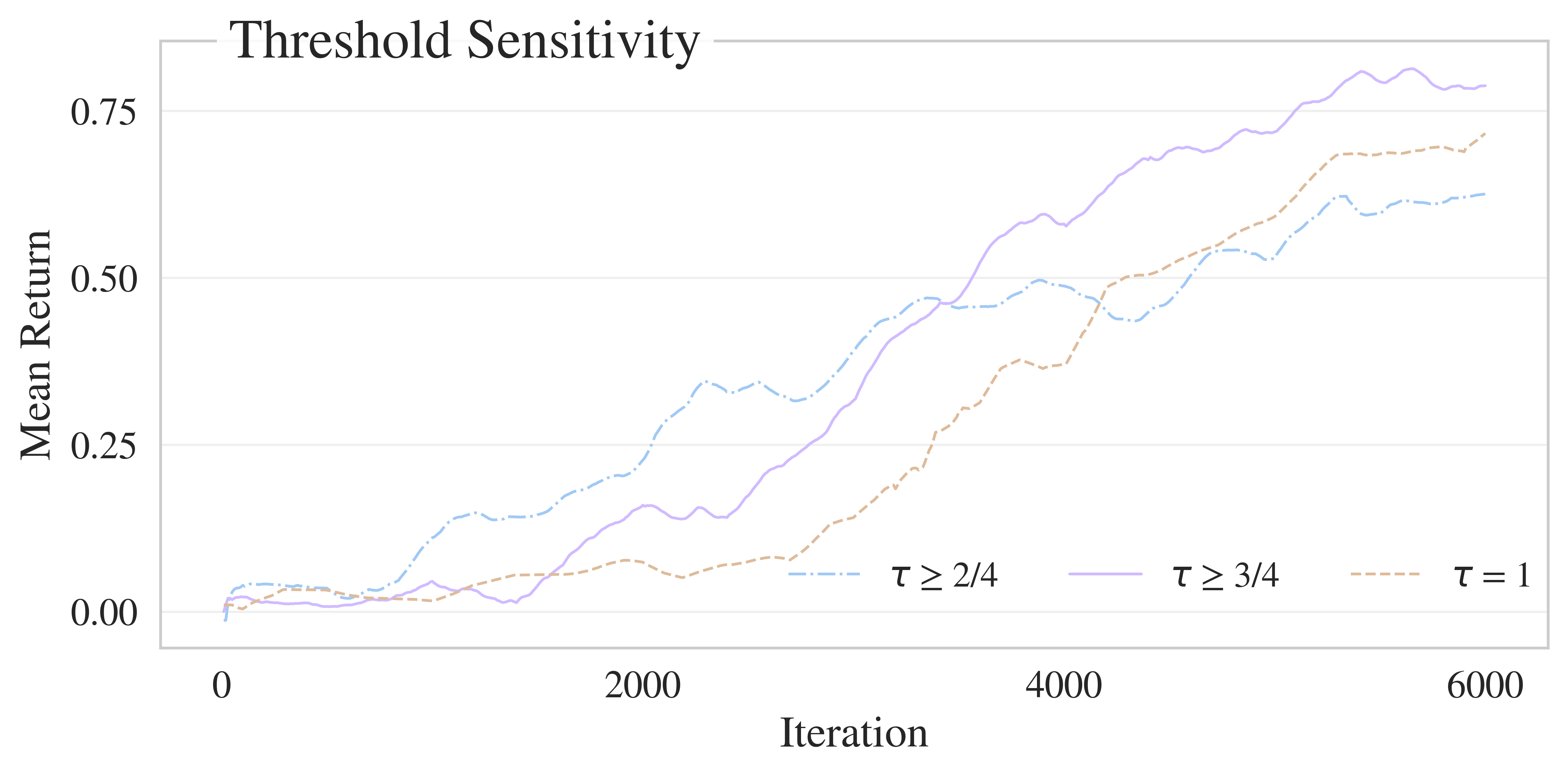}
    \caption{\n{Sensitivity to the screening threshold. Lenient thresholds graft more candidate nodes into memory early, producing broader shaping during exploration but a slower overall improvement rate. Stricter thresholds delay graph growth yet yield sharper mid-training gains once high-confidence nodes appear. All settings converge to a narrow performance band.}}
    \label{fig:threshold}
    \vspace{-5mm}
\end{wrapfigure}
T\n{o assess the sensitivity of the screening rule, we conduct a study that varies the acceptance threshold while keeping the overall method fixed.  The threshold, therefore, controls when and how densely the graph is populated, not whether shaping exists at all. 
The main results of the paper uses $k=3$ LLM completions, which was chosen for efficiency across all experiments. For this diagnostic test only, we set $k=4$ to obtain cleaner fractional thresholds corresponding to meaningful agreement levels on the \textsc{Doorkey} environment. We evaluate three settings: a lenient majority rule ($\tau \ge 2/4$), a stricter majority rule ($\tau \ge 3/4$), and unanimous agreement ($\tau = 1$). 
}

Figure~\ref{fig:threshold} shows that the threshold primarily affects the early phase of learning.
With a lenient threshold, more candidate suggestions are grafted as healthy graph nodes and added to the memory. This leads the agent to receive utility bonuses on more states while it is still exploring. As a result, even if some early nodes may be slightly misaligned, the overall trend still pushes the agent toward regions that have higher returns. It is then up to the agent to correct those nodes through experience or increase their confidence if they turn out to be helpful, which explains why the overall improvement can be slower in this setting.
In the stricter settings, fewer suggestions pass the screening, so graph growth is delayed and the shaping signal remains sparse. Early learning progresses more slowly, but once these high-confidence nodes enter the graph, they produce larger and more coherent utility bonuses along trajectories that already correlate with high return, which creates the steeper rise visible in the mid-training region. The near-final performance still lies in a narrow band, since the shaping term is bounded and eventually dominated by the learned value function.
This behavior is consistent with the design of our utility-based shaping, using the LLM’s possibly imperfect but often still useful knowledge to accelerate the initial learning phase while ensuring stable progress as more reliable evidence accumulates.

\pagebreak

\subsection{Early Advantage Dynamics}
\begin{wrapfigure}{r}{0.45\textwidth}
\vspace{-3mm}
   \centering
\includegraphics[width=\linewidth]{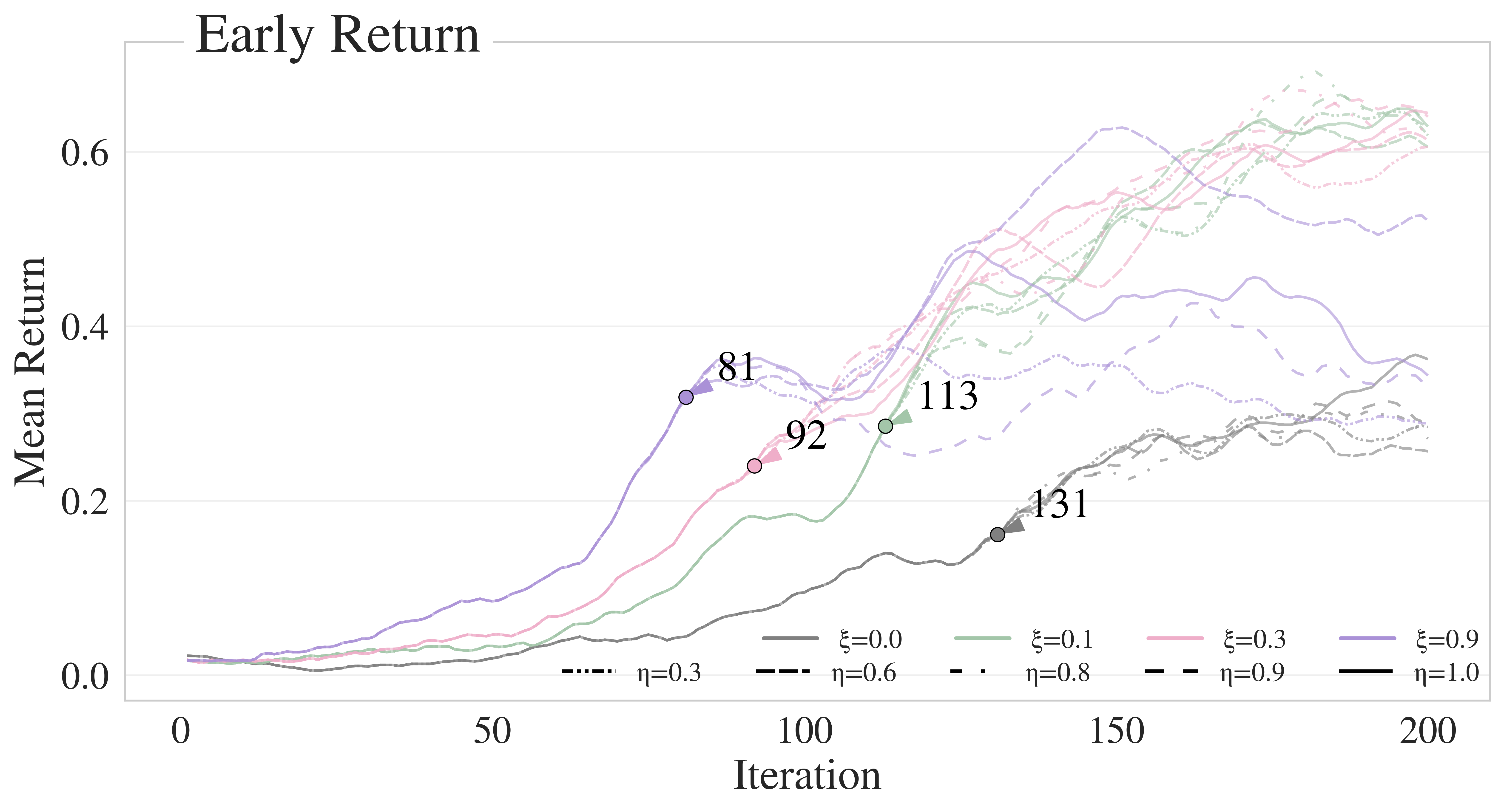}
    \caption{Return curves for different $\eta$ values under fixed $\xi$ settings. Markers indicate the first iteration where performance begins to diverge, signaling when $A_t$ starts to meaningfully affect learning. Early on, the critic signal is weak, and $\tilde{A}_t$ is driven mostly by the utility term. When $\xi$ is large enough, shaping accelerates the critic’s contribution by up to 50 iterations and leads to around 2.5× higher return compared to the unshaped case.}
    \label{fig:early}
    \vspace{-5mm}
\end{wrapfigure}
Figure~\ref{fig:early} provides empirical support for the central intuition behind our shaping formulation. 
We plot return curves for each $\xi$ group (color), across different $\eta$ values (line style). 
Early in training, return curves within each $\xi$ group remain tightly clustered, indicating that $A_t$, the critic’s estimate, provides little useful signal, regardless of how it is weighted. 
Divergence points, marked on the figure, denote the first iteration where the return spread across $\eta$ values exceeds a certain threshold, signaling that $A_t$ has begun contributing meaningfully to the shaped advantage $\tilde{A}_t = \eta_t A_t + \xi_t U_t$.

In the absence of shaping ($\xi = 0$, gray lines), this occurs relatively late (iteration 131), whereas with shaping ($\xi > 0$), it happens substantially earlier (iterations 81–113, depending on $\xi$). This shows that the utility term not only supports early learning but also accelerates the emergence of a reliable critic. These results validate our choice to softly shape advantages, and emphasize the importance of carefully tuning $\xi$ and $\eta$: insufficient shaping slows critic learning, which in turn leads to substantially lower mean returns.

\paragraph{Remark (Optimization Landscape in Sparse-Reward Regimes).}
In sparse-reward environments, standard policy gradient methods such as PPO may
exhibit near-zero expected gradients in early training, as reward-based
advantage estimates $A_t$ are often uninformative until a successful trajectory
is observed.
The non-vanishing update result (Theorem~\ref{thm:non_vanishing}) implies that the
proposed shaping objective induces a consistent gradient signal
$\xi_k \nabla_\theta \mathbb{E}[U]$ even when reward-based advantages are weak.
This additional structure modifies the local optimization landscape by providing
an informative descent direction derived from trajectory similarity, thereby
facilitating earlier and more stable optimization.

\subsection{Relative Wall Time}

\begin{wrapfigure}{r}{0.5\textwidth}
\vspace{-5mm}
   \centering
\includegraphics[width=\linewidth]{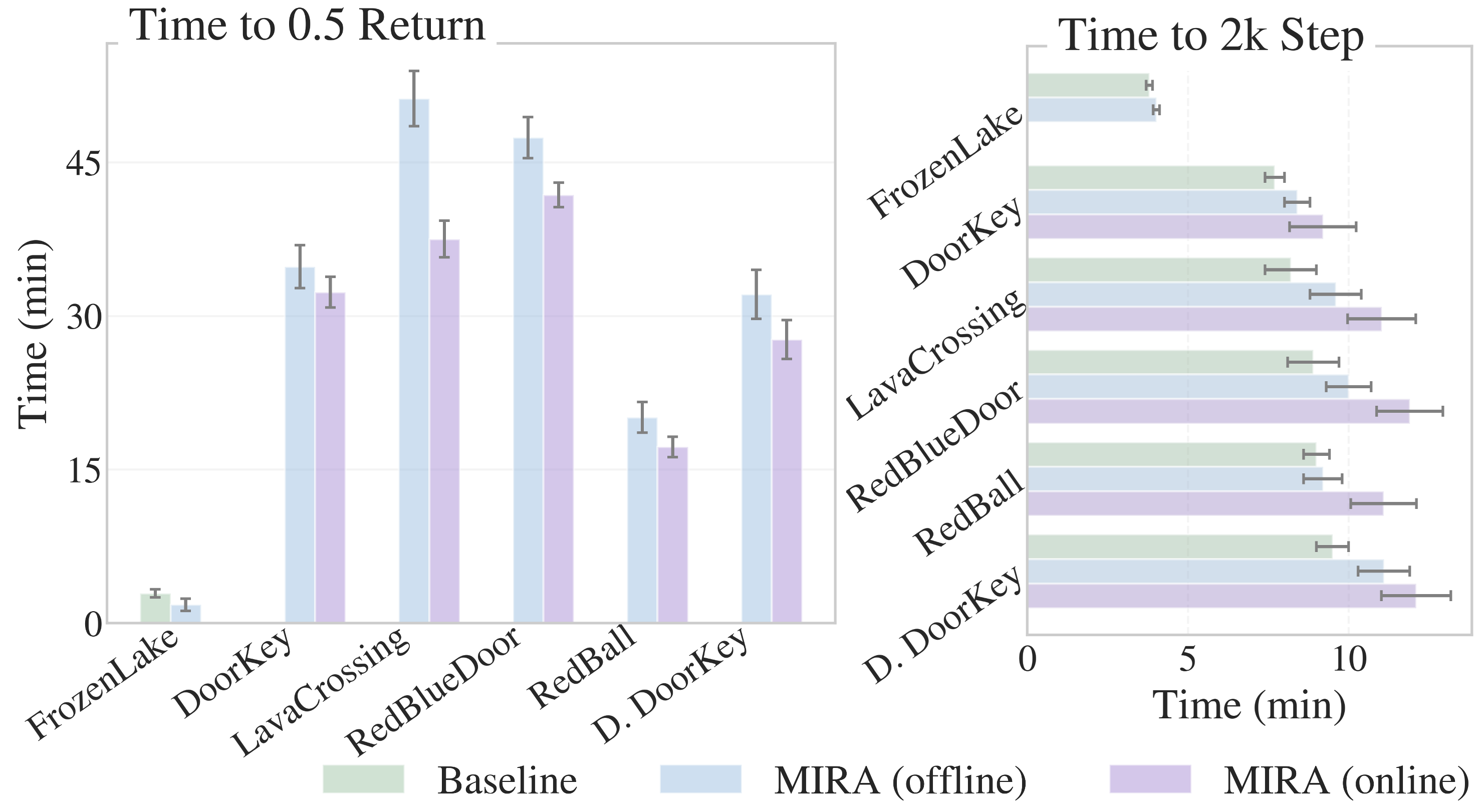}
    \caption{Wall-clock runtimes across environments. Time required to reach a 0.5 return (left): PPO reaches 0.5 only on \textsc{Frozen Lake}, while both MIRA variants converge across tasks.  Runtime for 2k training steps (right): Online MIRA incurs extra overhead from initial LLM queries, but this cost reduces wasted exploration and leads to faster convergence in terms of overall wall time.}
    \label{fig:wall}
    \vspace{-3mm}
\end{wrapfigure}
We measure relative wall-clock time as the end-to-end runtime per iteration to assess each method’s computational burden. Environments with a more complex step logic, such as \textsc{Distracted DoorKey}, which involves door toggling, key collection, and distractor dynamics, incur higher per-step simulation costs. 
Tasks like \textsc{RedBlueDoor} and \textsc{LavaCrossing} further increase runtime through frequent failures that trigger repeated episode resets and buffer re-initializations. In contrast, \textsc{Frozen Lake}’s tabular, low-dimensional transitions execute very quickly, so all methods complete rapidly (we do not run the online variant here since the offline approach suffices). Occasional LLM queries introduce network latency that further raises wall time in the slower domains. As a result, relative wall time grows with both the intrinsic simulation complexity of the environment and any additional algorithmic overhead (e.g., LLM calls).

Figure~\ref{fig:wall} reports wall-clock times for two measures: reaching a 0.5 return (left) and completing a 2k-step run (right). In the left panel, PPO reaches 0.5 only on \textsc{Frozen Lake}, while both MIRA variants converge across all environments.
In the right panel, PPO shows the lowest per-step runtime because online MIRA incurs some additional cost from its initial LLM queries. However, early queries reduce wasted exploration, allowing online MIRA to reach 0.5 return much faster overall, yielding a net gain in efficiency despite the upfront overhead.

A \n{complete view of these trade-offs comes from considering both Figure~\ref{fig:wall} and the right-hand panel of Figure~\ref{fig:all_abs}. Together, they show how higher return and wall-clock time interact when online LLM latency is present. Although online MIRA incurs additional latency from  occasional queries, a substantial portion of this cost is offset by faster policy improvement: the agent spends less time in unproductive exploration and reaches competent behavior sooner. This can be seen directly by comparing the left and right panels of Figure~\ref{fig:wall}: online MIRA has slightly higher per-step runtime, yet it reaches the 0.5 return threshold  earlier in wall-clock time. The right-hand panel of Figure~\ref{fig:all_abs} reinforces this result, showing that the higher cost of online queries is compensated by more rapid performance gains.}

\n{\subsection{Memory Growth Analysis}}

 \begin{wrapfigure}{r}{0.44\textwidth}
\vspace{-8.4mm}
    \centering
    \includegraphics[width=\linewidth]{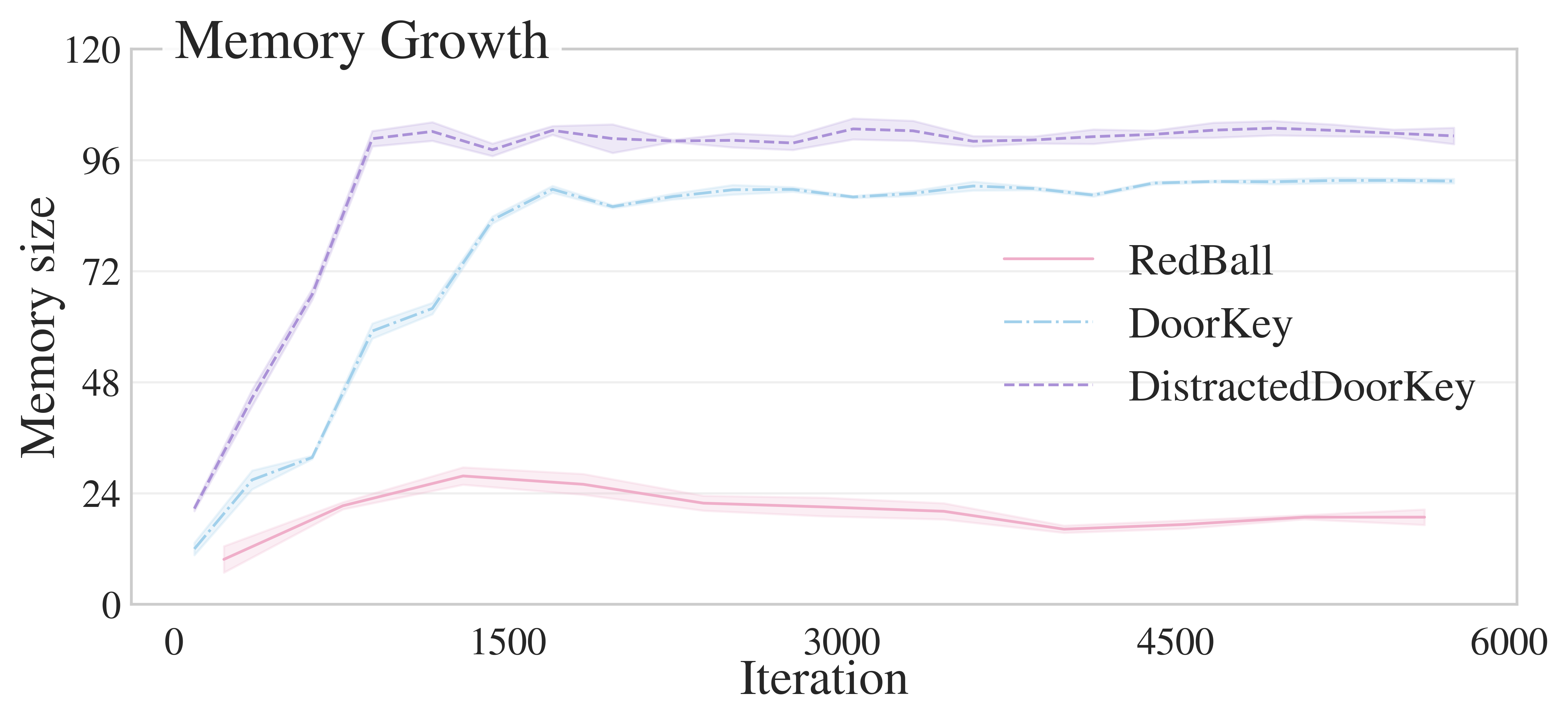}
    \caption{\n{Memory size growth over training. Early growth is followed by convergence once trajectories become consistent, with final sizes increasing from \textsc{RedBall} to \textsc{DoorKey} to \textsc{Distracted DoorKey} in line with task complexity.}}
    \label{fig:memory}
    \vspace{-5mm}
\end{wrapfigure}
W\n{e examine how the memory graph evolves during training. MIRA updates the graph only when computing the utility term during advantage estimation, so memory expansion occurs once per training batch rather than at every environment step. Because batches contain full trajectories, their lengths vary across environments. To obtain a consistent summary across tasks, we record the memory size every 100 training iterations.}

F\n{igure~\ref{fig:memory}  reports memory growth for \textsc{RedBall}, \textsc{DoorKey}, and \textsc{Distracted DoorKey}. In all three environments, memory grows quickly during early training, reflecting the period where the agent encounters diverse high-return segments and, when applicable, issues LLM queries that generate candidate memory additions. 
As trajectories converge to a consistent solution, memory expansion slows and eventually stabilizes. 
The plateau levels reflect the structural demands of each task: \textsc{RedBall} retains only a small set of nodes due to its simple subgoal structure, whereas \textsc{DoorKey} and \textsc{Distracted DoorKey} require a richer collection of region-anchored segments. \textsc{Distracted DoorKey} ends up with the largest memory, as the distractor objects create additional path variants that help guide the agent, while the total size remains bounded and does not grow throughout later training.}

\subsection{Query Frequency Performance Summary}
Table~\ref{tab:freq_res_three} expands on Figure~\ref{fig:all_abs} in Subsection~\ref{q_freq}. It shows how different online query budgets impact learning progress (SR90Return, indicating the mean return when success rate first exceeds 90\%), final return, and convergence speed (total steps to termination). The results reinforce that while all MIRA variants outperform PPO, higher online budgets further accelerate training and improve asymptotic performance.

\begin{table}[H]
\centering
\setlength{\tabcolsep}{4pt}
\small
\caption{Performance on \textsc{DoorKey}. SR90Return is the mean return when success rate first exceeds 90\%; Final Return is the return at the end of training; Final Step is the total environment steps. MIRA variants outperform the baseline in both early and final return, with {MIRA (large)} achieving the highest values while converging fastest.}
\begin{tabular}{@{}lccc@{}}
\toprule
\textbf{Method} & \text{SR90Return↑
} & \text{Final Return↑} & \text{Final Step↓} \\
\midrule
Baseline & 0 $ \pm $ 0.002 & 0.009 $ \pm $ 0.001 & 10362 \\
MIRA (offline) & 0.233 $ \pm $ 0.087 & 0.295 $ \pm $ 0.123 & 10351 \\
MIRA (mid) & 0.284 $ \pm $ 0.065 & 0.902 $ \pm $ 0.012 & 10257 \\
MIRA (large) & 0.851 $ \pm $ 0.060 & 0.91 $ \pm $ 0.013 & 9961 \\
\bottomrule
\end{tabular}
\label{tab:freq_res_three}
\end{table}

\subsection{MiniGrid Performance Summary}
Tables~\ref{tab:main_results_return} and~\ref{tab:success_results} report detailed numerical results for all four MiniGrid tasks, including  mean returns and success rates averaged over unseen seeds. MIRA consistently outperforms both PPO and the hierarchical baseline across all environments, including the more complex ones such as \textsc{DoorKey} and \textsc{RedBlueDoor}. Welch’s t-tests~\citep{ruxton2006unequal} show no statistically significant difference between MIRA and LLM4Teach at the 0.05 level across metrics and environments (Table~\ref{tab:llm4teach_vs_mira}). These results support the aggregate performance trends in the main text (Figure~\ref{fig:llm4teach_comparison}), demonstrating that MIRA improves both final return and task completion.

\begin{table}[h]
\centering
\small
\setlength{\tabcolsep}{3pt}
\caption{Mean return on unseen seeds across MiniGrid environments. MIRA achieves high and stable success, comparable to LLM4Teach, despite requiring substantially fewer LLM queries.}
\begin{tabular}{lcccc}
\toprule
\text{Method} & \textsc{DoorKey} & \textsc{LavaCrossing} & \textsc{RedBlueDoor} & \textsc{RedBall} \\
\midrule
\text{Baseline RL} & 0.018 $\pm$ 0.016 & 0.012 $\pm$ 0.027 & 0.044 $\pm$ 0.042 & 0.329 $\pm$ 0.205 \\
\n{\text{HRL}}   &
\n{0.852 $\pm$ 0.017} & 
\n{0.798 $\pm$ 0.090} & 
\n{0.830 $\pm$ 0.021 }& 
\n{0.939 $\pm$ 0.046}\\
\text{LLM4Teach}   & \textbf{0.912} $\pm$ 0.075 & \textbf{0.884} $\pm$ 0.100 & 0.901 $\pm$ 0.082 & \textbf{0.946} $\pm$ 0.051 \\
\rowcolor{mpurple!20}
\textbf{MIRA}      & 0.898 $\pm$ 0.093 & 0.855 $\pm$ 0.132 & \textbf{0.911} $\pm$ 0.077 & 0.942 $\pm$ 0.054 \\
\bottomrule
\end{tabular}
\vspace{2pt}
\label{tab:main_results_return}
\vspace{-5pt}
\end{table}

\begin{table*}[t]
\centering
\small
\setlength{\tabcolsep}{4pt}
\caption{Success rate on unseen seeds across MiniGrid environments.
MIRA achieves consistently high success rates, matching LLM4Teach while requiring fewer queries, and outperforming baseline and HRL methods.}
\begin{tabular}{lcccc}
\toprule
\text{Method} & \textsc{DoorKey} & \textsc{LavaCrossing} & \textsc{RedBlueDoor} & \textsc{RedBall} \\
\midrule
\text{Baseline RL} & 
0.023 $\pm$ 0.017 & 
0.017 $\pm$ 0.020 & 
0.036 $\pm$ 0.043 & 
0.539 $\pm$ 0.064 \\

\n{\text{HRL} }& 
\n{0.897 $\pm$ 0.013} & 
\n{0.841 $\pm$ 0.085} & 
\n{0.892 $\pm$ 0.012 }& 
\n{0.956 $\pm$ 0.025}\\

\text{LLM4Teach} & 
0.970 $\pm$ 0.004 & 
0.931 $\pm$ 0.011 & 
0.956 $\pm$ 0.003 & 
0.958 $\pm$ 0.021 \\
\rowcolor{mpurple!20}
\textbf{MIRA} & 
0.953 $\pm$ 0.043 & 
0.913 $\pm$ 0.077 & 
0.944 $\pm$ 0.020 & 
0.956 $\pm$ 0.036 \\
\bottomrule
\end{tabular}
\vspace{2pt}
\label{tab:success_results}
\end{table*}

\begin{table*}
\centering
\caption{Welch’s t-test comparing LLM4Teach and MIRA (MR: Mean Return - SR: Success Rate). None of the differences are statistically significant at $\alpha = 0.05$.}
\begin{tabular}{lccccc}
\toprule
\textbf{Metric} & \textbf{LLM4Teach} & \textbf{MIRA} & \textbf{t} & \textbf{p} & \textbf{95\% CI} \\
\midrule
\textsc{DoorKey} (MR)        & $0.912 \pm 0.075$ & $0.898 \pm 0.093$ & 0.203 & 0.8495 & [--0.181,\ 0.209] \\
\textsc{DoorKey} (SR)        & $0.970 \pm 0.004$ & $0.953 \pm 0.043$ & 0.682 & 0.5647 & [--0.0885,\ 0.1225] \\
\textsc{LavaCrossing} (MR)   & $0.884 \pm 0.100$ & $0.855 \pm 0.132$ & 0.303 & 0.7778 & [--0.2443,\ 0.3023] \\
\textsc{LavaCrossing} (SR)   & $0.931 \pm 0.011$ & $0.913 \pm 0.077$ & 0.401 & 0.7260 & [--0.1681,\ 0.2041] \\
\textsc{RedBlueDoor} (MR)    & $0.901 \pm 0.082$ & $0.911 \pm 0.077$ & --0.154 & 0.8851 & [--0.1906,\ 0.1706] \\
\textsc{RedBlueDoor} (SR)    & $0.956 \pm 0.003$ & $0.944 \pm 0.020$ & 1.028 & 0.4081 & [--0.0362,\ 0.0602] \\
\textsc{RedBall} (MR)        & $0.946 \pm 0.051$ & $0.942 \pm 0.054$ & 0.093 & 0.9302 & [--0.1152,\ 0.1232] \\
\textsc{RedBall} (SR)        & $0.958 \pm 0.021$ & $0.956 \pm 0.036$ & 0.083 & 0.9387 & [--0.0717,\ 0.0757] \\
\bottomrule
\end{tabular}
\label{tab:llm4teach_vs_mira}
\end{table*}
\subsubsection{t-test: MIRA vs. LLM4Teach}

To assess whether the performance differences between LLM4Teach and MIRA are statistically significant, we conduct Welch’s t-tests on the evaluation metrics across environments and seeds. Welch’s t-test is a two-sample statistical test that does not assume equal variance. As shown in Table~\ref{tab:llm4teach_vs_mira}, none of the differences reach significance at the $\alpha = 0.05$ level. This suggests that MIRA performs comparably to LLM4Teach across all reported metrics, despite MIRA having small lower final reward.

\section{Limitations} \label{lim}

While MIRA improves sample efficiency and reduces reliance on frequent LLM queries, it also comes with natural trade-offs. The method relies on offline LLM outputs to initialize its memory graph, which, if they include misleading information or are not well aligned with the environment dynamics, can slow convergence or increase the need for online queries. Our screening and pruning mechanisms reduce this risk, and in practice it is largely a limitation of current LLMs that is expected to diminish as models improve. MIRA also introduces shaping terms that require hyperparameter tuning to avoid instability between the actor and critic. We find, however, that they can be adjusted with standard procedures. Finally, our current study focuses on discrete action spaces; extending MIRA to continuous domains without discretization is a natural next step.

\section{Reproducibility} \label{rep}

Experiments were run on both a Linux server with Intel Xeon E5-2630 v4 CPUs (40 threads) and an Apple M2 (8-core CPU, 10-core GPU, 16GB unified memory). All LLM models used in our experiments correspond to the publicly available versions released in the first week of August 2025.

\subsection{Simulation Platforms}

\subsubsection{Gymnasium Toy text}

\textsc{Environment Details.}
Horizon indicates the maximum number of steps per episode before automatic termination (i.e., maxsteps in the environment configuration).
\begin{table}[H]
\centering
\label{tab:frozenlake_details}
\caption{FrozenLake environment details.}
\begin{tabular}{@{}lc@{}}
\toprule
\textbf{Property} & Value \\
\midrule
Observation Type & Discrete \\
Horizon & 200 \\
Reward Sparsity & Sparse \\
Action Space & 4 (tabular) \\
Dynamics & Slippery, irreversible \\
\bottomrule
\end{tabular}

\end{table}

\textsc{Hyperparameter.} Table~\ref{tab:froz} provides the main specifications of FrozenLake for \texttt{PPOConfig} in RLlib.
\begin{table}[h]
\centering
\caption{Hyperparameters of \textsc{FrozenLake}}
\begin{tabular}{@{}lccc@{}}
\toprule
\textbf{Parameter} & \text{Value
} \\
\midrule
Learning rate & $1 \times 10^{-4}$ \\
Batch size & 512 \\
Mini-batch size & 64 \\
Number of epochs & 4 \\
Entropy coefficient & 0.01 \\
Discount factor ($\gamma$) & 0.99 \\
GAE lambda ($\lambda$) & 0.95 \\
Utility ($\xi$) & [0.9]\\
Batch mode & ``complete episodes" \\
\bottomrule
\end{tabular}
\label{tab:froz}
\end{table}

\subsubsection{Minigrid and BabyAI}
\textsc{Environment Details.}
Horizon indicates the maximum number of steps per episode before automatic termination (i.e., maxsteps in the environment configuration).
\begin{table}[H]
\centering
\label{tab:mini}
\caption{MiniGrid suite details.}
\begin{tabular}{@{}lc@{}}
\toprule
\textbf{Property} & Value \\
\midrule
Observation Type & RGB \\
Reward Sparsity & Sparse and delayed \\
Action Space & 7 (tabular) \\
View Size & 7 \\
Horizon & 300 \\
\bottomrule
\end{tabular}
\end{table}
\begin{table}[H]
\centering

\label{tab:minigrid_dynamics}
\renewcommand{\arraystretch}{1}
\small
\caption{MiniGrid environments and their dynamics.}
\begin{tabular}{@{}ll@{}}
\toprule
\textbf{Environment} & \textbf{Dynamics} \\
\midrule
\setlength{\tabcolsep}{1pt}
\small
\textsc{RedBall}          & Reversible \\
\textsc{RedBlueDoor}      & Irreversible \\
\textsc{LavaCrossing}          & Irreversible \\ 
\textsc{DoorKey}            & Subgoal seq. \\
\textsc{Distracted DoorKey } & +Visual distractors  \\
\bottomrule
\end{tabular}
\end{table}

\textsc{Hyperparameter.}  Tables~\ref{tab:dk}-~\ref{tab:rb} provides the main specifications of all the MiniGrid environments for \texttt{PPOConfig} in RLlib.

\begin{table}[t]
\centering
\begin{minipage}{0.48\linewidth}
\centering
\small
\caption{Hyperparameters of \textsc{DoorKey}}
\begin{tabular}{@{}lc@{}}
\toprule
\textbf{Parameter} & \textbf{Value} \\
\midrule
Learning rate & $2.5 \times 10^{-4}$ \\
Batch size & 1024 \\
Mini-batch size & 64 \\
Number of epochs & 4 \\
Entropy coefficient & 0.01 \\
Discount factor ($\gamma$) & 0.99 \\
GAE lambda ($\lambda$) & 0.95 \\
Utility ($\xi$) & [0.25, 0.15] \\
Batch mode & ``complete episodes'' \\
\bottomrule
\end{tabular}
\label{tab:freq_res}
\end{minipage}\hfill
\begin{minipage}{0.48\linewidth}
\centering
\small
\caption{Hyperparameters of \textsc{LavaCrossing}}
\begin{tabular}{@{}lc@{}}
\toprule
\textbf{Parameter} & \textbf{Value} \\
\midrule
Learning rate & $2.5 \times 10^{-4}$ \\
Batch size & 1024 \\
Mini-batch size & 64 \\
Number of epochs & 4 \\
Entropy coefficient & 0.01 \\
Discount factor ($\gamma$) & 0.99 \\
GAE lambda ($\lambda$) & 0.95 \\
Utility ($\xi$) & [0.3] \\
Batch mode & ``complete episodes'' \\
\bottomrule
\end{tabular}

\label{tab:dk}
\end{minipage}
\end{table}

\begin{table}[t]
\centering
\begin{minipage}{0.48\linewidth}
\centering
\small
\caption{Hyperparameters of \textsc{RedBlueDoor}}
\begin{tabular}{@{}lc@{}}
\toprule
\textbf{Parameter} & \textbf{Value} \\
\midrule
Learning rate & $5 \times 10^{-5}$ \\
Batch size & 1024 \\
Mini-batch size & 64 \\
Number of epochs & 4 \\
Entropy coefficient & 0.01 \\
Discount factor ($\gamma$) & 0.99 \\
GAE lambda ($\lambda$) & 0.9 \\
Utility ($\xi$) & [0.25] \\
Batch mode & ``complete episodes'' \\
\bottomrule
\end{tabular}
\label{tab:rbd}
\end{minipage}\hfill
\begin{minipage}{0.48\linewidth}
\centering
\small
\caption{Hyperparameters of \textsc{RedBall}}
\begin{tabular}{@{}lc@{}}
\toprule
\textbf{Parameter} & \textbf{Value} \\
\midrule
Learning rate & $2 \times 10^{-4}$ \\
Batch size & 512 \\
Mini-batch size & 64 \\
Number of epochs & 4 \\
Entropy coefficient & 0.01 \\
Discount factor ($\gamma$) & 0.99 \\
GAE lambda ($\lambda$) & 0.95 \\
Utility ($\xi$) & [0.2] \\
Batch mode & ``complete episodes'' \\
\bottomrule
\end{tabular}
\label{tab:rb}
\end{minipage}
\end{table}

\textsc{Observation Space} In MiniGrid environments, the agent receives an RGB image of the grid, which is passed through a convolutional encoder~\ref{fig:cnn} to extract spatial features relevant for navigation and interaction. 
\begin{figure}[H]
  \centering
  % \vspace{-5mm}
    \includegraphics[width=.7\linewidth]{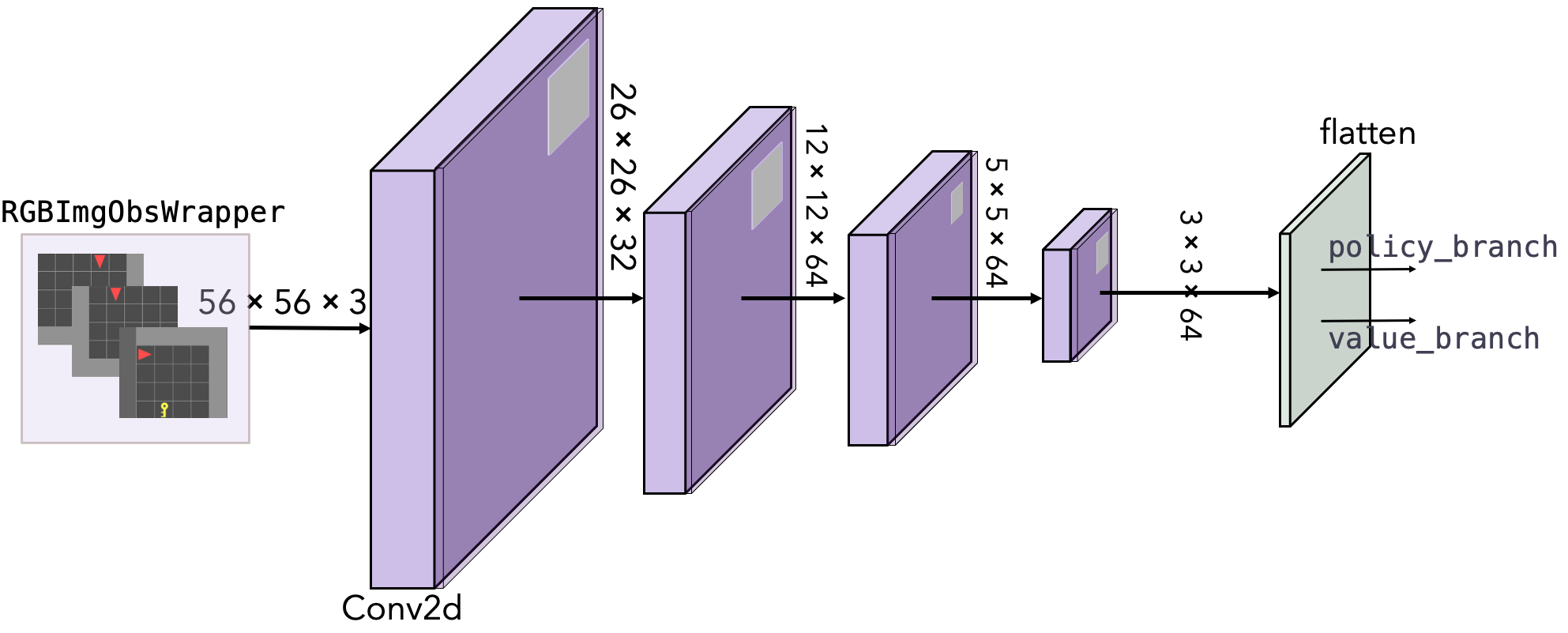}
    \caption{Convolutional encoder architecture used to process the agent’s $56 \times 56 \times 3$ RGB observation in MiniGrid environments. The input passes through a series of Conv2D layers, reducing spatial dimensions while increasing channel depth. The final activation is flattened and fed to both policy and value heads. This encoder captures spatial layout, object presence, and agent-centric context for decision-making.}
    \label{fig:cnn}
\end{figure}
This CNN processes the visual input into a compact feature vector, capturing object positions, colors, and layout structure. 
The resulting embedding is concatenated with a learned directional encoding and passed to the policy and value heads for action selection and value estimation.

\subsection{LLM Confidence Settings}

For completions where token-level likelihoods are available, confidence is computed using an exponential of the geometric-mean log-probability ($\exp\left({1}/{L} \sum \log p_i\right)$) with a fixed likelihood threshold of $\tau  \ge 0.65$. When likelihoods are unavailable, we obtain k independent completions ($k = 3$) and retain only outputs that pass a majority-consistency test with a fixed agreement threshold $\tau \ge 2/3$.

% \subsection{Utility Computer}

\section*{Use of Large Language Models (LLMs)} \label{disc}

During the preparation of this manuscript, the authors used OpenAI’s ChatGPT to assist with grammar and readability. No research ideas, technical content, or analysis were generated by the tool. All content was reviewed and verified by the authors, who take full responsibility for the final version.
\end{document}